\documentclass[11pt]{article} %

\usepackage{times}
\usepackage{url}            %
\usepackage{booktabs} 
\usepackage{wrapfig}%
\usepackage{amsfonts}       %
\usepackage{nicefrac}       %
\usepackage{microtype}      %
\usepackage{mkolar_definitions}
\usepackage{subcaption} 
\usepackage{caption}
\usepackage{xspace}
\usepackage{multirow}
\usepackage{amsmath}
\usepackage{algorithm}
\usepackage{algorithmic}
\usepackage{color}
\usepackage{color, colortbl}
\usepackage{enumitem}
\usepackage{comment}
\usepackage{bm}
\usepackage[left=2cm,top=2cm,right=2cm]{geometry}
\usepackage[scaled=.9]{helvet}
\usepackage{url}
\usepackage{authblk}
\usepackage{csquotes}
\usepackage[dvipsnames]{xcolor}
\usepackage[colorlinks=true, linkcolor=blue, citecolor=blue]{hyperref}
\usepackage{natbib}
\bibpunct{(}{)}{;}{a}{,}{,}
\usepackage[para,symbol]{footmisc}
 
\newcommand{\olive}{\textsc{Olive}\xspace}
\newcommand{\moffle}{\textsc{Moffle}\xspace}
\newcommand{\algname}{\textsc{Briee}\xspace}
\newcommand{\homer}{\textsc{Homer}\xspace}
\newcommand{\repucb}{\textsc{Rep-UCB}}
\newcommand{\alglong}{Block-structured Representation learning with Interleaved Explore Exploit\xspace}

\newcommand{\featurestep}{\hyperref[line:least_square_phi]{line 7}}
\newcommand{\discstep}{\hyperref[line:discriminator_selection]{line 5}}

\definecolor{Gray}{gray}{0.9}

\newcommand{\wen}[1]{\noindent{\textcolor{purple}{\{{\bf WS:} \em{#1}\}}}}

\begin{document}
\title{Efficient Reinforcement Learning in Block MDPs: A Model-free Representation Learning Approach}

\author[1]{Xuezhou Zhang\footnote{\{xz7392, mengdiw\}@princeton.edu}}
\author[2]{Yuda Song\footnote{yudas@andrew.cmu.edu}}
\author[3]{Masatoshi Uehara\footnote{\{mu223, ws455\}@cornell.edu}}
\author[1]{Mengdi Wang${}^\ast$}
\author[4]{Alekh Agarwal\footnote{alekhagarwal@google.com}}
\author[3]{Wen Sun${}^\ddagger$}
\affil[1]{Princeton University}
\affil[2]{Carnegie Mellon University}
\affil[3]{Cornell University}
\affil[4]{Google Research}
\date{}
\maketitle
\begin{abstract}
We present \algname~(\alglong), an algorithm for efficient reinforcement learning in Markov Decision Processes with block structured dynamics (i.e., Block MDPs), where rich observations are generated from a set of unknown latent states. \algname~interleaves latent states discovery, exploration, and exploitation together, 
and can provably learn a near-optimal policy with sample complexity
scaling polynomially in the number of latent states, actions, and the time horizon, with no dependence on the size of the potentially infinite observation space.
Empirically, we show that \algname is more sample efficient than the state-of-art Block MDP algorithm \homer  and other empirical RL baselines on challenging rich-observation combination lock problems which require deep exploration.
\end{abstract}

\section{Introduction}

Representation learning in Reinforcement Learning (RL) has gained increasing attention in recent years from both theoretical and empirical research communities \citep{schwarzer2020data,laskin2020curl} due to its potential in enabling sample-efficient non-linear function approximation, the benefits in multitask settings \citep{zhang2020learning,yang2020provable,sodhani2021block}, and the potential to leverage advances on representation learning in related areas such as computer vision and natural language processing. Despite this interest, there remains a gap between the theoretical and empirical literature, where the theoretically sound methods are seldom evaluated or even implemented and often rely on strong assumptions, while the empirical techniques are not backed with any theoretical guarantees even under stylistic assumptions. This leaves open the key challenge of designing representation learning methods that are both theoretically sound and empirically effective. 

\begin{figure*}[!th]
\begin{tabular}{cc}
\begin{minipage}{0.65\textwidth}
\begin{center}
\begin{tabular}{|c|c|c|}
\hline
&Sample Complexity & Reward\\ \hline
\olive~\citep{jiang2017contextual}& $\frac{|\Zcal|^3H^3|\Acal|^2\log|\Phi|}{\epsilon^2}$ & Yes\\ \hline
\textsc{Rep-UCB}~\citep{uehara2021representation} & $\frac{|\Zcal|^4H^5|\Acal|^2\ln(|\Phi| {\textcolor{red}{|\Upsilon|}})}{\epsilon^2}$ & Yes\\ \hline
\moffle~\citep{modi2021model} & $\frac{|\Zcal|^7H^8|\Acal|^{13}\ln|\Phi|}{\min(\epsilon^2\eta_{\min},\eta_{\min}^5)}$ & No \\[1ex] \hline
\homer~\citep{misra2019kinematic} & $\frac{|\Zcal|^6H|\Acal|(|\Zcal|^2|\Acal|^3 + \ln|\Phi|)}{\min(\eta_{\min}^3,\epsilon^2)}$ & No \\[0.5ex] \hline
\rowcolor{Gray}\algname (this paper) & $\frac{|\Zcal|^8H^9|\Acal|^{14}\ln|\Phi|}{\epsilon^4}$ & Yes \\\hline
\end{tabular}\\~\\
(a)
\end{center}
\end{minipage}
&
\begin{minipage}{0.32\textwidth}\label{exp:fig:comblock}
\centering
\includegraphics[width=\textwidth]{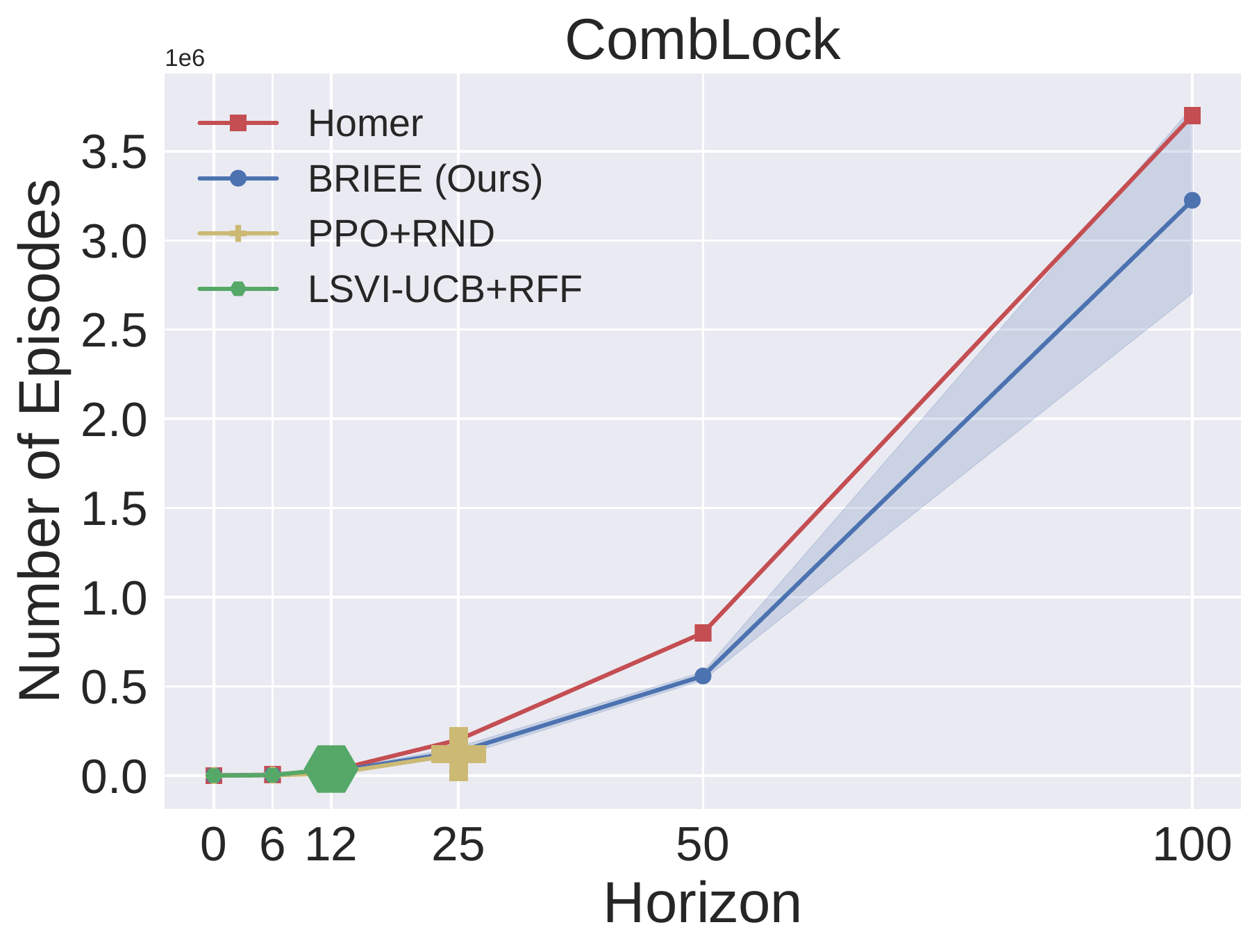}
(b)
\end{minipage}
\end{tabular}
\caption{\textbf{(a)} Sample complexities of different approaches to learning an $\epsilon$-optimal policy for Block MDPs on a given reward function. Some approaches are reward-free and we only count the sample complexity of exploration in these cases. For \moffle and \homer, $\eta_{\min}$ is the minimum probability of reaching any latent state and can be arbitrarily small or even zero in the worst case. $\Upsilon$ in \textsc{Rep-UCB} refers to the function class for modeling the emission distribution since their approach is model-based, which is an additional inductive bias. We omit all log factors other than those depending on function class sizes. Our bound for \olive is a factor of $|\Zcal||\Acal|$ larger than that of~\citep{jiang2017contextual} as we include the uniform convergence over the linear function class $\{w^\top\phi~:~\|w\|_2 \leq 1, \phi\in\Phi\}$, where $\phi\in\RR^{|\Zcal||\Acal|}$, to capture the $Q^\star$ function in a Block MDP. \textbf{(b)} Empirical evaluation of \algname against baselines in a challenging block MDP (see Section~\ref{sec:exp} for details), showing the number of episodes required to find a near-optimal solution for varying horizon lengths. Large marker indicates that the method fails to solve the problem (within the sample budget) for larger horizon. \homer and ours can solve the problem up to $H=100$, but our method consistently outperforms \homer.}
\label{tbl:results}
\end{figure*}

In this work, we tackle this challenge for a special class of problems called \emph{Block MDPs}, where the high dimensional and rich observations of the agent are generated from certain latent states and there exists some fixed, but unknown mapping from observations to the latent states (each observation is generated only by one latent state). Prior works~\citep{dann2018oracle, du2019provably,misra2020kinematic,zhang2020learning,sodhani2021block} have motivated the Block MDP model through scenarios such as navigation tasks and image based robotics tasks where the observations can often be reasonably mapped to the latent physical location and states. 
We develop a new algorithm \algname, which finds a provably good policy for any Block MDP. It performs model-free representation learning with a form of adversarial training to learn the features, interleaved with deep exploration and exploitation. Unlike prior theoretical works, our new approach does not require \textit{uniform reachability}, i.e., every latent state is reachable with sufficient probability, which is a strong assumption that cannot be guaranteed or verified.  
We also demonstrate the empirical effectiveness of our algorithm in Block MDPs that are challenging to explore. Importantly, our technique is model-free which means there is no need to model the observation generation process that can be complex for high dimensional sensory data.

\paragraph{Contributions} Our key contributions are three folds: 
\begin{enumerate}[leftmargin=*,itemsep=0pt]
\item We design a new algorithm \algname  that solves any Block MDP with polynomial sample complexity, with no explicit dependence on the number of states which could be infinite;
\item \algname does not require reachability assumption, and can directly optimize a given reward function; \looseness=-1
\item \algname's computation oracles can be easily implemented using standard gradient based optimization. Our experiments show that \algname is more sample efficient than \homer \citep{misra2020kinematic}, can be extended to richer MDPs where the block structure does not hold, and can leverage dense reward structure to achieve improved sample efficiency.
\end{enumerate}
We summarize our theoretical results in Figure~\ref{tbl:results}(a).
Note that our strong empirical performance beyond Block MDPs (e.g., low-rank MDPs) suggests that it might be possible to extend the theoretical analysis to more general settings, and we leave this as an important direction for future work.

\subsection{Related Work}

We survey some of the relevant literature here. 

\paragraph{RL with function approximation.} There has been considerable progress on sample-efficient RL with function approximation in recent years. While some of it focuses on the linear case (e.g.~\citep{jin2020provably,yang2019sample}) which does not involve representation learning, other works have developed information-theoretically efficient methods for non-linear function approximation~\citep{jiang2017contextual,sun2019model,du2021bilinear,jin2021bellman}, some of which subsume our setup in this paper. Particularly relevant is \olive ~\citep{jiang2017contextual} which can solve Block MDPs in a model-free manner  and has a better sample complexity than \algname. However, it is known to be computationally intractable~\citep{dann2018oracle} even for tabular MDPs. 


\paragraph{Low-rank MDPs} Low-rank MDP is \emph{strictly more general} than linear MDPs which assume representation is known a priori. There are several related papers come from the recent literature on provable representation learning for low-rank MDPs~\citep{agarwal2020flambe,modi2021model,uehara2021representation,ren2021free}. Low-rank MDPs generalize Block MDPs, so these algorithms are applicable in our setting. Of these, however, only \citet{modi2021model} handles the model-free case, while the other approaches are model-based and pay a significant sample complexity overhead in modeling the generative process of the observations. 
The model-free approach of \citet{modi2021model} is the closest to our work and we build on some of their algorithmic and analysis ideas. However, their work makes a significantly stronger assumption that each latent state must be reachable with at least constant probability and do not interleave exploration and exploitation. They also do not provide any empirical validation of their approach. Taking a slightly different approach, \citet{sekhari2021agnostic} studies low-rank MDP from an agnostic policy-based perspective and only requires a policy class that not necessarily contains the optimal policy (thus the name agnostic). However, they show that exponential sample complexity in this setting is not avoidable, which indeed justifies the need for using function approximation to capture representations.
\citet{zhang2021provably} and \citet{papini2021reinforcement} also refer to their settings as \textit{representation learning}, their goal is to choose the most efficient representation among a set of correct representations (i.e., every representation still linearizes the transition), which is stronger than the typical notion of feature learning from an arbitrary function class. 

\paragraph{Block MDPs} There are two prior results on sample-efficient and practical learning in Block MDPs~\citep{du2019provably,misra2019kinematic}. \citet{du2019provably} requires the start state to be deterministic, makes the reachability assumption on latent states, and their sample complexity has an undesirable polynomial dependency on the failure probability $1/\delta$. \citet{misra2019kinematic} removes the deterministic start state assumption but still requires reachability. Both approaches are tailor-made for Block MDPs. These approaches also have to learn in a layer-by-layer forward fashion, which is not ideal in practice (i.e., they cannot learn stationary policies for episodic infinite horizon discounted setting, while our approach can be extended straightforwardly).  
In contrast, \algname learns in all layers simultaneously. \citep{zhang2020learning} extends the Block MDP to multi-task learning and study how the error from a \emph{given} state abstraction affects the multi-task performance, but do not theoretically study how to learn such an  abstraction.  \citet{feng2020provably} assume a high level oracle that can decode latent states. \citet{foster2020instance} focus on instance-dependent bounds, but their bounds scale with a \textit{value function disagreement coefficient} and inverse value gap, both of which can be arbitrarily large in general Block MDPs (e.g., disagreement coefficient is a stronger notation than the usual classic notation of uniform convergence which is what we use here). Finally, \citet{duan2018state} and \citet{ni2021learning} study state abstraction learning from logged data, without identifying the optimal policy.

\paragraph{Approaches from the empirical literature} There are exploration techniques with non-linear function approximation from the deep reinforcement learning literature (e.g.~\citep{bellemare2016unifying,pathakICMl17curiosity,burda2018exploration,machado2020count,sekar2020planning}. Of these, we include the RND approach of~\citet{burda2018exploration} in our empirical evaluation. The use of adversarial discriminators for feature learning is somewhat related to the insights in~\citet{bellemare2019geometric}, but unlike our approach, they use random adversaries in the empirical evaluation, and do not focus on strategic exploration and data collection.

\section{Preliminaries}
\label{sec:prelim}
We consider a finite horizon episodic Markov Decision Process $\left\langle \Scal, \Acal, \{r_h\}_{h=0}^{H-1}, H, \{P_h\}_{h=0}^{H-1} , d_0 \right\rangle$, where $\Scal$ and $\Acal$ are the state and action space, $P_h, r_h$ are the transition and reward at time step $h\in[H]$, $H$ being the episode length; $d_0 \in \Delta(\Scal)$ is the initial state distribution. For normalization, we assume the trajectory cumulative reward is bounded in $[0,1]$.

An MDP is called a \textbf{low-rank MDP}~\citep{rendle2010factorizing,yao2014pseudo,jiang2017contextual} if the transition matrix at any time step $h$ is low-rank, i.e., there exist two mappings $\mu^\star_h :\Scal\mapsto \mathbb{R}^d$, and $\phi^\star_h:\Scal\times\Acal\mapsto \mathbb{R}^d$, such that for any $s,a,s'$, we have $P^\star_h(s' | s,a) = \mu^\star_h(s')^{\top} \phi^\star_h(s,a)$. Denote $d$ the rank of $P^\star_h$. Note that for low-rank MDP, \emph{neither $\mu^\star_h$ nor $\phi^\star_h$ are known}, which is fundamentally different from the linear MDP model \citep{jin2020provably,yang2019sample} where $\phi^\star_h$ is known.  Learning in low-rank MDPs requires either directly learning a near-optimal policy through general function approximation, or doing representation learning  first (again through nonlinear function approximation), followed by linear techniques. Either way, low-rank MDPs provide an expressive framework for analyzing non-linear function approximation in RL.

In this work, we mainly focus on analyzing a special case of low-rank MDPs, called \textbf{Block MDPs}~\citep{du2019provably, misra2020kinematic}. 
We denote $\Zcal$ as a latent state space where $|\Zcal|$ is small. Denote $\Zcal\times\Acal$ as a joint space whose size is $|\Zcal||\Acal|$.  In a Block MDP, each state $s$ in generated from a unique latent state $z$ as described below (hence the name block), which means that the latent state is decodable by just looking at the state. Denote the (unknown) ground truth mapping from $s$ to the corresponding $z$ as $\psi_h^\star: \Scal\mapsto \Zcal$ for all $h$. A Block MDP is formally defined as follows. 
\begin{definition}[Block MDP] \label{def:block_mdp}
Consider any $h\in [H]$. A Block MDP has an emission distribution $o_h(\cdot | z) \in \Delta(\Scal)$ and a latent state space transition $T_h(z' | z, a)$, such that for any $s\in\Scal$, $o_h(s|z) > 0$ for a unique $z\in \Zcal$ denoted as $\psi_h^\star(s)$. Together with the ground truth decoder $\psi^\star_h$, it defines the transitions $P^\star_h(s' | s, a) = \sum_{z'\in \Zcal} o_h(s' | z') T_h(z' | \psi^\star_h(s), a)$.
\end{definition}
The Block MDP structure allows us to model the setting where the states $s$ are high dimensional rich observations (e.g., images) and the state space $\Scal$ is exponentially large or even infinite. In the rest of the paper, we use words \emph{state} and \emph{observation} interchangeably for $s$ with the impression that $s$ is a high dimensional object from an extremely large space $\Scal$. 
Block MDPs are generalized by the {low-rank MDP model}. Denote the ground truth feature vector $\phi_h^{\star}(s,a)$ at step $h$ as a $|\Zcal||\Acal|$-dimensional vector $e_{(\psi^\star(s),a)}$ where $e_i$ is the $i_{th}$ basis vector, so that it is non-zero only in the coordinate corresponding to $(\psi_h^\star(s),a)\in\Zcal\times\Acal$. Correspondingly, for any $s\in \Scal$, $\mu_h^{\star}(s)$ is a $|\Zcal||\Acal|$ dimensional vector such that the $(z,a)_{th}$ entry is $\sum_{z'\in\Zcal} o_h(s | z')T_h(z'|z,a)$.
Then $P_h^{\star}(s' | s,a) = \mu_h^{\star}(s')^{\top} \phi_h^{\star}(s,a)$, so that the Block MDP is a low-rank MDP with rank $d=|\Zcal||\Acal|$. We assume that the reward function $r_h(s,a)$ is known.

\paragraph{Function approximation} Our representation learning approach to learn Block MDPs requires a feature class $\{\Phi_h\}_{h=0}^{H-1}$. Since the features $\phi^\star_h$ are one-hot in a Block MDP as described above, it is natural to use the same structure in the class $\Phi_h$ as well, since it yields statistical and algorithmic advantages as we will explain in the sequel. So any $\phi_h\in\Phi_h$ is parameterized by a candidate decoder $\psi_h~:~\Scal\to\Zcal$ that aims to approximate $\psi^\star_h$, with $\phi_h(s,a) = e_{(\psi_h(s),a)}$.  
In algorithm and analysis, we will mostly work with the state-action representation class $\Phi_h$ directly, but discuss the benefits of the specific Block MDP structure when important. This is because we intend to make our algorithm as general as possible and indeed as we will see, our algorithm can be directly applied to low-rank MDP, although our analysis only focuses on Block MDPs. 

We aim to learn a near optimal policy with sample complexity scaling polynomially with respect to $|\Zcal|, |\Acal|, H$, and the statistical complexity of $\Phi_h$ rather than the size of the state space $|\Scal|$ which could be infinite here. In this work, we will focus our analysis on the setting of finite $\Phi_h$ and thus the statistical complexity $\Phi_h$ is simply $\log(|\Phi_h|)$. Extending to continuous $\Phi_h$ is straightforward by using statistical complexities such as covering number, since our analysis only uses the standard uniform convergence property on $\Phi_h$.

\paragraph{Model-free vs model-based} Like \homer \citep{misra2020kinematic},  we only model $\phi^\star$ via function approximation (hence model-free), while FLAMBE and Rep-UCB  \citep{ agarwal2020flambe,uehara2021representation} additionally model omission distributions $o(\cdot | z)$ (i.e., model-based), which could be complex when observations are high dimensional. 

\paragraph{Notation} We denote $\pi = \{\pi_0,\dots, \pi_{H-1}\}$ as the non-stationary Markovian policy, where each $\pi_h$ maps from a state $s$ to a distribution over actions $\Delta(\Acal)$, and $V^{\pi}_h(s)$ as the value function of $\pi$ at time step $h$, i.e., $V^{\pi}_h(s) = \EE\left[ \sum_{\tau = h}^{H-1} r_h | \pi, s_h = s \right]$. We denote $Q^\pi_h(s,a) = r_h(s,a) + \EE_{s'\sim P^\star_h(s,a)} V^\pi_{h+1}(s')$.  We denote $V^{\pi}_{P,r} \in \mathbb{R}^+$ as the expected total reward of $\pi$ under non-stationary transitions $P:=\{P_h\}_h$ and rewards $r:=\{r_h\}_h$.

We define $d^{\pi}_h(s,a)$ as the probability of $\pi$ visiting a state-action pair $(s,a)$ at time step $h$. We abuse notation a bit and denote $d^{\pi}_h(s)$ as the marginalized state distribution, i.e., $d^{\pi}_h(s) = \sum_a d^{\pi}_h(s,a)$. Given $d^{\pi}_h$, we denote $s\sim d^{\pi}_h$ as sampling a state at time step $h$ from $d^{\pi}_h$, which can be done by executing $\pi$ for $h-1$ steps starting from $h = 0$. We denote $U(\Acal)$ as a uniform distribution over action space $\Acal$. For a vector $x$ and a PSD matrix $\Sigma$, we denote $\|x\|^2_{\Sigma} = x^{\top} \Sigma x$. For $n \in \mathbb{N}^+$, we use $[n] = \{0,1,\dots, n-1\}$. Lastly, we denote $|\Phi|\coloneqq\max_{h\in[H]}|\Phi_h|$, and $a\wedge b = \min(a,b)$.

\begin{algorithm}[t!] 
\caption{\alglong (\algname)} \label{alg:main}
\begin{algorithmic}[1]
  \STATE {\bf  Input:} Representation classes $\{\Phi_h\}_{h=0}^{H-1}$, discriminator classes $\{\Fcal_h\}_{h=0}^{H-1}$, parameters $N, T_n,\alpha_n,\lambda_n$
  \STATE Initialize policy $\hat\pi^0 = \{\pi_0,\dots, \pi_{H-1}\}$ arbitrarily and replay buffers $\Dcal_h = \emptyset, {\Dcal}'_h = \emptyset$ for all $h$
  \FOR{$n = 1 \to N$}
    \STATE Data collection from $\hat \pi^{n-1}$: $\forall h \in [H]$,\\
    \quad $ s\sim d^{\hat \pi^{n-1}}_h, a\sim U(\Acal), s' \sim P^\star_h(s,a);$ \\
    \quad $\tilde{s} \sim d^{\hat \pi^{n-1}}_{h-1}, \tilde a\sim U(\Acal), \tilde s' \sim P^\star_{h-1}(\tilde s,\tilde a)$,\\
    \quad $\tilde a'\sim U(\Acal), \tilde s''\sim P^\star_h(\tilde s',\tilde a')$\\
    \quad $\Dcal_h = \Dcal_h \cup \{s,a,s'\}$ and $\Dcal_h' = \Dcal_h' \cup \{\tilde s',\tilde a',\tilde s''\}$.
    \label{line:data}
    \STATE Learn representations for all $h\in[H]$:\\
        $\hat\phi^n_h = \textsc{RepLearn}\left(\Dcal_h\cup\Dcal'_h, \Phi_h, \Fcal_h, \lambda_n, T_n, \ell_n\right)$
    \vspace{0.1cm}
    \STATE Define exploration bonus for all $h\in[H]$:\\
        $\hat b^n_h(s,a) := \min\Big\{ \alpha_n \sqrt{ {\hat \phi_h^n}(s,a)^\top \Sigma^{-1}_{h} {\hat \phi_h^n}(s,a)  }, 2 \Big\}$,  
    with $\Sigma_h := \sum_{s,a,s'\sim \Dcal_h} \hat\phi^n_h(s,a) {\hat\phi^n_h}(s,a)^{\top} + \lambda_n I$.
    \label{line:bonus}
    \STATE Set $\hat\pi^n$ as the policy returned by:\\
    \hspace{-0.4cm}$\textsc{LSVI}\big( \{r_h+ \hat b^n_h\}_{h=0}^{H-1}, \{ \hat\phi^n_h \}_{h=0}^{H-1}, \{ \Dcal_h \cup \Dcal'_h \}_{h=0}^{H-1}, \lambda_n \big)$.
  \ENDFOR 
  \STATE \textbf{Return } $\hat\pi^0,\dots, \hat\pi^{N}$
\end{algorithmic}
\end{algorithm}

\section{Our Algorithm}

In this section, we present our algorithm \textsc{Briee}: \emph{\alglong}. We first give an overview of our algorithm and then describe how to perform representation learning. 

\paragraph{Algorithm Overview} \pref{alg:main} operates in an episodic setting. In episode $n$, we use the latest policy $\hat\pi^{n-1}$ to collect new data for every time step $h$. 
Note that in our data collection scheme, for each time step $h$, we maintain two replay buffers $\Dcal_h$ and ${\Dcal}'_h$ of transitions $(s,a,s')$ which draw the state $s$ from slightly different distributions (\hyperref[line:data]{line 4}).
With $\Dcal_h$ and ${\Dcal}'_h$, we update the representation $\hat\phi_h$ for time step $h$ by calling our \textsc{RepLearn} oracle which is described in \pref{alg:rep_learn}. 
We then formulate the linear-bandit and linear MDP style bonus $\hat b^n_h(s,a)$ using the latest representation $\hat\phi^n_h$. 
Note that the bonus is constructed using only the replay buffer $\Dcal_h$. When the features $\hat\phi^n_h$ are one-hot for a Block MDP, the first term inside the minimum in the bonus definition (\hyperref[line:bonus]{line 6}) simplifies to $\alpha_n/\sqrt{\lambda_n+N_h^n(\hat\psi^n_h(s),a)}$, where $\hat\psi^n_h(s)$ is the estimated latent state for $s$ corresponding to the index of the non-zero entry in $\hat\phi^n_h(s,a)$ and $N_h^n(\hat\psi^n_h(s),a)$ is the number of times we observe a transition $(\bar{s},\bar a,\bar s')$ in $\mathcal{D}_h$ with $\hat\psi^n_h(\bar{s}) = \hat\psi^n_h(s)$ and $\bar a = a$. 
With bonus $\hat b^n_h$, the representation $\hat\phi^n_h$, and the dataset $\Dcal_h + {\Dcal}'_h$, we use the standard Least Square Value Iteration (LSVI) (\pref{alg:lsvi}) to update our policy to $\pi^{n}$ using the combined reward $r_h + \hat b^n_h$.

Our algorithm is conceptually simple: it resembles the UCB style LSVI algorithm designed for linear MDPs where the ground truth features $\phi^\star$ are known. However, since $\phi^\star$ is unknown, we additionally update the representation $\hat \phi^n_h$ in every episode. Note that if the features $\hat \phi^n_h$ are one-hot, we can alternatively use the counts $N(\hat \psi^n_h(s),a)$ to estimate a tabular transition model over the inferred latent states and do tabular value iteration when the rewards only depend on the latent states. We choose to use the more general LSVI approach as it keeps our algorithm more general and we will comment more on this aspect at the end of this section.

\begin{algorithm}[t] 
\caption{Representation Learning Oracle (\textsc{RepLearn})} \label{alg:rep_learn}
\begin{algorithmic}[1]
  \STATE {\bf  Input:} Dataset $\Dcal = \{s,a,s'\}$, representation class $\Phi$, discriminator class $\Fcal$, regularization $\lambda$,  iterations $T$, termination threshold $\ell$.
  
\STATE Initialize $\phi^0 \in \Phi$ arbitrarily
\STATE Denote least squares loss:\\
        $\Lcal_{\lambda,\Dcal}(\phi,w,f) :=\EE_{\Dcal} \left( w^\top \phi(s,a) - f(s')\right)^2 + \lambda \|w\|_2^2$.
\FOR{$t = 0 \to T-1$}
    \STATE \label{line:discriminator_selection} Discriminator selection:\\  {\textcolor{blue}{\emph{\# find a discriminator that cannot be linearly predicted by the current features }}}\\
    $f^{t} = \argmax_{f} \max_{\tilde \phi\in \Phi} \Big[ \min_{w} \big[ \Lcal_{\lambda,\Dcal}(\phi^t,w,f)\big]$  \\
    \qquad\qquad\qquad\qquad$-\min_{\tilde w} \big[ \Lcal_{\lambda,\Dcal}(\tilde\phi,\tilde w,f)\big] \Big]$
    
    \STATE \label{line:earlystop} \textbf{If} $f^t$ achieves an objective value at most $\ell$:
    \textbf{Break} and \textbf{Return} $\phi^{t}$. 
    \STATE\label{line:least_square_phi} Feature selection via Least Square minimization:   \\ {\textcolor{blue}{\emph{\# find a feature map that can linearly predict all discriminators' values at next states }}}\\
    $\phi^{t+1} = \argmin_{\phi\in\Phi} \min_{\{w_i\}_{i=0}^t } \sum_{i=0}^t \Lcal_{\lambda,\Dcal}(\phi,w_i,f^i)$.
  \ENDFOR 
\end{algorithmic}
\end{algorithm}

\paragraph{Representation Learning} Now we explain our representation learning algorithm (\pref{alg:rep_learn}). This representation learning oracle follows the algorithm from \moffle \citep{modi2021model}. For completeness, we explain the intuition of the representation learning oracle here. Given a dataset $\Dcal = \{(s,a,s')\}$, \pref{alg:rep_learn} aims to learn a representation via adversarial training using the following \emph{ideal objective}:
\begin{equation*}
\min_{\phi \in \Phi_h} \max_{f\in\Fcal_h}  \left[  \min_{w} \EE_{s,a \in \Dcal}   \left( w^{\top} \phi(s,a) - \EE_{s'\sim P^\star(s,a)} f(s')  \right)^2 \right]
\end{equation*}  
where $\Fcal_h \subset [\Scal \to \RR]$ are the discriminators. In \pref{sec:analysis}, we instantiate $\Fcal_h$ as a class of linear functions on top of the representations in $\Phi_{h+1}$. To understand the intuition here, first note that regardless of $f$,  $\EE_{s'\sim P^\star_h(s,a)} f(s')$ is always a linear function with respect to the ground truth features $\phi^\star_h$ (see e.g. Proposition 2.3 in~\citet{jin2020provably}). Hence, $\phi^\star_h$ is always a minimizer of the objective above for any class $\Fcal_h$, and by using a sufficiently rich class $\Fcal_h$, we hope that any other approximate optimum is also a good approximation to $\phi^\star_h$ under the same distribution.

However, the conditional expectation inside the squared loss in our ideal objective precludes easy optimization, or even direct unbiased estimation from samples, related to the "double sampling" issue in Bellman Residual objectives from the policy evaluation literature. Following a standard approach from offline RL~\citep{antos2008learning} also used in \moffle, we instead rewrite the ideal objective with an additional term to remove the residual variance:
\begin{align}
\label{eq:final_rep_learn_obj}
\textstyle \min_{\phi \in \Phi_h}& \max_{f\in\Fcal_h} \Big[  \min_{w} \EE_{ \Dcal} \big( w^{\top} \phi(s,a) -  f(s')  \big)^2 - \min_{\tilde w, \tilde \phi\in \Phi_h} \EE_{\Dcal} \big( \tilde w^{\top} \tilde \phi - f(s') \big)^2 \Big]. 
\end{align}

\pref{alg:rep_learn} optimizes Eq.~\pref{eq:final_rep_learn_obj} through alternating updates over $f$ and $\phi$. At each iteration, it first picks features that can linearly capture expectations of all the discriminators found so far by solving a least squares problem (\featurestep). If $\phi$ is a parametric function, we can \emph{solve the least square regression problem via gradient descent} on the parameters of $\phi$ and $w_i$. Given the latest representation $\phi^t$, we simply search for a discriminator which cannot be linearly predicted by $\phi^t$ for any choice of weights (\discstep). If no such discriminator can be found (\hyperref[line:earlystop]{line 6}), then the current features $\phi^t$ are near optimal and the algorithm terminates and returns $\phi^t$. For a Block MDP, our analysis shows this process will terminate in 
polynomial number of rounds. 
Note that in \discstep, since we use ridge linear regression for $w$ and $\tilde w$, both $w$ and $\tilde w$ have closed form solutions given $\phi^t$ and $\tilde \phi$. Thus, if $f$ and $\tilde \phi$ are parameterized functions, we can compute the gradient of the objective function, which allows us to \emph{directly optimize $f$ and $\tilde \phi$ jointly via gradient ascent}.

\paragraph{Extensions and Computation} Note that our algorithm is stated in a general way that does not use any Block MDP structures. This means that the algorithm can be applied to any low-rank MDP directly. While our theoretical results only hold for Block MDPs, our experimental results indicate that the algorithm can work for more general low-rank MDPs. Note that all prior Block MDP algorithms cannot be directly applied to low-rank MDPs, and use very different approaches than general low-rank MDP learning algorithms.

Another benefit of the algorithm being not tailored to the Block MDP structure is that when we apply our algorithm to linear MDPs where $\phi^\star_h$ are known a priori,  we  can simply set $\Phi_h = \{\phi^\star_h\}$ as a singleton, rendering it as efficient as LSVI-UCB in this setting.

Compared to more general approaches such as \olive, our main benefits are algorithmic and computational. Note that \olive is provably intractable for even a tabular or linear MDP \citep{dann2018oracle}, due to complicated version space constraints. Empirically, the constraints in \olive are unsuitable for gradient-based approaches due to the presence of indicator functions involving the parameters.

\begin{algorithm}[t] 
\caption{Least Square Value Iteration \textsc{LSVI}} \label{alg:lsvi}
\begin{algorithmic}[1]
  \STATE {\bf  Input:} Reward $\{r_h(s,a)\}_{h=0}^{H-1}$, features $\{\phi_h\}_{h=0}^{H-1}$, datasets $\{\Dcal_h\}_{h=0}^{H-1}$, regularization $\lambda$
  \STATE Initialize $V_{H}(s) = 0, \forall s$
  \FOR{$h = H-1 \to 0$}
    \STATE $\Sigma_h = \sum_{s,a,s'\in \Dcal_h} \phi_h(s,a) \phi_h(s,a)^{\top} + \lambda I$ \label{line:covariance_lsvi}
    \STATE $w_h = \Sigma_h^{-1} \sum_{s,a,s'\in \Dcal_h} \phi_h(s,a) V_{h+1}(s')$.
    \STATE Set $Q_h(s,a) = w_h^{\top} \phi_h(s,a) + r_h(s,a)$, and $V_h(s) = \max_{a} Q_h(s,a)$
    \STATE Set $\pi_h(s) = \argmax_a Q_h(s,a)$
  \ENDFOR 
  \STATE \textbf{Return } $\pi := \{ \pi_0,\cdots,\pi_{H-1} \}$
\end{algorithmic}
\end{algorithm}

\section{Analysis}

\label{sec:analysis}

We start by specifying the discriminator class $\Fcal_h$ constructed using the representations from $\Phi_{h+1}$. 
Define two sets of discriminators $\Fcal^{(1)}_h$ and $\Fcal^{(2)}_h$ as
\begin{align*} 
    \Fcal^{(1)}_h &= \Big\{f(s) :  \EE_{a\sim U(\Acal)}\big[\phi(s,a)^\top \theta-\phi'(s,a)^\top \theta'\big]\phi,\phi'\in\Phi_{h+1}\; \max(\|\theta\|_\infty,\|\theta'\|_\infty){\leq} 1\Big\},\\
   \Fcal^{(2)}_h &= \Big\{ f(s){:} \max_a \Big( {\tfrac{r_{h+1}(s,a) + w^\top \phi(s,a)\wedge 2}{2H+1} + \phi(s,a)^\top w'}\Big) \phi\in \Phi_{h+1};  \;\|w\|_\infty\leq c,\|w'\|_\infty\leq 1 \Big\}\\
\end{align*}
respectively, where $c \in \mathbb{R}^+$ is some positive constant that we will specify in the main theorem.
We let our discriminator class be $\Fcal_h = \Fcal^{(1)}_h\cup \Fcal^{(2)}_h$. Note that $\Fcal_h$ contains linear functions of the features in $\Phi_{h+1}$ which makes bounding the statistical complexity (e.g., covering number) of $\Fcal_h$ straightforward.
Using this definition of $\Fcal_h$ in~\pref{alg:main}, we have the following guarantee when the environment is a Block MDP:
\begin{theorem}[PAC bound of \algname]\label{thm:main}
Consider a Block MDP (\pref{def:block_mdp}) and assume that $\phi_h^\star \in \Phi_h$ for all $h\in [H]$. Fix $\delta, \epsilon\in(0,1)$, and let $\hat\pi$ be a uniform mixture of $\hat\pi^0,...,\hat\pi^{N-1}$. By setting the parameters as
\begin{align*}
    \alpha_n = \tilde\Theta\Big((nd^5)^{\frac{1}{4}}|\Acal|\ln\tfrac{|\Phi|n}{\delta}\Big),~\lambda_n = \Theta\Big(d\ln\tfrac{|\Phi|n}{\delta}\Big),
    T_n = \left\lceil\sqrt{\frac{n}{d\ln(\frac{|\Phi|}{\delta})}}\right\rceil, 
    \ell_n = \Theta\left(d^{2} \sqrt{\frac{\ln(\frac{|\Phi|}{\delta})}{n}}\right), ~ c  = \frac{\alpha_N}{\sqrt{\lambda_N}}
    \end{align*}
with probability at least $1-\delta$, we have $V^{\pi^{\star}}_{P^{\star}}-V^{\hat\pi}_{P^{\star}}\leq \epsilon$, after at most
\begin{equation*}
    H\cdot N\leq O\left(\frac{H^9|\Acal|^{14}|\Zcal|^{8}\ln(H|\Acal||\Zcal||\Phi|/\delta\epsilon)}{\epsilon^{4}}\right)
\end{equation*}
episodes of interaction with the environment.
\end{theorem}
\pref{thm:main} certifies a polynomial sample complexity of \algname for all Block MDPs. Furthermore, the bound on $T$ means that our overall computational complexity is favorable as long as one iteration of \textsc{RepLearn} can be efficiently executed. 
Compared to other efficient methods for solving Block MDPs, our result is fully general and places no additional assumptions on the reachability or stochasticity in the underlying dynamics. Compared with statistically superior methods like \olive, our algorithm is amenable to practical implementation as we demonstrate through our experiments. 
Finally, compared with prior works which are primarily model-based, our model-free guarantees do not require modeling the emission process of observations in a Block MDP. Compared with the prior model-free representation learning guarantee of~\citet{modi2021model}, we do not require reachability of latent states. 

Our next result shows that our \algname framework is flexible enough that a variation of the algorithm can solve the reward-free exploration problem. To perform reward-free exploration, we simply run \algname by setting reward $r_h(s,a) = 0, \forall h, s,a$. After the exploration phase, once we are given a new reward function that is linear in the ground truth feature $\phi^\star$, we can find a near optimal policy via LSVI. The detailed procedure is described in  \pref{alg:reward_free} in appendix. 
\begin{theorem}[Reward-free Exploration]\label{thm:reward_free}
Consider a Block MDP and assume that $\phi_h^\star \in \Phi_h$ for all $h\in [H]$. Fix $\delta, \epsilon\in(0,1)$, by setting the parameters as in \pref{thm:main}, given $N$ episodes in the exploration phase,
with probability at least $1-\delta$, for all reward function $r(s,a)$ that is linear in $\phi^\star$, \pref{alg:reward_free} finds a policy $\hat\pi$ in the planning phase such that
\begin{align*}
    V^{\pi^\star}_{P^\star,r}-V^{\hat\pi}_{P^\star,r}\leq O\left(H^{5/2}|\Acal|^{1/2}d\log(|\Phi|/\delta)^{1/4}N^{-1/4}\right).
\end{align*}
\end{theorem}

\subsection{The Representation Learning Oracle's Guarantee}

The following lemma gives a bound on the reconstruction error with the learned features at any iteration of \algname, under the data distribution used for representation learning. 

\begin{lemma}[\textsc{RepLearn} guarantee]\label{lem:rep_learn_informal}
Consider parameters defined in \pref{thm:main} and time step $h$. Denote $\rho$ as the joint  distribution for $(s,a,s')$ in the dataset $\Dcal$ of size $n$.  Let $\hat\phi = \textsc{RepLearn}(\Dcal,\Phi_h, \Fcal_h, \lambda_n, T_n, \ell_n)$. Denote $\widehat{w}_f = \argmin_w \sum_{s,a,s'\in\Dcal}(w^\top \hat\phi(s,a) - f(s'))^2 + \lambda_n \|w\|_2^2$~for any $f:\Scal\to\mathbb{R}$. Then, with probability at least $1-\delta$:
\begin{align*}
    &\max_{f\in\Fcal_h} \EE_{s,a\sim \rho} \left[\left( \hat w_f^\top \hat\phi(s,a) - \EE_{s'\sim P^\star_h(s,a)}[f(s')]\right)^2\right]
    \leq \zeta_n =  O\left(d^2\sqrt{\log\left(n|\Phi| / \delta \right)/n}\right).
\end{align*}
\end{lemma}

Note that the guarantee above holds under the data distribution $\rho$, which might not be sufficiently exploratory at intermediate iterations of the algorithm. 
However, it shows that our representation is always good for the entire set of discriminators on the available data distribution, and any deficiencies of the representation can only be addressed by improving the data coverage in subsequent iterations. 
As mentioned earlier, the number of iterations $T$ of \textsc{RepLearn} before this guarantee is achieved is bounded by the setting of $T$ in Theorem~\ref{thm:main}.

Indeed the main challenge in the proof of Lemma~\ref{lem:rep_learn_informal} is to establish that \textsc{RepLearn} terminates in a small number of iterations $T$. 
While this need not be true for a general MDP, the low-rank property which implies that $\EE[f(s')|s,a] = w^{*\top}_f \phi^\star(s,a)$ allows us to achieve this. We do so by following a similar elliptical potential argument proposed in \moffle. The main difference here is that we need to incorporate ridge regularization to make it consistent with the ridge regression used in LSVI. Note that ridge linear regression is important in implementation as well since it has closed-form solution which enables direct gradient ascent optimization on discriminators.  While the key ideas follow from \moffle, there are technical differences in our analysis, as well as room to additionally leverage the Block MDP structure, which results in an improved bound. To contrast, \moffle requires $\tilde O(d^{7/3}n^{-1/3})$ samples (c.f. Lemma 5 in \citet{modi2021model}).

\subsection{Proof Sketch}

We now give a high-level proof sketch for Theorem~\ref{thm:main}. Even though our algorithm is model-free, it would be helpful to leverage an equivalent non-parametric model-based interpretation which has also been discussed before in~\citet{parr2008analysis,jin2020provably,lykouris2021corruption}. We describe this model-based perspective for completeness here, before discussing how we utilize it.

\paragraph{A Model-based Perspective:}
Let us focus on an iteration $n$. For notational simplicity, we drop the superscript $n$ here. Denote the learned features as $\hat\phi_h$ for $h\in [H]$. Define the non-parametric model $\hat P_h$ as follows:
\begin{align*}
\hat P_{h}(s' \vert s,a)&= \hat\phi_h(s,a)^{\top} \Sigma_h^{-1} \sum_{\tilde{s},\tilde a,\tilde{s}'\in \Dcal_h \cup \Dcal'_h }  \hat\phi_h(\tilde s,\tilde a)\one(s'=\tilde{s'})
= \frac{N_{\Dcal_h \cup \Dcal'_h}(\hat\phi(s,a),s')}{N_{\Dcal_h \cup \Dcal'_h}(\hat\phi(s,a))+\lambda},
\end{align*}
where $\Sigma_h :=  \sum_{s,a\in \Dcal_h \cup \Dcal'_h}  \hat\phi_h(s,a)\hat\phi_h(s,a)^{\top} + \lambda I$, $\one(s = s')$ is the indicator function, and $N_{\Dcal}(x)$ is the number of occurrences of $x$ in the dataset $\Dcal$. The intuition behind this model is that we learn the model via multi-variate linear regression from representation $\hat\phi_h(s,a)$ to the vector $e_{s'}$. 
In particular, the conditional expectation $\EE_{\hat{P}_h}[f(s')|s,a]$ is $w^\top\hat\phi_h(s,a)$, where $w$ minimizes $\sum_{(s,a,s')\in\Dcal_h \cup\Dcal'_h}(w^\top\hat\phi_h(s,a) - f(s'))^2$, underscoring the connection between this model definition and our LSVI planner. The second equality is specific to the block structure of our features and does not hold in general low-rank MDPs. 
Interestingly, this is perhaps the only ingredient of our analysis which does not generalize beyond Block MDPs. 
Notice that $\hat{P}_h$ is not a normalized conditional distribution, in the sense that $\sum_{s'\in \Scal}\hat{P_h}(s'|s,a)\neq1$ for $\lambda>0$. Nevertheless, it is still non-negative and we can still define the occupancy measures and value functions inductively as follows:
\begin{align*}    
    &d^{\pi}_{\hat P; 0}(s,a) = d_0(s) \pi(a|s);\quad V^\pi_{(\hat P,r);H}(s) = 0;
    &d^{\pi}_{\hat P; h+1}(s,a) = \sum_{\tilde{s},\tilde{a}} d^{\pi}_{\hat P; h}(\tilde{s},\tilde{a}) \hat{P}_h(s | \tilde{s},\tilde{a}) \pi(a| s);\\
    &V^\pi_{(\hat P,r);h}(s) = \EE_{a\sim\pi(s)}\left[r(s,a) + \hat P_h(\cdot\vert s,a)^\top V^\pi_{(\hat P,r);h+1}\right].
\end{align*}
where $P(\cdot | s,a)^\top f$ is in short of $\sum_{s'\in\Scal} P(s' | s,a)  f(s')$.
As we will see, these constructions are sufficient for us to carry out a model-based analysis on our model-free algorithm.

Using our constructed representation $\{ \hat \phi^n_h\}_{h\in [H]}$ and bonus $\{ \hat b^n_h\}_{h\in [H]}$, we can prove the following near-\emph{optimism} claim. Note that the optimism only holds at the initial state distribution, in contrast to the stronger versions that hold in a point-wise manner. Note that from here, our analysis significantly departs from the prior Block MDP works' analysis and the analysis of \moffle which rely on a reachability assumption. This part of our proof leverages some ideas in the analysis of a recent model-based representation learning algorithm \repucb~\citep{uehara2021representation}.

For $\hat\pi:=  \textsc{LSVI}\left(\{ r_h + \hat b_h \}, \{\hat\phi_h\}, \{\Dcal_h + \Dcal'_h\}, \lambda_n\right) $, it is not hard to see that $\hat\pi$ is indeed the optimal policy for the MDP model with transitions $\{ \hat P_h\}$ and rewards $\{ r_h + \hat b_h \}$, i.e., $\hat\pi = \argmax_{\pi'}V^{\pi'}_{\hat P,r+\hat b}$ from our construction of $\hat P_h$.
The following lemma shows that $V^{\pi}_{\hat P, r + \hat b}$ almost upper bounds $V^{\pi^\star}_{P^\star, r}$, i.e., we achieve almost optimism. 
\begin{lemma}[Optimism]\label{lem:opt_main}
Using the settings of \pref{thm:main},
with probability $1-\delta$, we have for all iterations $n\in [N]$,
$
  V^{\hat\pi^n}_{\hat P^n,r+\hat b^n}- V^{\pi^\star}_{P^{\star},r}\geq -\sigma_n,
  $
where $\sigma_n = O(|\Acal|^{1/2}d\left(\log(n|\Phi|/\delta)/n\right)^{1/4})$. 
\end{lemma}

Optimism allows us to upper bound policy regret as follows. For $\hat\pi^n$, conditioned on the optimism event and via the standard simulation lemma (which also works on the unnormalized transitions $\hat P^n$), we  have:
\begin{align}\label{eq:regret_simulation}
 V^{\pi^\star}_{P^\star, r}  & - V^{\hat\pi^n}_{P^\star, r}
\leq V^{\hat\pi^n}_{\hat P^n, r + \hat b^n} - V^{\hat\pi^n}_{P^\star, r} + \sigma_n =  \sum_{h=0}^{H-1} \EE_{d^{\hat\pi^n}_h} [ \hat b^n_h(s,a) +  \underbrace{ (\hat{P}^n_h-P^\star_h)(s,a)^{\top} \hat V_{h+1}^{\hat\pi^n}}_{\xi^n_h(s,a)}]+ \sigma_n
\end{align} 
where we denote $\hat V^{\hat\pi^n}_{h+1}(\cdot )$ as the expected total reward  function of $\hat\pi^n$ starting from time step $h+1$, under model $\hat P^n$ and reward $r+\hat b^n$. 

The rest of the proof is to directly control the expectation of $\hat b^n_h+\xi^n_h$ under $d^{\hat\pi^n}_h$, which is done in \pref{lem:tsbt2}. \pref{lem:tsbt2} uses a key property of the low-rank MDP which is that for any $f$, $\EE_{s,a\sim d^{\pi}_h} f(s,a)$ can be written in a bilinear form:  $$\left\langle \EE_{\tilde s,\tilde a\sim d^{\pi}_{h-1}} \phi^\star_{h-1}(\tilde s,\tilde a), \int_{s} \mu^\star_{h-1}(s)\EE_{a\sim \pi(s)} f(s,a) ds  \right\rangle.$$ This ``one step back" trick (i.e., moving from $h$ to $h-1$) leverages the bilinear structure in $P^\star_{h-1}$, and was first used in \citet{agarwal2020flambe}.  Denote $\gamma_h^n(s,a) = \sum_{i=0}^{n-1} d^{\pi^i}_h(s,a) / n$ as the mixture state-action distribution, and $\Sigma_{\gamma^n_h, \phi^\star} = n\EE_{s,a\sim \gamma^n_h} \phi_h^\star(s,a)(\phi_h^\star(s,a))^\top + \lambda I$ as the regularized covariance matrix under the ground truth representation $\phi^\star$. We can further upper bound the above bilinear form via Cauchy-Schwart under norm induced by $\Sigma_{\gamma^n_h, \phi^\star}$: 
\begin{align*}
& \left\langle \EE_{\tilde s,\tilde a\sim d^{\pi}_{h-1}} \phi^\star_{h-1}(\tilde s,\tilde a), \int_{s} \mu^\star_{h-1}(s)\EE_{a\sim \pi(s)} f(s,a) ds  \right\rangle \\
\leq&  \EE_{\tilde s,\tilde a\sim d^{\pi}_{h-1}} \left\| \phi^\star_{h-1}(\tilde s,\tilde a) \right\|_{\Sigma^{-1}_{\gamma^n_{h-1},\phi^\star}} 
\times \left\| \int_{s} \mu^\star_{h-1}(s)\EE_{a\sim \pi(s)} f(s,a) ds\right\|_{\Sigma^n_{\gamma^n_{h-1}, \phi^\star}}.
\end{align*} The first term on the RHS of the above inequality is related to the classic elliptical potential function defined with the true $\phi^\star$. The second term on the RHS can be further converted to the expectation of $f^2$ under the training data distribution (i.e., the distribution over $\Dcal_h$), which is controllable when $f := \hat b^n_h + \xi^n_h$. This step of distribution change from $d^\pi$ to the training data distribution is the key to directly upper bounding regret. 
The detailed proof for \pref{thm:main} can be found in Appendix \ref{sec:app_proof_main}.

\section{Experiments}
\label{sec:exp}

We test \algname on MDPs motivated by the rich observation combination lock benchmark created by \citep{misra2019kinematic}, which contains latent states but the observed states are high-dimensional continuous observations. We ask:

\begin{enumerate}[leftmargin=*,itemsep=0pt]
\item Is \algname~more sample efficient than \homer on their own benchmark? 
\item Can \algname solve an MDP where the block structure does not hold, e.g., low-rank MDP? 
\item If the MDP happens to have dense rewards that make it easy for policy optimization, can \algname leverage that? 
\end{enumerate}
In short, our experiments provide affirmative answers to all three questions above. Note that prior algorithms such as \homer and PCID cannot be applied to low-rank MDPs, and are not able to leverage dense reward structures, since they require a reward-free exploration phase. In what follows we describe the experiment setup.

\paragraph{Reproducibility} Our code can be find at \url{https://github.com/yudasong/briee}. We include experiment details in Appendix~\ref{sec:app:exp}.

\paragraph{The Environment} We evaluate our algorithm on the \textit{diabolical combination lock} (comblock) problem (Fig. \ref{exp:fig:comblock_pic}), which has horizon $H$ and 10 actions. At each step $h$, there are three latent states $z_{i;h}$ for $i \in \{0,1,2\}$. We call $z_{i;h}$ for $i\in \{0,1\}$ good states and $z_{2;h}$ bad states. 
For each $z_{i;h}$ with $i\in\{0,1\}$, we randomly pick an action $a^*_{i;h}$ from the 10 actions. 
While at $z_{i;h}$ for $i\in\{0,1\}$, taking action $a^*_{i;h}$ transits the agent to $z_{0;h+1}$ and $z_{1;h+1}$ with equal probability. 
Taking other actions transits the agent to $z_{2;h+1}$ deterministically. At $z_{2;h}$, regardless what action the agent takes, it will transit to $z_{2;h+1}$. For reward function, we give reward $1$ at state $z_{i;H}$ for $i\in \{0,1\}$, i.e., good states at $H$ have reward 1. With probability 0.5, the agent will also receive an anti-shaped reward $0.1$ from transiting from a good state to a bad state.  We have reward zero for any other states and transitions.
The observation $s$ has dimension $2^{\ceil{\log(H+4)}}$, created by concatenating the one-hot vectors of latent state $z$ and the one-hot vectors of horizon $h$, followed by adding noise sampled from $\mathcal{N}(0,0.1)$ on each dimension, appending 0 if necessary, and multiplying with a Hadamard matrix. The initial state distribution is uniform over $z_{i;0}, i\in \{0,1\}$. Note that the optimal policy picks the special action $a^*_{i;h}$ at every $h$. Once the agent hits a bad state, it will stay in bad states for the entire episode, missing the large reward signal at the end. This is an extremely challenging exploration problem, since a  uniform random policy will only have $10^{-H}$ probability of hitting the goals. 

\paragraph{\algname Implementation} Here we provide the details for implementing \pref{alg:rep_learn}. For features $\phi_h$ we have: $\phi_h(s,a) = \text{softmax}(A_h s / \tau) \otimes \one_{a}$, where $A_h \in \mathbb{R}^{3 \times \text{dim}_s}$, $\tau$ is the temperature, and $\one_{a} \in \{0,1\}^{|\Acal|}$ is the one-hot indicator vector. This design of decoder follows from \homer, for the purpose of a fair comparison. For discriminator we use two-layer neural network with tanh activation. In each step of \discstep, we perform gradient ascent on $\tilde{\phi}$ and $f$ jointly. Similarly for \featurestep, we perform gradient descent on $A_h$.  

\paragraph{Baselines} In the following experiments, in addition to \homer, we compare with empirical deep RL baselines Proximal Policy
Optimization (PPO) \citep{schulman2017proximal} and Random Network Distillation (PPO+RND)  \citep{burda2018exploration}. We also include LSVI-UCB  \citep{jin2020provably} with ground-truth features (i.e., it is an aspirational baseline with access to the latent state information), and LSVI-UCB with Random Fourier Features, i.e., RFF directly on top of $(s,a)$ \citep{rahimi2007random} (LSVI-UCB+RFF) as baselines.

\begin{figure*}[ht!]
     \centering
     \begin{subfigure}[b]{0.33\textwidth}
         \centering
         \includegraphics[width=\textwidth]{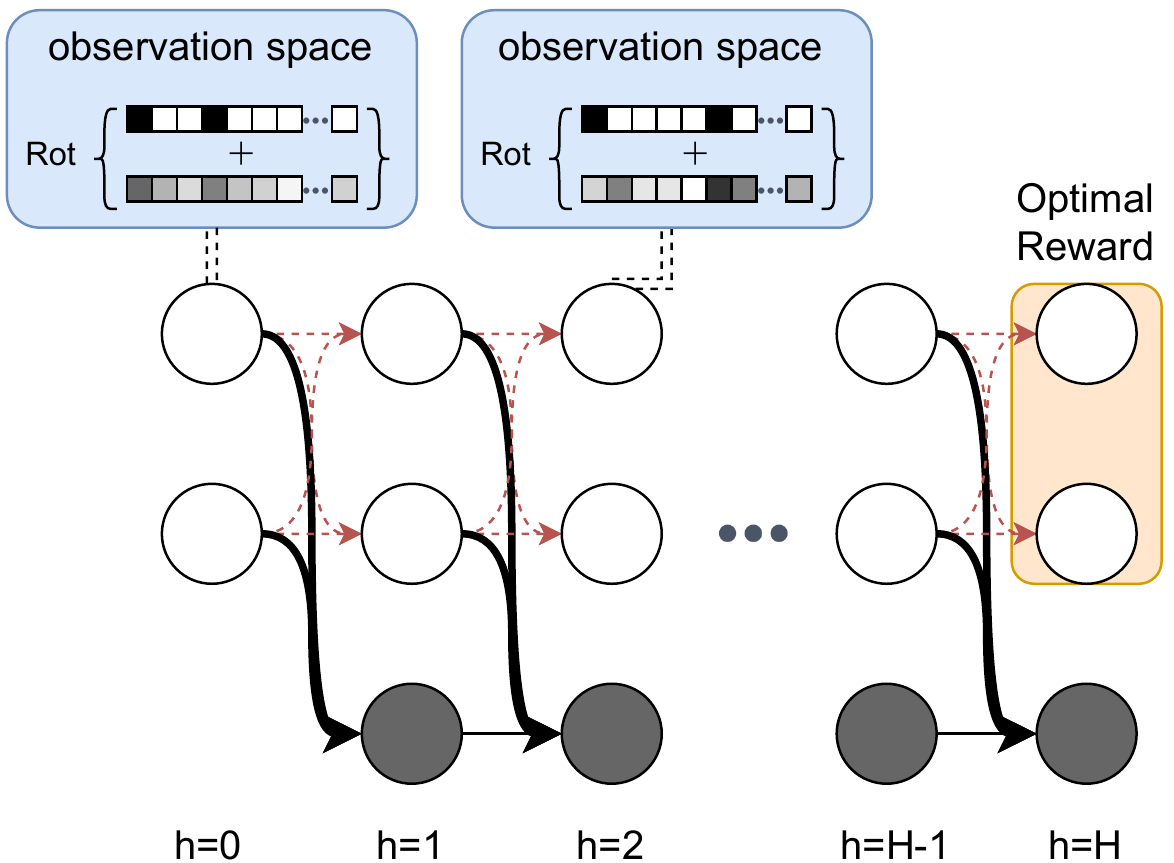}
         \caption{Rich-obs Comblock w/ sparse reward.}
         \label{exp:fig:comblock_pic}
     \end{subfigure}
     \hfill
     \begin{subfigure}[b]{0.33\textwidth}
         \centering
         \includegraphics[width=\textwidth]{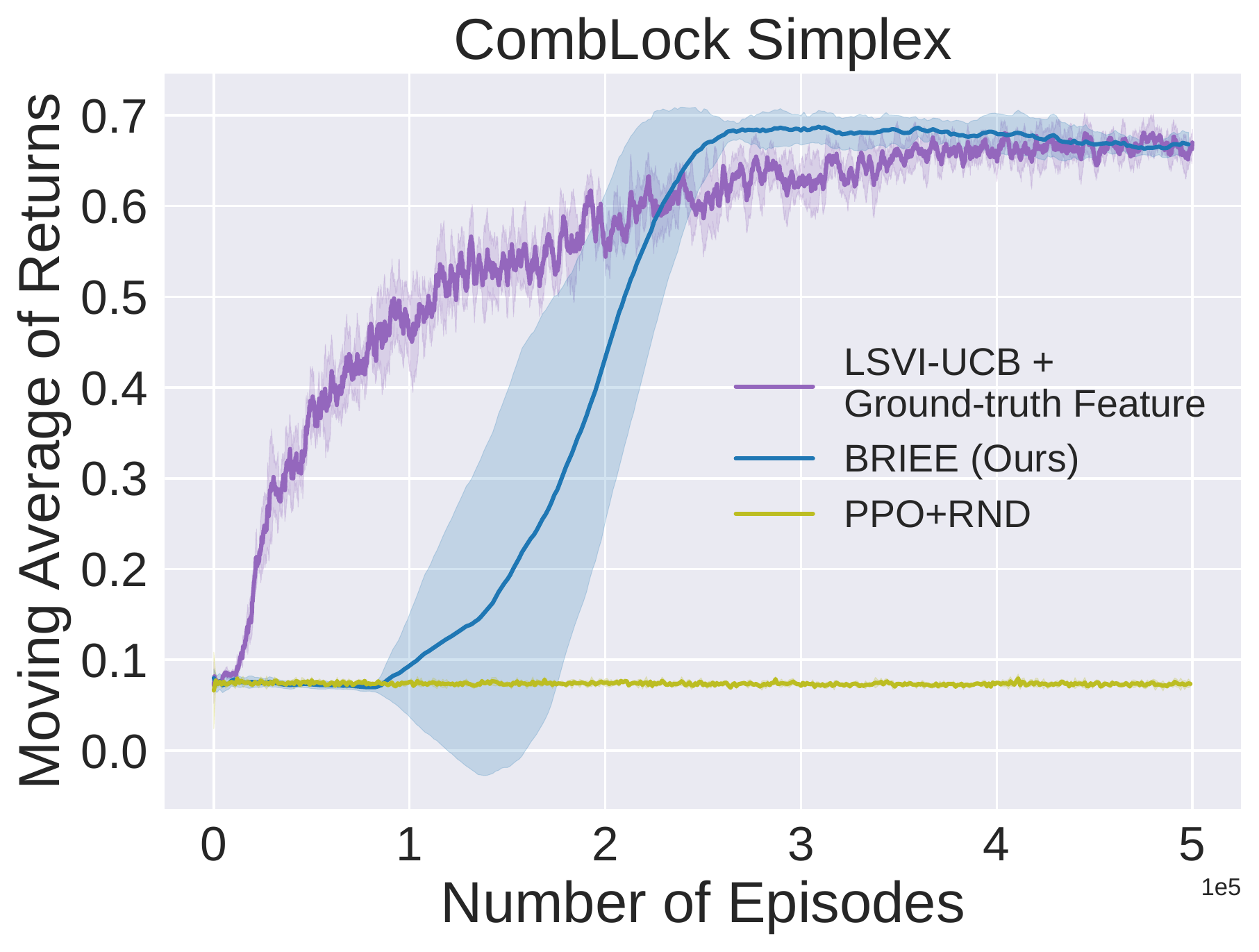}
         \caption{Comblock with simplex features.}
         \label{exp:fig:simplex}
     \end{subfigure}
     \hfill
     \begin{subfigure}[b]{0.32\textwidth}
         \centering
         \includegraphics[width=\textwidth]{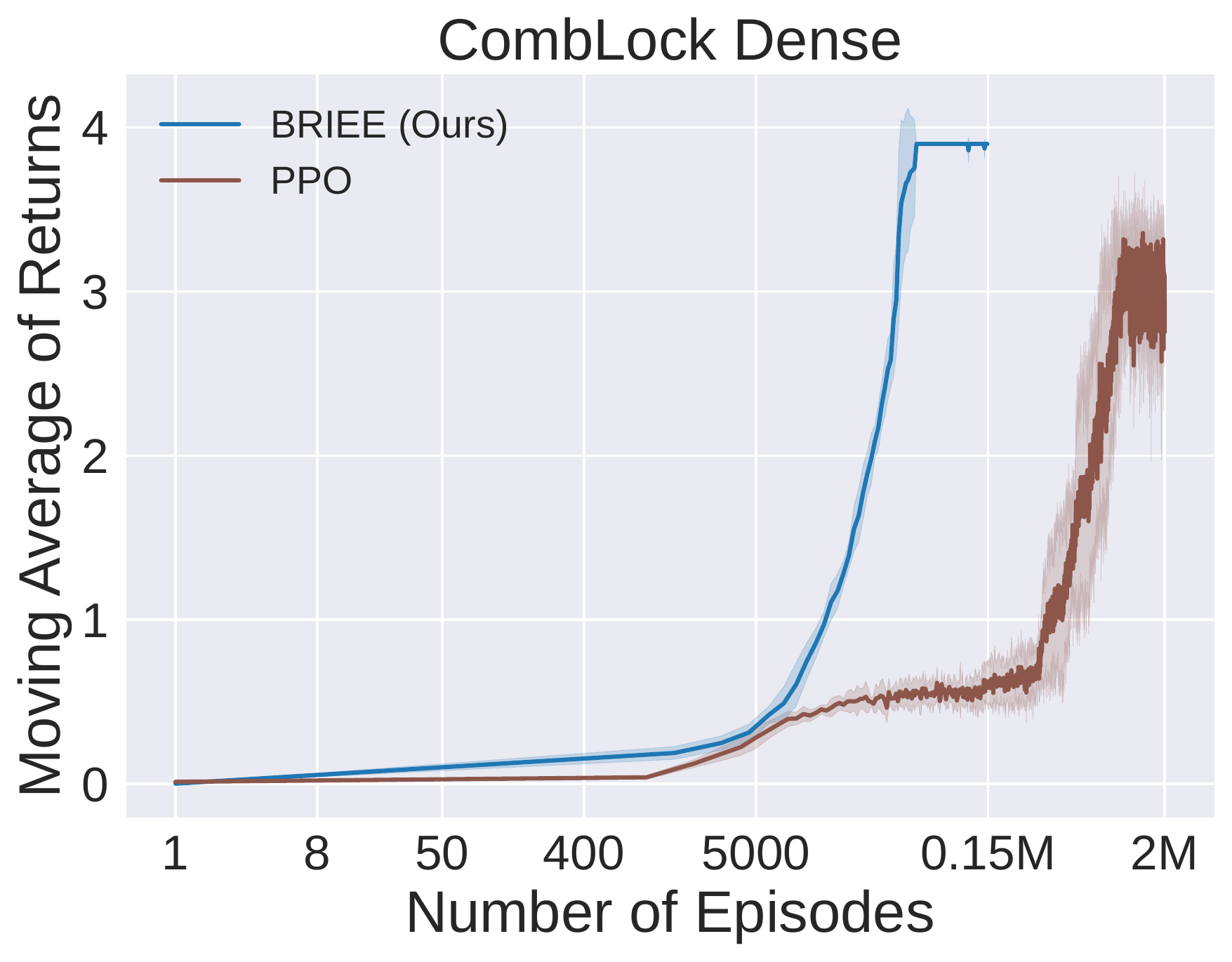}
         \caption{Comblock dense reward.}
         \label{exp:fig:dense}
     \end{subfigure}
        \caption{\textbf{(a):} Visualization of comblock. The blue area denotes the (infinitely many) possible observations from one latent state. The top vector denotes the one-hot vector and the bottom vector denotes the noise.  $\text{Rot}\{\cdot\}$ denotes multiplying with the Hadamard matrix. The dark arrows denote transiting from good states (white) to bad states (gray). Once the agent transits to a bad state, it stays in bad states for the entire episode thus missing the goal. Note that in order to reach the goal (orange), one has to stay in good states which can only be achieved by picking the right action (red) at every time step (thus the name \emph{combination lock}). 
        \textbf{(b):} Moving average of the evaluation returns for comblock with simplex feature and $H=30$. The converged return of LSVI-UCB with ground truth feature suggests the optimal value. \textbf{(c):} Moving average of evaluation returns for comblock with dense reward and $H=30$. The x-axis is scaled in natural log.}
\label{exp:fig}
\end{figure*}

\paragraph{Comparison with \homer} In this experiment we focus on the comblock environment. We test \algname and baselines for different horizon values. We record the number of episodes that each algorithm needs to identify the optimal policy (i.e., a policy that achieves the optimal total reward 1).  
The results are shown in Fig.~\ref{tbl:results}a. We compare with \homer, PPO+RND and LSVI-UCB+RFF. We reuse the results of \homer and PPO-RND from \citep{misra2019kinematic}. For \algname, we run 5 random seeds and plot the confidence interval within 1 std. Note that LSVI-UCB + RFF can only solve $H$ up to $12$, indicating that  simply lifting linear MDPs to RBF kernel space is not enough to capture the underlying nonlinearity in our problem. PPO+RND can solve up to $H = 25$. Both \homer and \algname solve up to $H = 100$ with \algname being more sample efficient.

\paragraph{Visualization of decoders} We visualize the outputs of decoders in Fig.~\ref{app:fig:decoder} in Appendix~\ref{app:visualization_latent}. In short, our decoders successfully decode latent states (up to a permutation).

\paragraph{Comblock with Simplex Feature}  Here we extend the above environment beyond Block MDP. Instead of decoding $s$ to a unique latent state, we modify the ground truth decoder to make it stochastic. Given $s,a$, the ground truth decoder maps $s$ to a distribution over latent state space, then a latent state is sampled from this distribution, and then together with $a$, it transits to the next latent state, followed by generating the next $s'$ from the emission distribution. This is not a block MDP anymore, and indeed it is a low-rank MDP (i.e., $\phi^\star_h(s,a)$ is not one hot, and it is from $\Delta(\Zcal\times\Acal)$). Note that \homer provably fails in this example. 
We show the results for $H=30$  in Fig.~\ref{exp:fig:simplex}. While PPO+RND completely fails, \algname matches the return from  LSVI-UCB with the features $\phi^{\star}$ (that are unknown to \algname and PPO+RND). 

\paragraph{Dense Reward Comblock} We also test if \algname can leverage dense reward signals to further speed up learning.  We modify the reward as follows. Instead of getting an anti-shaped reward from transition to a bad state, we get a positive reward every time one transits to a good state. Thus a greedy algorithm that collects one step immediate reward should be able to reach the final goals. In Fig.~\ref{exp:fig:dense} we show the mean evaluation returns in the $H=30$ dense-reward comblock environment. Compared with PPO -- a greedy policy gradient algorithm, \algname can consistently reach the optimal value (3.9) and uses orders of magnitude fewer samples. Note that compared to the results in Fig.~\ref{tbl:results}b where reward is sparse and anti-shaped, \algname is able to solve the problem two times faster, indicating that it indeed can leverage the dense reward structure for further speedup. Note that \homer cannot leverage such informative reward signals and will have to perform reward-free exploration regardless, which is a huge waste of samples.

\section{Conclusion}

In this paper, we present a new algorithm \algname that provably solves block MDPs. Unlike prior block MDP algorithms, \algname does not require any reachability assumption and can directly optimize the given reward function. Unlike \textsc{Flambe} and \textsc{Rep-UCB}, \algname is model-free which means that it is more suitable for tasks where states are high dimensional objects. Experimentally, on the benchmarks motivated by \homer, we show our approach is more sample efficient than \homer and other empirical RL baselines. We also demonstrate the efficacy of our approach on a low-rank MDP where the block structure does not hold.

\section*{Acknowledgement}
Xuezhou Zhang and Mengdi Wang acknowledge support by NSF grants IIS-2107304, CMMI-1653435, AFOSR grant and ONR grant 1006977. Masatoshi Uehara is partly supported by MASASON Foundation.

\bibliography{ref}

\begin{thebibliography}{41}
\providecommand{\natexlab}[1]{#1}
\providecommand{\url}[1]{\texttt{#1}}
\expandafter\ifx\csname urlstyle\endcsname\relax
  \providecommand{\doi}[1]{doi: #1}\else
  \providecommand{\doi}{doi: \begingroup \urlstyle{rm}\Url}\fi

\bibitem[Agarwal et~al.(2020{\natexlab{a}})Agarwal, Henaff, Kakade, and
  Sun]{agarwal2020pc}
Alekh Agarwal, Mikael Henaff, Sham Kakade, and Wen Sun.
\newblock Pc-pg: Policy cover directed exploration for provable policy gradient
  learning.
\newblock \emph{Advances in neural information processing systems},
  2020{\natexlab{a}}.

\bibitem[Agarwal et~al.(2020{\natexlab{b}})Agarwal, Kakade, Krishnamurthy, and
  Sun]{agarwal2020flambe}
Alekh Agarwal, Sham Kakade, Akshay Krishnamurthy, and Wen Sun.
\newblock Flambe: Structural complexity and representation learning of low rank
  mdps.
\newblock \emph{Advances in neural information processing systems},
  2020{\natexlab{b}}.

\bibitem[Antos et~al.(2008)Antos, Szepesv{\'a}ri, and Munos]{antos2008learning}
Andr{\'a}s Antos, Csaba Szepesv{\'a}ri, and R{\'e}mi Munos.
\newblock Learning near-optimal policies with bellman-residual minimization
  based fitted policy iteration and a single sample path.
\newblock \emph{Machine Learning}, 71\penalty0 (1):\penalty0 89--129, 2008.

\bibitem[Bellemare et~al.(2016)Bellemare, Srinivasan, Ostrovski, Schaul,
  Saxton, and Munos]{bellemare2016unifying}
Marc Bellemare, Sriram Srinivasan, Georg Ostrovski, Tom Schaul, David Saxton,
  and Remi Munos.
\newblock Unifying count-based exploration and intrinsic motivation.
\newblock \emph{Advances in neural information processing systems},
  29:\penalty0 1471--1479, 2016.

\bibitem[Bellemare et~al.(2019)Bellemare, Dabney, Dadashi, Ali~Taiga, Castro,
  Le~Roux, Schuurmans, Lattimore, and Lyle]{bellemare2019geometric}
Marc Bellemare, Will Dabney, Robert Dadashi, Adrien Ali~Taiga, Pablo~Samuel
  Castro, Nicolas Le~Roux, Dale Schuurmans, Tor Lattimore, and Clare Lyle.
\newblock A geometric perspective on optimal representations for reinforcement
  learning.
\newblock \emph{Advances in neural information processing systems},
  32:\penalty0 4358--4369, 2019.

\bibitem[Burda et~al.(2018)Burda, Edwards, Storkey, and
  Klimov]{burda2018exploration}
Yuri Burda, Harrison Edwards, Amos Storkey, and Oleg Klimov.
\newblock Exploration by random network distillation.
\newblock In \emph{International Conference on Learning Representations}, 2018.

\bibitem[Dann et~al.(2018)Dann, Jiang, Krishnamurthy, Agarwal, Langford, and
  Schapire]{dann2018oracle}
Christoph Dann, Nan Jiang, Akshay Krishnamurthy, Alekh Agarwal, John Langford,
  and Robert~E Schapire.
\newblock On oracle-efficient pac rl with rich observations.
\newblock In \emph{Proceedings of the 32nd International Conference on Neural
  Information Processing Systems}, pages 1429--1439, 2018.

\bibitem[Du et~al.(2019)Du, Krishnamurthy, Jiang, Agarwal, Dudik, and
  Langford]{du2019provably}
Simon Du, Akshay Krishnamurthy, Nan Jiang, Alekh Agarwal, Miroslav Dudik, and
  John Langford.
\newblock Provably efficient rl with rich observations via latent state
  decoding.
\newblock In \emph{International Conference on Machine Learning}, pages
  1665--1674. PMLR, 2019.

\bibitem[Du et~al.(2021)Du, Kakade, Lee, Lovett, Mahajan, Sun, and
  Wang]{du2021bilinear}
Simon~S Du, Sham~M Kakade, Jason~D Lee, Shachar Lovett, Gaurav Mahajan, Wen
  Sun, and Ruosong Wang.
\newblock Bilinear classes: A structural framework for provable generalization
  in rl.
\newblock \emph{International Conference on Machine Learning}, 2021.

\bibitem[Duan et~al.(2019)Duan, Ke, and Wang]{duan2018state}
Yaqi Duan, Tracy Ke, and Mengdi Wang.
\newblock State aggregation learning from markov transition data.
\newblock \emph{Advances in Neural Information Processing Systems},
  32:\penalty0 4486--4495, 2019.

\bibitem[Feng et~al.(2020)Feng, Wang, Yin, Du, and Yang]{feng2020provably}
Fei Feng, Ruosong Wang, Wotao Yin, Simon~S Du, and Lin Yang.
\newblock Provably efficient exploration for reinforcement learning using
  unsupervised learning.
\newblock \emph{Advances in Neural Information Processing Systems}, 33, 2020.

\bibitem[Foster et~al.(2021)Foster, Rakhlin, Simchi-Levi, and
  Xu]{foster2020instance}
Dylan~J Foster, Alexander Rakhlin, David Simchi-Levi, and Yunzong Xu.
\newblock Instance-dependent complexity of contextual bandits and reinforcement
  learning: A disagreement-based perspective.
\newblock \emph{{Conference on learning theory}}, 2021.

\bibitem[Jiang et~al.(2017)Jiang, Krishnamurthy, Agarwal, Langford, and
  Schapire]{jiang2017contextual}
Nan Jiang, Akshay Krishnamurthy, Alekh Agarwal, John Langford, and Robert~E
  Schapire.
\newblock Contextual decision processes with low bellman rank are
  pac-learnable.
\newblock In \emph{International Conference on Machine Learning}, pages
  1704--1713. PMLR, 2017.

\bibitem[Jin et~al.(2020)Jin, Yang, Wang, and Jordan]{jin2020provably}
Chi Jin, Zhuoran Yang, Zhaoran Wang, and Michael~I Jordan.
\newblock Provably efficient reinforcement learning with linear function
  approximation.
\newblock In \emph{Conference on Learning Theory}, pages 2137--2143. PMLR,
  2020.

\bibitem[Jin et~al.(2021)Jin, Liu, and Miryoosefi]{jin2021bellman}
Chi Jin, Qinghua Liu, and Sobhan Miryoosefi.
\newblock Bellman eluder dimension: New rich classes of rl problems, and
  sample-efficient algorithms.
\newblock \emph{Advances in neural information processing systems}, 2021.

\bibitem[Laskin et~al.(2020)Laskin, Srinivas, and Abbeel]{laskin2020curl}
Michael Laskin, Aravind Srinivas, and Pieter Abbeel.
\newblock Curl: Contrastive unsupervised representations for reinforcement
  learning.
\newblock In \emph{International Conference on Machine Learning}, pages
  5639--5650. PMLR, 2020.

\bibitem[Lykouris et~al.(2021)Lykouris, Simchowitz, Slivkins, and
  Sun]{lykouris2021corruption}
Thodoris Lykouris, Max Simchowitz, Alex Slivkins, and Wen Sun.
\newblock Corruption-robust exploration in episodic reinforcement learning.
\newblock In \emph{Conference on Learning Theory}, pages 3242--3245. PMLR,
  2021.

\bibitem[Machado et~al.(2020)Machado, Bellemare, and Bowling]{machado2020count}
Marlos~C Machado, Marc~G Bellemare, and Michael Bowling.
\newblock Count-based exploration with the successor representation.
\newblock In \emph{Proceedings of the AAAI Conference on Artificial
  Intelligence}, volume~34, pages 5125--5133, 2020.

\bibitem[Misra et~al.(2019)Misra, Henaff, Krishnamurthy, and
  Langford]{misra2019kinematic}
Dipendra Misra, Mikael Henaff, Akshay Krishnamurthy, and John Langford.
\newblock Kinematic state abstraction and provably efficient rich-observation
  reinforcement learning.
\newblock \emph{arXiv preprint arXiv:1911.05815}, 2019.

\bibitem[Misra et~al.(2020)Misra, Henaff, Krishnamurthy, and
  Langford]{misra2020kinematic}
Dipendra Misra, Mikael Henaff, Akshay Krishnamurthy, and John Langford.
\newblock Kinematic state abstraction and provably efficient rich-observation
  reinforcement learning.
\newblock In \emph{International conference on machine learning}, pages
  6961--6971. PMLR, 2020.

\bibitem[Modi et~al.(2021)Modi, Chen, Krishnamurthy, Jiang, and
  Agarwal]{modi2021model}
Aditya Modi, Jinglin Chen, Akshay Krishnamurthy, Nan Jiang, and Alekh Agarwal.
\newblock Model-free representation learning and exploration in low-rank mdps.
\newblock \emph{arXiv preprint arXiv:2102.07035}, 2021.

\bibitem[Ni et~al.(2021)Ni, Zhang, Duan, and Wang]{ni2021learning}
Chengzhuo Ni, Anru Zhang, Yaqi Duan, and Mengdi Wang.
\newblock Learning good state and action representations via tensor
  decomposition.
\newblock \emph{arXiv preprint arXiv:2105.01136}, 2021.

\bibitem[Papini et~al.(2021)Papini, Tirinzoni, Pacchiano, Restelli, Lazaric,
  and Pirotta]{papini2021reinforcement}
Matteo Papini, Andrea Tirinzoni, Aldo Pacchiano, Marcello Restelli, Alessandro
  Lazaric, and Matteo Pirotta.
\newblock Reinforcement learning in linear mdps: Constant regret and
  representation selection.
\newblock \emph{Advances in Neural Information Processing Systems}, 34, 2021.

\bibitem[Parr et~al.(2008)Parr, Li, Taylor, Painter-Wakefield, and
  Littman]{parr2008analysis}
Ronald Parr, Lihong Li, Gavin Taylor, Christopher Painter-Wakefield, and
  Michael~L Littman.
\newblock An analysis of linear models, linear value-function approximation,
  and feature selection for reinforcement learning.
\newblock In \emph{Proceedings of the 25th international conference on Machine
  learning}, pages 752--759, 2008.

\bibitem[Pathak et~al.(2017)Pathak, Agrawal, Efros, and
  Darrell]{pathakICMl17curiosity}
Deepak Pathak, Pulkit Agrawal, Alexei~A. Efros, and Trevor Darrell.
\newblock Curiosity-driven exploration by self-supervised prediction.
\newblock In \emph{ICML}, 2017.

\bibitem[Rahimi and Recht(2007)]{rahimi2007random}
Ali Rahimi and Benjamin Recht.
\newblock Random features for large-scale kernel machines.
\newblock In \emph{Proceedings of the 20th International Conference on Neural
  Information Processing Systems}, pages 1177--1184, 2007.

\bibitem[Ren et~al.(2021)Ren, Zhang, Szepesv{\'a}ri, and Dai]{ren2021free}
Tongzheng Ren, Tianjun Zhang, Csaba Szepesv{\'a}ri, and Bo~Dai.
\newblock A free lunch from the noise: Provable and practical exploration for
  representation learning.
\newblock \emph{arXiv preprint arXiv:2111.11485}, 2021.

\bibitem[Rendle et~al.(2010)Rendle, Freudenthaler, and
  Schmidt-Thieme]{rendle2010factorizing}
Steffen Rendle, Christoph Freudenthaler, and Lars Schmidt-Thieme.
\newblock Factorizing personalized markov chains for next-basket
  recommendation.
\newblock In \emph{Proceedings of the 19th international conference on World
  wide web}, pages 811--820, 2010.

\bibitem[Schulman et~al.(2017)Schulman, Wolski, Dhariwal, Radford, and
  Klimov]{schulman2017proximal}
John Schulman, Filip Wolski, Prafulla Dhariwal, Alec Radford, and Oleg Klimov.
\newblock Proximal policy optimization algorithms.
\newblock \emph{arXiv preprint arXiv:1707.06347}, 2017.

\bibitem[Schwarzer et~al.(2020)Schwarzer, Anand, Goel, Hjelm, Courville, and
  Bachman]{schwarzer2020data}
Max Schwarzer, Ankesh Anand, Rishab Goel, R~Devon Hjelm, Aaron Courville, and
  Philip Bachman.
\newblock Data-efficient reinforcement learning with self-predictive
  representations.
\newblock In \emph{International Conference on Learning Representations}, 2020.

\bibitem[Sekar et~al.(2020)Sekar, Rybkin, Daniilidis, Abbeel, Hafner, and
  Pathak]{sekar2020planning}
Ramanan Sekar, Oleh Rybkin, Kostas Daniilidis, Pieter Abbeel, Danijar Hafner,
  and Deepak Pathak.
\newblock Planning to explore via self-supervised world models.
\newblock In \emph{ICML}, 2020.

\bibitem[Sekhari et~al.(2021)Sekhari, Dann, Mohri, Mansour, and
  Sridharan]{sekhari2021agnostic}
Ayush Sekhari, Christoph Dann, Mehryar Mohri, Yishay Mansour, and Karthik
  Sridharan.
\newblock Agnostic reinforcement learning with low-rank mdps and rich
  observations.
\newblock \emph{Advances in Neural Information Processing Systems}, 34, 2021.

\bibitem[Sodhani et~al.(2021)Sodhani, Meier, Pineau, and
  Zhang]{sodhani2021block}
Shagun Sodhani, Franziska Meier, Joelle Pineau, and Amy Zhang.
\newblock Block contextual mdps for continual learning.
\newblock \emph{arXiv preprint arXiv:2110.06972}, 2021.

\bibitem[Sun et~al.(2019)Sun, Jiang, Krishnamurthy, Agarwal, and
  Langford]{sun2019model}
Wen Sun, Nan Jiang, Akshay Krishnamurthy, Alekh Agarwal, and John Langford.
\newblock Model-based rl in contextual decision processes: Pac bounds and
  exponential improvements over model-free approaches.
\newblock In \emph{Conference on learning theory}, pages 2898--2933. PMLR,
  2019.

\bibitem[Uehara et~al.(2021)Uehara, Zhang, and Sun]{uehara2021representation}
Masatoshi Uehara, Xuezhou Zhang, and Wen Sun.
\newblock Representation learning for online and offline rl in low-rank mdps.
\newblock \emph{arXiv preprint arXiv:2110.04652}, 2021.

\bibitem[Yang et~al.(2022)Yang, Hu, Lee, and Du]{yang2020provable}
Jiaqi Yang, Wei Hu, Jason~D Lee, and Simon~S Du.
\newblock Provable benefits of representation learning in linear bandits.
\newblock \emph{ICLR}, 2022.

\bibitem[Yang and Wang(2019)]{yang2019sample}
Lin Yang and Mengdi Wang.
\newblock Sample-optimal parametric q-learning using linearly additive
  features.
\newblock In \emph{International Conference on Machine Learning}, pages
  6995--7004. PMLR, 2019.

\bibitem[Yao et~al.(2014)Yao, Szepesv{\'a}ri, Pires, and Zhang]{yao2014pseudo}
Hengshuai Yao, Csaba Szepesv{\'a}ri, Bernardo~Avila Pires, and Xinhua Zhang.
\newblock Pseudo-mdps and factored linear action models.
\newblock In \emph{2014 IEEE Symposium on Adaptive Dynamic Programming and
  Reinforcement Learning (ADPRL)}, 2014.

\bibitem[Zanette et~al.(2021)Zanette, Cheng, and
  Agarwal]{zanette2021cautiously}
Andrea Zanette, Ching-An Cheng, and Alekh Agarwal.
\newblock Cautiously optimistic policy optimization and exploration with linear
  function approximation.
\newblock \emph{COLT}, 2021.

\bibitem[Zhang et~al.(2020)Zhang, Sodhani, Khetarpal, and
  Pineau]{zhang2020learning}
Amy Zhang, Shagun Sodhani, Khimya Khetarpal, and Joelle Pineau.
\newblock Learning robust state abstractions for hidden-parameter block mdps.
\newblock In \emph{International Conference on Learning Representations}, 2020.

\bibitem[Zhang et~al.(2021)Zhang, He, Zhou, Zhang, and Gu]{zhang2021provably}
Weitong Zhang, Jiafan He, Dongruo Zhou, Amy Zhang, and Quanquan Gu.
\newblock Provably efficient representation learning in low-rank markov
  decision processes.
\newblock \emph{arXiv preprint arXiv:2106.11935}, 2021.

\end{thebibliography}


\newpage
\appendix
\onecolumn

\section{Sample Complexity Analysis}\label{sec:app_proof_main}

Recall that we define the non-parametric model $\hat  P^{n}_h$ as follows:
\begin{align*}
\hat P^n_{h}(s' \vert s,a)=& \hat \phi^n_h(s,a)^{\top} \Sigma_h^{-1} \sum_{\tilde{s},\tilde a,\tilde{s}'\in \Dcal_h \cup \Dcal'_h }  \hat \phi^n_h(\tilde s,\tilde a)\one(s'=\tilde{s'}), \quad \Sigma_h := \Big( \sum_{s,a\in \Dcal'_h + \Dcal'_h}  \hat \phi^n_h(s,a)\hat \phi^n_h(s,a)^{\top} + \lambda_n I \Big)^{-1}.
\end{align*}
and we define 
$$\hat \mu^n_h = \Sigma_h^{-1} \sum_{\tilde{s},\tilde a,\tilde{s}'\in \Dcal_h \cup \Dcal'_h }  \hat\phi^n_h(\tilde s,\tilde a)\one(s'=\tilde{s'}) .$$
Throughout the appendix, we also abuse the notation $\EE_{(s,a)\sim p(s,a)}[f(s,a)] = \int  f(s,a)p(s,a)d(s,a)$ for any non-negative function $p(s,a)\geq 0$.

Recall that in the case of Block MDP where $\hat\phi^n_h$ are one-hot vectors, and thus $\hat P^n_{h}(\cdot \vert s,a)$ can be further simplified as
\begin{align*}
    \hat P^n_{h}(s' \vert s,a) = \frac{N_{\Dcal_h + \Dcal'_h}(\hat\phi^n_h(s,a),s')}{N_{\Dcal_h + \Dcal'_h}(\hat\phi^n_h(s,a))+\lambda_n}
\end{align*}
where $N_{\Dcal}(\hat\phi^n_h(s,a),s')$ denotes the number of triples $(\tilde s, \tilde a, \tilde s')\in \Dcal$ such that $\hat\phi^n_h(\tilde s, \tilde a) = \hat\phi^n_h(s,a)$ and $\tilde s'=s'$, and $N_{\Dcal}(\hat\phi^n_h(s,a)) = \sum_{s'\in\Scal} N_{\Dcal}(\hat\phi^n_h(s,a),s')$. Therefore, we can clearly see that 
$\hat P^n_{h}(s' \vert s,a)>0$ and $\sum_{s'\in\Scal} \hat P^n_{h}(s' \vert s,a)<1$ for $\lambda_n>0$.

For $\hat\pi^{n} :=  \textsc{LSVI}\left(\{ r_h + \hat b^n_h\}_{h=0}^{H-1}, \{\hat\phi^n_h\}_{h=0}^{H-1}, \{\Dcal_h + \Dcal'_h\}_{h=0}^{H-1},\lambda_n\right) $, $\hat\pi^n$ is indeed the optimal policy for the MDP model with transitions $\{ \hat P^n_h\}$ and rewards $\{ r_h + \hat b^n_h \}$ (i.e., the output of Value Iteration in $\{\hat P^n_h, r_h+\hat b^n_h\}_{h=0}^{H-1}$). In this section, we will take the model-based perspective and analyze based on this fitted model $\hat P^n$.

We define a few mixture distributions that will be used extensively in the analysis. For any $n,h$, define $\rho^n_h \in \Delta(\Scal\times\Acal)$ as follows:
\begin{align*}
\rho^n_h(s,a) = \frac{1}{n}\sum_{i=0}^{n-1} d^{\hat\pi_i}_h(s) U(a), \forall s,a\in\Scal\times\Acal.
\end{align*} For any $n, h \geq 1$ define $\beta^n_h$ as follows:
\begin{align*}
\beta^n_h(s,a) = \frac{1}{n}\sum_{i=0}^{n-1}\EE_{\tilde s \sim d^{\hat\pi_i}_{h-1}, \tilde a\sim U(\Acal)}  P^\star(s | \tilde s, \tilde a) U(a).
\end{align*}
For any $n, h$, we also define $\gamma^n_h \in \Delta(\Scal\times\Acal)$ as follows:
\begin{align*}
\gamma^n_h(s,a) = \frac{1}{n} \sum_{i=0}^{n-1} d^{\hat\pi_i}_h(s,a).
\end{align*}
For an iteration $n$, a distribution $\rho$ and a feature $\phi$, we denote the expected feature covariance as
\begin{align*}
    \Sigma^n_{\rho,\phi} = n\EE_{(s,a)\sim\rho}[\phi(s,a)\phi(s,a)^\top] + \lambda_n I
\end{align*}
which in the case of Block MDP is a diagonal matrix.
Notice that the dataset $\Dcal_h$ of size $n$ is sampled from $\rho^n_h$ and the dataset $\Dcal'_h$ of size $n$ is sampled from $\beta^n_h$.
Below we focus on a particular iteration $n$, and drop the $n$ superscript. 

For the remainder of this section, we assume that we have learned a representation $\hat\phi_h$ such that the following generalization bound holds:
\begin{align}\label{eq:rep-learn}
\max_{f\in\Fcal_h}\EE_{(s,a)\sim \rho_h}[(\hat P_h(\cdot \mid s,a)^\top f-P^{{\star}}_h(\cdot \mid s,a)^\top f)^2]&\leq \zeta_n, \forall h\in [H].\\
\max_{f\in\Fcal_h}\EE_{(s,a)\sim \beta_h}[(\hat P_h(\cdot \mid s,a)^\top f-P^{{\star}}_h(\cdot \mid s,a)^\top f)^2]&\leq \zeta_n, \forall h\in [H]\backslash\{0\}\nonumber.
\end{align}
where again the discriminator class is set to $\Fcal = \Fcal^{(1)}_h \bigcup \Fcal^{(2)}_h$.
\begin{align} \label{eq:Fcal}
    \Fcal^{(1)}_h &= \left\{f(s): =  \EE_{a\sim U(\Acal)}\left[\phi(s,a)^\top \theta-\phi'(s,a)^\top \theta'\right]  \;\Big\vert \; \phi,\phi'\in\Phi_{h+1}\; \max(\|\theta\|_\infty,\|\theta'\|_\infty)\leq 1\right\},\\
    \Fcal_h^{(2)} &= \left\{ f(s): = \max_a \left( \frac{r_{h+1}(s,a) + \min\left\{ w^\top \phi(s,a)  , 2\right\}}{2H+1} + {w'}^{\top} \phi(s,a)     \right) \Big \vert  \phi\in \Phi_{h+1};  \;\|w\|_\infty\leq c,\|w'\|_\infty\leq 1 \right\}\nonumber
\end{align}
In the analysis below, we actually need $\Fcal_h$ to capture the follow two forms of function:
\begin{align}
    f^{(1)}(s) = &\EE_{a\sim U(\Acal)}\left[\EE_{\hat P_h(s'_{h}\mid s_h,a_h)}[V^{\pi^{\star}}_{P^{\star},r,h+1}(s_h')]-\EE_{P^{\star}_h(s'_{h}\mid s_h,a_h)}[V^{\pi^{\star}}_{P^{\star},r,h+1}(s_h')]\right]\label{eq:F1}\\
    f^{(2)}(s) = &\frac{1}{2H+1}V^{\hat\pi^n}_{\hat P,r+\hat b,h+1}(s) = \frac{1}{2H+1}\left(\max_a r_{h+1}(s,a)+\hat b_{h+1}(s,a)+\hat\phi^{\top}_{h+1}(s,a)\int \hat \mu_{h+1}(s')V^{\hat\pi^n}_{\hat P,r+\hat b,h+2}(s') \text{d}s'\right)\label{eq:F2}
\end{align}
$f^{(1)}(s)$ is naturally captured by $\Fcal^{(1)}(s)$ since $V^{\pi^{\star}}_{P^{\star},r,h+1}(s_h')$ is bounded in $[0,1]$, and by \pref{lem:bound_w}, the expectation of any bounded function under $\hat P_h$ (resp. $P^*_h$) is a linear function of $\hat\phi_h$ (resp. $\phi^*_h$). For $f^{(2)}(s)$, the last term is captured by ${w'}^{\top} \phi(s,a)$ in $\Fcal^{(2)}$, because $V^{\pi^n}_{\hat P,r+b,h+2}(s)$ is bounded in $[0,2H]$. For the bonus term, recall that due to the Block MDP structure, the bonus takes the form of 
\begin{equation*}
 \hat b^n_h(s,a) = \alpha_n\sqrt{\hat\phi^n_h(s,a)^\top\Sigma_h^{-1}\hat\phi^n_h(s,a)} = \hat\phi^n_h(s,a)\cdot \sqrt{\frac{\alpha^2_n}{N(\hat\phi^n_h(s,a))+\lambda_n}}
\end{equation*}
which is in fact linear in $\hat\phi^n_h(s,a)$ and can be captured by the 2nd term in $\Fcal^{(2)}_h$, with $c=\frac{\alpha_N}{\sqrt{\lambda_N}}$. Note that $\frac{\alpha_N}{\sqrt{\lambda_N}}\geq \frac{\alpha_n}{\sqrt{\lambda_n}}$ for all $n\in[N]$.

To begin with, we establish two forms of one-step-back tricks that are central to our analysis. They are of close resemblance to the one-step-back lemmas in \repucb~\citep{uehara2021representation}.

\begin{lemma}[One-step-back inequality in the learned model]\label{lem:tsbt1}
Consider a set of functions $\{g_h\}_{h=0}^{H}$ that satisfies $g_h\in \Scal\times \Acal \to \RR$, s.t. $\|g_h\|_{\infty}\leq B$ and $\EE_{a\sim U(\Acal)}g_{h+1}(s,a)\in\Fcal_h$ for all $h\in[H]$. We condition on the event where the \textsc{RepLearn} guarantee \eqref{eq:rep-learn} holds,
where $\Fcal_h$ are defined as in Eq.~\pref{eq:Fcal}.
Then, we have for any policy $\pi$,
\begin{align*}
\sum_{h=0}^{H-1}\EE_{(s,a)\sim d^{\pi}_{\hat P,h }}&[g_h(s,a)]  \leq 
\sum_{h=0}^{H-2}\EE_{(\tilde s,\tilde a)\sim d^{\pi}_{\hat P ,h}} \min\Bigl\{\|\hat \phi_h(\tilde s,\tilde a)\|_{\Sigma_{\rho_h,\hat \phi_h}^{-1}} \cdot
\\
&\sqrt{n|\Acal|^2\EE_{(s,a)\sim \beta_{h+1}}\bracks{g_{h+1}^2(s,a)}+ B^2\lambda_n d+ n |\Acal|^2\zeta_n},B\Bigl\}
+\sqrt{|\Acal|\EE_{(s,a)\sim \rho_0}[g_0^2(s,a)]}. 
\end{align*}
\end{lemma}

\begin{proof}
For step $h=0$, we have
\begin{align*}
\EE_{(s,a)\sim d^{\pi}_{\hat P,0 }} [g_0(s,a)] =& \EE_{s\sim d_0,a\sim\pi_0(s)} [g_0(s,a)]\\
\leq &\sqrt{\max_{(s,a)}\frac{d_0(s)\pi_0(a\mid s)}{\rho_0(s,a)}\EE_{(s,a)\sim \rho_0}\bracks{g_0^2(s,a)}}\tag{Jensen}\\
\leq& \sqrt{ \max_{(s,a)}\frac{d_0(s)\pi_0(a\mid s)}{d_0(s)u(a)}\EE_{(s,a)\sim \rho_0}\bracks{g_0^2(s,a) }}\tag{behavior policy has uniform action}\\
\leq &
\sqrt{|\Acal|\EE_{(s,a)\sim \rho_0}\bracks{g_0^2(s,a) }}. 
\end{align*}

For step $h = 1,...,H-1$, we observe the following one-step-back decomposition:
\begin{align*}
    &\EE_{(s,a)\sim d^{\pi}_{\hat P,h }} [g_h(s,a)]
    \\
    =& \EE_{(\tilde s,\tilde a)\sim d^{\pi}_{\hat P,{h-1}},s\sim \hat P_{h-1}(\tilde s,\tilde a),a\sim \pi_{h-1}(s)}[g_h(s,a)]
    \\
    =&\EE_{(\tilde s,\tilde a)\sim d^{\pi}_{\hat P,{h-1} }}\hat \phi_{h-1}(\tilde s,\tilde a)^{\top}\int \sum_{a}\hat \mu_{h-1}(s)\pi_{h-1}(a\mid s)g_h(s,a) \text{d}s
    \\ 
   \leq& \EE_{(\tilde s,\tilde a)\sim d^{\pi}_{\hat P,{h-1}}} \left[\min\left\{\|\hat \phi_{h-1}(\tilde s,\tilde a)\|_{\Sigma_{\rho_{h-1},\hat \phi_{h-1}}^{-1}}\left \|\int \sum_{a}\hat \mu_{h-1}(s)\pi_{h-1}(a\mid s)g_h(s,a) d(s)\right\|_{\Sigma_{\rho_{h-1},\hat \phi_{h-1}}},B\right\}\right].
\end{align*}
where the last step follows because $\hat \phi_{h-1}(\tilde s,\tilde a)^{\top}\int \sum_{a}\hat \mu_{h-1}(s)\pi_{h-1}(a\mid s)g_h(s,a) \text{d}s\leq B$ for all $(\tilde s,\tilde a)$ pairs.

Next, for any $h$,
\begin{align*}
  & \left \|\int \sum_{a}\hat \mu_{h}(s)\pi_{h}(a\mid s)g_{h+1}(s,a) d(s)\right\|^2_{\Sigma_{\rho_{h},\hat \phi_{h}}}\\
=&  \braces{\int \sum_{a}\hat \mu_{h}(s)\pi_{h}(a\mid s)g_{h+1}(s,a) d(s)}^{\top}\braces{n \EE_{(\tilde s,\tilde a)\sim \rho_{h}}[\hat \phi_{h} \hat \phi^{\top}_{h}]+\lambda_n I  }\braces{\int \sum_{a}\hat \mu_{h}(s)\pi_{h}(a\mid s)g_{h+1}(s,a) d(s)}\\
\leq&  n \EE_{(\tilde s,\tilde a)\sim \rho_h}\braces{\bracks{\int \sum_a \hat \mu_h(s)^{\top}\hat \phi_h(\tilde s,\tilde a)\pi_h(a\mid s)g_{h+1}(s,a) d(s)}^2}+ B^2\lambda_n d \tag{Use the assumption $\|\sum_a \pi_h(a\mid s)g_{h+1}(s,a)\|_{\infty}\leq B$ and $\int \|\hat \mu_h(s)h(s)d(s)\|_2 \leq \sqrt{d}$ for any $h:\Scal \to [0,1]$. }\\
=&  n \EE_{(\tilde s,\tilde a)\sim \rho_h}\bracks{\braces{\EE_{s\sim \hat P_h(\tilde s,\tilde a), a\sim \pi_h(s)}\bracks{g_{h+1}(s,a) }}^2}+ B^2\lambda_n d \\
=&  n |\Acal|^2\EE_{(\tilde s,\tilde a)\sim \rho_h}\bracks{\braces{|\Acal|\EE_{s\sim \hat P_h(\tilde s,\tilde a), a\sim U(\Acal)}\bracks{g_{h+1}(s,a) }}^2}+ B^2\lambda_n d \tag{Importance Sampling}\\
\leq& n|\Acal|^2 \EE_{(\tilde s,\tilde a)\sim \rho_h}\bracks{\braces{\EE_{s\sim P_h^{{\star}}(\tilde s,\tilde a), a\sim U(\Acal)}\bracks{g_{h+1}(s,a) }}^2}+ B^2\lambda_n d + n |\Acal|^2\zeta_n \tag{\textsc{RepLearn}: $\EE_{a\sim U(\Acal)}g_{h+1}(s,a)\in\Fcal_h$}\\
\leq&  n|\Acal|^2\EE_{(\tilde s,\tilde a)\sim \rho_h, s\sim P_h^{{\star}}(\tilde s,\tilde a), a\sim U(\Acal)}\bracks{g_{h+1}^2(s,a) }+ B^2\lambda_n d+ n |\Acal|^2\zeta_n.  \tag{Jensen} \\
=& n|\Acal|^2\EE_{(s,a)\sim \beta_{h+1}}\bracks{g_{h+1}^2(s,a)}+ B^2\lambda_n d+ n |\Acal|^2\zeta_n
\end{align*}
Summing the decomposition for all steps $h$ gives the desired result.
\end{proof}

\begin{lemma}[One-step back inequality for the true model]\label{lem:tsbt2} 
Consider a set of functions $\{g_h\}_{h=0}^{H}$ that satisfies $g_h\in \Scal\times \Acal \to \RR$, s.t. $\|g_h\|_{\infty}\leq B$ for all $h\in[H]$. Then, for any policy $\pi$,
\begin{align*}
\sum_{h=0}^{H-1}\EE_{(s,a)\sim d^{\pi}_{P^{\star},h }}&[g_h(s,a)]  \leq 
\sum_{h=0}^{H-2}\EE_{(\tilde s,\tilde a)\sim d^{\pi}_{P^{\star} ,h}} \|\phi^{\star}_h(\tilde s,\tilde a)\|_{\Sigma_{\gamma_h,\phi^{\star}_h}^{-1}} \cdot\\
&\sqrt{n|\Acal|\EE_{(s,a)\sim \rho_{h+1}}\bracks{g_{h+1}^2(s,a)}+ B^2\lambda_n d}
+\sqrt{|\Acal|\EE_{(s,a)\sim \rho_0}[g_0^2(s,a)]}.
\end{align*}

\end{lemma}
\begin{proof}
For step $h=0$, we similarly have
\begin{align*}
\EE_{(s,a)\sim d^{\pi}_{P^{\star},0 }} [g_0(s,a)] =& \EE_{s\sim d_0,a\sim\pi_0(s)} [g_0(s,a)]\\
\leq &\sqrt{\max_{(s,a)}\frac{d_0(s)\pi_0(a\mid s)}{\rho_0(s,a)}\EE_{(s,a)\sim \rho_0}\bracks{g_0^2(s,a)}}\tag{Jensen}\\
\leq& \sqrt{ \max_{(s,a)}\frac{d_0(s)\pi_0(a\mid s)}{d_0(s)u(a)}\EE_{(s,a)\sim \rho_0}\bracks{g_0^2(s,a) }}\tag{behavior policy has uniform action}\\
\leq &
\sqrt{|\Acal|\EE_{(s,a)\sim \rho_0}\bracks{g_0^2(s,a) }}. 
\end{align*}

For step $h = 1,...,H-1$, we observe the following one-step-back decomposition:
\begin{align*}
    \EE_{(s,a)\sim d^{\pi}_{P^{\star},h }} [g_h(s,a)] =& \EE_{(\tilde s,\tilde a)\sim d^{\pi}_{P^{\star},{h-1}},s\sim P^{\star}_{h-1}(\tilde s,\tilde a),a\sim \pi_{h-1}(s)}[g_h(s,a)]\\
    =&\EE_{(\tilde s,\tilde a)\sim d^{\pi}_{P^{\star},{h-1} }}\phi^{\star}_{h-1}(\tilde s,\tilde a)^{\top}\int \sum_{a} \mu^{\star}_{h-1}(s)\pi_{h-1}(a\mid s)g_h(s,a) \text{d}s\\ 
   \leq& \EE_{(\tilde s,\tilde a)\sim d^{\pi}_{P^{\star},{h-1}}} \|\phi^{\star}_{h-1}(\tilde s,\tilde a)\|_{\Sigma_{\gamma_{h-1}, \phi^{\star}_{h-1}}^{-1}}\left \|\int \sum_{a} \mu^{\star}_{h-1}(s)\pi_{h-1}(a\mid s)g_h(s,a) d(s)\right\|_{\Sigma_{\gamma_{h-1}, \phi^{\star}_{h-1}}}.
\end{align*}
For any $h$,
\begin{align*}
  & \left \|\int \sum_{a} \mu^{\star}_{h}(s)\pi_{h}(a\mid s)g_{h+1}(s,a) d(s)\right\|_{\Sigma_{\gamma_{h},\phi^{\star}_{h}}}\\
\leq&  \braces{\int \sum_{a}\mu^{\star}_{h}(s)\pi_{h}(a\mid s)g_{h+1}(s,a) d(s)}^{\top}\braces{n \EE_{(\tilde s,\tilde a)\sim \gamma_{h}}[\phi^{\star}_{h} \phi^{{\star}\top}_{h}]+\lambda_n I  }\braces{\int \sum_{a}\mu^{\star}_{h}(s)\pi_{h}(a\mid s)g_{h+1}(s,a) d(s)}\\
\leq&  n \EE_{(\tilde s,\tilde a)\sim \gamma_h}\braces{\bracks{\int \sum_a \mu^{\star}_h(s)^{\top}\phi^{\star}_h(\tilde s,\tilde a)\pi_h(a\mid s)g_{h+1}(s,a) d(s)}^2}+ B^2\lambda_n d \tag{Use the assumption $\|\sum_a \pi_h(a\mid s)g(s_i,a_i)\|_{\infty}\leq B$ and $\int \| \mu^{\star}_h(s)h(s)d(s)\|_2\leq \sqrt{d}$ for any $h:\Scal \to [0,1]$. }\\
=&  n \EE_{(\tilde s,\tilde a)\sim \gamma_h}\bracks{\braces{\EE_{s\sim P^{\star}_h(\tilde s,\tilde a), a\sim \pi_h(s)}\bracks{g_{h+1}(s,a) }}^2}+ B^2\lambda_n d \\
\leq&  n\EE_{(\tilde s,\tilde a)\sim \gamma_h, s\sim P_h^{{\star}}(\tilde s,\tilde a), a\sim \pi_h(s)}\bracks{g_{h+1}^2(s,a) }+ B^2\lambda_n d.\tag{Jensen} \\
\leq&  n|\Acal|\EE_{(\tilde s,\tilde a)\sim \gamma_h, s\sim P_h^{{\star}}(\tilde s,\tilde a), a\sim U(\Acal)}\bracks{g_{h+1}^2(s,a) }+ B^2\lambda_n d
\tag{Importance Sampling}\\
\leq& n|\Acal|\EE_{(s,a)\sim \rho_{h+1}}\bracks{g_{h+1}^2(s,a)}+ B^2\lambda_n d
\end{align*}

Then, the final statement is immediately concluded.
\end{proof}

Notice that compared to \pref{lem:tsbt1}, \pref{lem:tsbt2} post no structural assumption on $g_h$ other than boundedness, and does not rely on the \textsc{RepLearn} guarantee. Next, we prove the almost optimism Lemma presented in \pref{lem:opt_main}, restated below.

\begin{lemma}[Almost Optimism at the Initial Distribution]\label{lem:optimism}
Consider an episode $n (1\leq n\leq N)$ and set  
\begin{align*}
    \alpha_n = \sqrt{n|\Acal|^2\zeta_n+ 4\lambda_n d+ n \zeta_n}/c ,\quad \lambda_n=O\prns{d\ln (|\Phi|n/\delta)}. 
\end{align*}
where $c$ is an absolute constant.
Conditioning on the event that the \textsc{RepLearn} guarantee \eqref{eq:rep-learn} holds, then with probability $1-\delta$, we have for all $n\in [1,\cdots,N]$,
\begin{align*}
  V^{\pi^\star}_{\hat P_n,r+\hat b_n}- V^{\pi^\star}_{P^{{\star}},r}\geq -\sqrt{|\Acal|\zeta_n}. 
\end{align*}
\end{lemma}
\begin{proof}
Then, from simulation lemma (\pref{lem:pdl}), we have 
\begin{align}
    & V^{\pi^{\star}}_{\hat P,r+\hat b}- V^{\pi^{\star}}_{P^{{\star}},r} \nonumber \\
    &=\sum_{h=0}^{H-1}\EE_{(s_h,a_h)\sim d^{\pi^{\star}}_{\hat P,h}}\left[\hat b_h(s_h,a_h)+ \EE_{\hat P_h(s'_{h}\mid s_h,a_h)}[V^{\pi^{\star}}_{P^{\star},r,h+1}(s_h')]-\EE_{P^{\star}_h(s'_{h}\mid s_h,a_h)}[V^{\pi^{\star}}_{P^{\star},r,h+1}(s_h')] \right]\nonumber\\
    &\geq\sum_{h=0}^{H-1}\EE_{(s_h,a_h)\sim d^{\pi^{\star}}_{\hat P,h}}\left[\min \prns{c\alpha_n \|\hat \phi_h(s,a)\|_{ \Sigma^{-1}_{\rho_h,\hat \phi_h}},2 }+ \EE_{\hat P_h(s'_{h}\mid s_h,a_h)}[V^{\pi^{\star}}_{P^{\star},r,h+1}(s_h')]-\EE_{P^{\star}_h(s'_{h}\mid s_h,a_h)}[V^{\pi^{\star}}_{P^{\star},r,h+1}(s_h')] \right] \label{eq:simulation_lemma_bonus_model_error}
\end{align} 
where in the last step, we apply \pref{lem:con} to replace the empirical covariance by the population covariance and $c$ is an absolute constant. Denote
\begin{align*}
   f_h(s,a) = \EE_{\hat P_h(s'_{h}\mid s,a)}[V^{\pi^{\star}}_{P^{\star},r,h+1}(s_h')]-\EE_{P^{\star}_h(s'_{h}\mid s,a)}[V^{\pi^{\star}}_{P^{\star},r,h+1}(s_h')]
\end{align*}
Notice that we have $\|f_h(s,a)\|_\infty\leq 2$,and since $V^{\pi^{\star}}_{P^{{\star}},r,h+1}(s') = \max_a r_{h+1}(s',a)+\PP^{\star}_{h+1}V^{\star}_{P^{\star},r,h+2}$ and $\PP^{\star}_{h+1}V^{\star}_{P^{\star},r,h+2}$ is linear in $\phi^{\star}_{h+1}\in\Phi_{h+1}$, we know $V^{\pi^{\star}}_{P^{{\star}},r,h+1}\in \Fcal_h$.
Then, by the \textsc{RepLearn} guarantee, we have
\begin{equation*}
    \EE_{(s,a)\sim \rho_h}\bracks{f_h^2(s,a)}\leq \zeta_n, \EE_{(s,a)\sim \beta_{h}}\bracks{f_h^2(s,a)}\leq \zeta_n
\end{equation*}
Also, since $\EE_{\hat P_h(s'_{h}\mid s_h,a_h)}[V^{\pi^{\star}}_{P^{\star},r,h+1}(s_h')]$ is linear in $\hat\phi_h$ and $\EE_{P^{\star}_h(s'_{h}\mid s_h,a_h)}[V^{\pi^{\star}}_{P^{\star},r,h+1}(s_h')]$ is linear in $\phi^{\star}_h$, we have $\EE_{a\sim U(\Acal)}f_h(s,a)\in\Fcal_h$ as well.

Then, substituting $g_h(s,a) = f_h(s,a)$ in \pref{lem:tsbt1}, we have: 
\begin{align*}
&\sum_{h=0}^{H-1}\EE_{(s,a)\sim d^{\pi^{\star}}_{\hat P,h }}[g_h(s,a)]\\  
\leq& 
\sum_{h=0}^{H-2}\EE_{(\tilde s,\tilde a)\sim d^{\pi^{\star}}_{\hat P ,h}} \min\Bigl\{\|\hat \phi_h(\tilde s,\tilde a)\|_{\Sigma_{\rho_h,\hat \phi_h}^{-1}} \cdot
\sqrt{n|\Acal|^2\EE_{(s,a)\sim \beta_{h}}\bracks{g_{h+1}^2(s,a)}+ 4\lambda_n d+ n \zeta_n},2 \Bigl\}
+\sqrt{|\Acal|\EE_{(s,a)\sim \rho_0}[g_0^2(s,a)]}\\
\leq& 
\sum_{h=0}^{H-2}\EE_{(\tilde s,\tilde a)\sim d^{\pi^{\star}}_{\hat P ,h}} \min\Bigl\{\|\hat \phi_h(\tilde s,\tilde a)\|_{\Sigma_{\rho_h,\hat \phi_h}^{-1}} \cdot
\sqrt{n|\Acal|^2\zeta_n+ 4\lambda_n d+ n \zeta_n} ,2 \Bigl\}
+\sqrt{|\Acal|\zeta_n}\\
\leq& 
\sum_{h=0}^{H-2}\EE_{(\tilde s,\tilde a)\sim d^{\pi^{\star}}_{\hat P ,h}} \min\Bigl\{c\alpha_n\|\hat \phi_h(\tilde s,\tilde a)\|_{\Sigma_{\rho_h,\hat \phi_h}^{-1}},2 \Bigl\}
+\sqrt{|\Acal|\zeta_n}
\end{align*}
where in the last step we denote
\begin{align*}
    \alpha_n= \sqrt{n|\Acal|^2\zeta_n+ 4\lambda_n d+ n \zeta_n}/c
\end{align*}

Going back to \pref{eq:simulation_lemma_bonus_model_error}, we have 
\begin{align*}
    & V^{\pi^{\star}}_{\hat P,r+b}- V^{\pi^{\star}}_{P^{{\star}},r} \nonumber \\
    =&\sum_{h=0}^{H-1}\EE_{(s_h,a_h)\sim d^{\pi^{\star}}_{\hat P,h}}\left[\min \prns{c\alpha_n \|\hat \phi_h(s,a)\|_{ \Sigma^{-1}_{\rho_h,\hat \phi_h}},2 }+ \EE_{\hat P_h(s'_{h}\mid s_h,a_h)}[V^{\pi^{\star}}_{P^{\star},r,h+1}(s_h')]-\EE_{P^{\star}_h(s'_{h}\mid s_h,a_h)}[V^{\pi^{\star}}_{P^{\star},r,h+1}(s_h')] \right]\\
    =&\sum_{h=0}^{H-1}\EE_{(s_h,a_h)\sim d^{\pi^{\star}}_{\hat P,h}}\left[\min \prns{c\alpha_n \|\hat \phi_h(s,a)\|_{ \Sigma^{-1}_{\rho_h,\hat \phi_h}},2}\right]-
    \sum_{h=0}^{H-2}\EE_{(s_h,a_h)\sim d^{\pi^{\star}}_{\hat P,h}}\left[\min \prns{c\alpha_n \|\hat \phi_h(s,a)\|_{ \Sigma^{-1}_{\rho_h,\hat \phi_h}}+\sqrt{|\Acal|\zeta_n},2}\right]\\
    \geq& -\sqrt{|\Acal|\zeta_n}
\end{align*} 
This concludes the proof.
\end{proof}

With the above preparations, we are now ready to prove our main theorem.
\begin{theorem}
[Pseudo-Regret of \algname]\label{thm:pseudo_regret}
With probability $1-\delta$, we have 
\begin{align*}
    \sum_{n=0}^{N-1}  V^{\pi^{{\star}}}_{P^{{\star}},r}-V^{\hat \pi^n}_{P^{{\star}},r} \leq O\left(H^{5/2}|\Acal|^{1/2}d\log(|\Phi|/\delta)^{1/4}N^{3/4}\right)
\end{align*}
\end{theorem}

\begin{proof}

Similar to \pref{lem:optimism}, we condition on the event that the \textsc{RepLearn} guarantee \eqref{eq:rep-learn} holds, which by \pref{thm:rep_learn} happens with probability $1-\delta$.

For any fixed episode $n$ we have   
\begin{align*}
& V^{\pi^{{\star}}}_{P^{{\star}},r}-V^{\hat \pi^n}_{P^{{\star}},r}\\
\leq& V^{\pi^{{\star}}}_{\hat P,r+\hat b}-V^{\hat \pi^n}_{P^{{\star}},r}+\sqrt{|\Acal| \zeta_n} \tag{\pref{lem:optimism}}\\ 
\leq& V^{\hat \pi^n}_{\hat P,r+ \hat b}-V^{\hat \pi^n}_{P^{{\star}},r}+\sqrt{|\Acal| \zeta_n} \tag{$\hat \pi^n = \argmax_\pi V^\pi_{\hat P^n,r+\hat b^n}$}\\
=&\sum_{h=0}^{H-1}\EE_{(s_h,a_h)\sim d^{\hat \pi^n}_{P^{\star},h}}\left[\hat b_h(s_h,a_h)+ 
\EE_{\hat P_h(s'_{h}\mid s_h,a_h)}[V^{\hat \pi^n}_{\hat P,r+\hat b,h+1}(s_h')]-
\EE_{P^{\star}_h(s'_{h}\mid s_h,a_h)}[V^{\hat \pi^n}_{\hat P,r+\hat b,h+1}(s_h')] \right]
+\sqrt{|\Acal| \zeta_n}
\end{align*}
We used the 2nd form of Simulation Lemma (\pref{lem:pdl}) in the last display. 
Denote
\begin{align*}
   f_h(s,a) = \frac{1}{2H+1}\left\{\EE_{\hat P_h(s'_{h}\mid s_h,a_h)}[V^{\hat \pi^n}_{\hat P,r+\hat b,h+1}(s_h')]-
\EE_{P^{\star}_h(s'_{h}\mid s_h,a_h)}[V^{\hat \pi^n}_{\hat P,r+\hat b,h+1}(s_h')]\right\}
\end{align*}
Then, noting $\|\hat b\|_{\infty}\leq 2$, 
we have $\|V^{\hat \pi^n}_{\hat P,r+b,h+1}\|_{\infty}\leq (2H+1)$, and $\frac{1}{2H+1}V^{\hat \pi^n}_{\hat P,r+b,h+1}\in\Fcal_h$. Combining this fact with the above expansion, we have
\begin{align}
V^{\pi^{\star}}_{P^{{\star}},r} -V^{\hat \pi^n}_{P^{{\star}},r}
=\underbrace{\sum_{h=0}^{H-1}\EE_{(s_h,a_h)\sim d^{\hat \pi^n}_{P^{\star},h}}\left[\hat b_h(s_h,a_h)\right]}_{\text{(a)}}+ 
\underbrace{(2H+1)\sum_{h=0}^{H-1}\EE_{(s_h,a_h)\sim d^{\hat \pi^n}_{P^{\star},h}}\left[f_h(s_h,a_h)\right]}_{\text{(b)}}
+\sqrt{|\Acal| \zeta_n} \label{eq:regret_middle}
\end{align}

First, we calculate the first term (a) in Eq.~\pref{eq:regret_middle}. Following \pref{lem:tsbt2} and noting again $\|\hat b_h\|_\infty\leq 2$, we have

\begin{align*}
&\sum_{h=0}^{H-1}\EE_{(s_h,a_h)\sim d^{\hat \pi^n}_{P^{\star},h}}\left[\hat b_h(s_h,a_h)\right]\\
\leq &
\sum_{h=0}^{H-2}\EE_{(\tilde s,\tilde a)\sim d^{\hat \pi^n}_{P^{\star} ,h}} \|\phi^{\star}_h(\tilde s,\tilde a)\|_{\Sigma_{\gamma_h,\phi^{\star}_h}^{-1}}
\sqrt{n|\Acal|\EE_{(s,a)\sim \rho_{h+1}}\bracks{(\hat b_{h+1}(s,a))^2}
+ 4\lambda_n d} 
+\sqrt{|\Acal|\EE_{(s,a)\sim \rho_0}[(\hat b_0(s,a))^2]}.\\
\leq &
\sum_{h=0}^{H-2}\EE_{(\tilde s,\tilde a)\sim d^{\hat \pi^n}_{P^{\star} ,h}} \|\phi^{\star}_h(\tilde s,\tilde a)\|_{\Sigma_{\gamma_h,\phi^{\star}_h}^{-1}}
\sqrt{n|\Acal|\alpha_n^2\EE_{(s,a)\sim \rho_{h+1}}\bracks{\|\hat\phi_{h+1}\|^2_{\Sigma^{-1}_{\rho_{h+1},\hat\phi_{h+1}}}}
+ 4\lambda_n d} 
+\sqrt{|\Acal|\alpha_0^2\EE_{(s,a)\sim \rho_0}\left[\|\hat\phi_0\|^2_{\Sigma^{-1}_{\rho_0,\hat\phi_0}}\right]}.\\
\end{align*}

Note that we use the fact that $B=2$ when applying \pref{lem:tsbt2}.
In addition, we have that for any $h\in[H]$,
\begin{align*}
     n\EE_{(s,a)\sim \rho_h}\bracks{\|\hat \phi_h(s,a)\|^2_{\Sigma^{-1}_{\rho_h,\hat \phi_h}} }=n\text{Tr}(\EE_{\rho_h}[\hat \phi_h\hat \phi^{\top}_h]\{n\EE_{\rho_h}[\hat \phi_h\hat \phi^{\top}_h]+\lambda_n I\}^{-1} )\leq d.
\end{align*}
Then,
\begin{align*}
\sum_{h=1}^H\EE_{(s,a)\sim d^{\hat \pi^n}_{P^{\star},h }}[b_h(s,a)]
\leq
\sum_{h=0}^{H-2}\EE_{(\tilde s,\tilde a)\sim d^{\hat \pi^n}_{P^{\star} ,h}} \|\phi^{\star}_h(\tilde s,\tilde a)\|_{\Sigma_{\rho_h,\phi^{\star}_h}^{-1}}
\sqrt{|\Acal|\alpha_n^2d
+ 4\lambda_n d} 
+\sqrt{|\Acal|\alpha_1^2d/n}.
\end{align*}

Second, we calculate the term (b) in Eq.~\pref{eq:regret_middle}. Following \pref{lem:tsbt2} and noting $\|f_h(s,a)\|_\infty\leq 1$, we have 
\begin{align*}
& \sum_{h=0}^{H-1}\EE_{(s_h,a_h)\sim d^{\hat \pi^n}_{P^{\star},h}}\left[f_h(s_h,a_h)\right]\\
\leq & \sum_{h=0}^{H-2}\EE_{(\tilde s,\tilde a)\sim d^{\hat \pi^n}_{P^{\star} ,h}} \|\phi^{\star}_h(\tilde s,\tilde a)\|_{\Sigma_{\gamma_h,\phi^{\star}_h}^{-1}}
\sqrt{n|\Acal|\EE_{(s,a)\sim \rho_{h+1}}\bracks{f_{h+1}^2(s,a)}
+ 4\lambda_n d} 
+\sqrt{|\Acal|\EE_{(s,a)\sim \rho_0}[f_0^2(s,a)]}.\\
\leq & \sum_{h=0}^{H-2}\EE_{(\tilde s,\tilde a)\sim d^{\hat \pi^n}_{P^{\star} ,h}} \|\phi^{\star}_h(\tilde s,\tilde a)\|_{\Sigma_{\gamma_h,\phi^{\star}_h}^{-1}}
\sqrt{n|\Acal|\zeta_n
+ 4\lambda_n d} 
+\sqrt{|\Acal|\zeta_n}.\\
\end{align*}
where in the second inequality, we use $\EE_{s,a\sim\rho_h} [f_h^2(s,a)] \leq \zeta_n$.
Then, by combining the above calculation of the term (a) and term (b) in Eq.~\pref{eq:regret_middle}, we have:
\begin{align*}
&V^{\pi^{\star}}_{P^{{\star}},r} -V^{\hat \pi^n}_{P^{{\star}},r}\\
=&\sum_{h=0}^{H-1}\EE_{(s_h,a_h)\sim d^{\hat \pi^n}_{P^{\star},h}}\left[b_h(s_h,a_h)\right]+ 
(2H+1)\sum_{h=0}^{H-1}\EE_{(s_h,a_h)\sim d^{\hat \pi^n}_{P^{\star},h}}\left[f_h(s_h,a_h)\right]
+\sqrt{|\Acal| \zeta_n}\\
\leq &\sum_{h=0}^{H-2}\EE_{(\tilde s,\tilde a)\sim d^{\hat \pi^n}_{P^{\star} ,h}} \|\phi^{\star}_h(\tilde s,\tilde a)\|_{\Sigma_{\gamma_h,\phi^{\star}_h}^{-1}}
\sqrt{|\Acal|\alpha_n^2d
+ 4\lambda_n d} 
+\sqrt{|\Acal|\alpha_1^2d/n} +\\ &\quad(2H+1)\sum_{h=0}^{H-2}\EE_{(\tilde s,\tilde a)\sim d^{\hat \pi^n}_{P^{\star} ,h}} \|\phi^{\star}_h(\tilde s,\tilde a)\|_{\Sigma_{\gamma_h,\phi^{\star}_h}^{-1}}
\sqrt{n|\Acal|\zeta_n
+ 4\lambda_n d} 
+(2H+1)\sqrt{|\Acal|\zeta_n} +\sqrt{|\Acal| \zeta_n}.
\end{align*}
Hereafter, we take the dominating term out. First,  recall
\begin{equation*}
    \alpha_n= O(\sqrt{n|\Acal|^2\zeta_n+ 4\lambda_n d+ n \zeta_n})
\end{equation*}
Second, recall that $\gamma^n_{h}(s,a) = \frac{1}{n} \sum_{i=0}^{n-1} d^{\pi_i}_h(s,a)$, and thus 
\begin{align*}
 &\sum_{n=0}^{N-1} \EE_{(\tilde s,\tilde a)\sim d^{ \hat \pi^n}_{P^{{\star}},h }} \|\phi^{{\star}}(\tilde s,\tilde a)\|_{\Sigma_{\gamma^n_h,\phi_h^{{\star}}}^{-1}}   \leq \sqrt{N\sum_{n=1}^N \EE_{(\tilde s,\tilde a)\sim d^{\hat \pi^n}_{P^{{\star}},h }}[\phi_h^{{\star}}(\tilde s,\tilde a)^{\top}\Sigma^{-1}_{\gamma^n_h,\phi_h^{{\star}}}\phi_h^{{\star}}(\tilde s,\tilde a)]} \tag{CS inequality}\\
& \leq \sqrt{N\prns{\ln\det(\sum_{n=1}^N\EE_{(\tilde s,\tilde a)\sim d^{\hat \pi^n}_{P^{{\star}},h }}[\phi_h^{{\star}}(\tilde s,\tilde a)  \phi_h^{{\star}}(\tilde s,\tilde a)^{\top} ]  )  - \ln\det(\lambda_1 I)  }  }   \tag{\pref{lem:reduction} and $\lambda_1\leq \cdots\leq \lambda_N$}\\ 
&\leq \sqrt{dN \ln \prns{1+\frac{N}{d\lambda_1} }}.   \tag{Potential function bound, \pref{lem:potential} noting $\|\phi_h^{{\star}}(s,a)\|_2\leq 1$ for any $(s,a)$.}
\end{align*}
Finally, 
The \textsc{RepLearn} guarntee gives
\begin{align*}
    \zeta_n = O\left(d^{2}\sqrt{\frac{\log\left(dn|\Phi|/\delta\right)}{k}}\right)
\end{align*}
Combining all of the above, we have
\begin{align*}
    \sum_{n=1}^N V^{\pi^{\star}}_{P^{{\star}},r} -V^{\hat \pi^n}_{P^{{\star}},r}\leq&  O\left(H^2|\Acal|^{3/2}d^{2}n^{3/4}\log\left(dn|\Phi|/\delta\right)^{1/4}\right)
\end{align*}

This concludes the proof and gives us a sample complexity of $ O\left(\frac{H^8|\Acal|^6d^{8}\log\left(d|\Phi|/\delta\epsilon\right)}{\epsilon^{4}}\right)$.
\end{proof}

\section{Representation Learning Analysis}

In this section we prove \pref{lem:rep_learn_informal}.
Below we omit the superscript $n$ and subscript $h$ when clear from the context. Denote
\begin{align}
    \Lcal_{\lambda,\Dcal}(\phi,w,f) &= \frac{1}{|\Dcal|}\sum_{(s,a,s')\in \Dcal}\left(\phi(s,a)^\top w - f(s')\right)^2+\frac{\lambda}{|\Dcal|}\|w\|^2_2\\
    \Lcal_{\Dcal}(\phi,w,f) &= \frac{1}{|\Dcal|}\sum_{(s,a,s')\in \Dcal}\left(\phi(s,a)^\top w - f(s')\right)^2\\
    \Lcal_{\rho}(\phi,w,f) &= \EE_{(s,a)\sim\rho,s'\sim P^\star(\cdot\mid s,a)}\left(\phi(s,a)^\top w - f(s')\right)^2
\end{align}

The following lemma quantifies the complexity of our discriminator class $\Fcal_h$ using its covering number.
\begin{lemma}[Covering Number of $\Fcal_h$]\label{lem:cov_num}
The $\gamma$-covering number of Eq.~$\Fcal_h$ defined in \pref{eq:Fcal} is at most $2|\Phi_{h+1}|^2\cdot \left(\frac{2\alpha_N}{\sqrt{\lambda_N}\gamma}\right)^{2d}$.
\end{lemma}
\begin{proof}
Recall that the discriminator class is defined as follows:
\begin{align*}
    \Fcal^{(1)}_h &= \left\{f(s): =  \EE_{a\sim U(\Acal)}\left[\phi(s,a)^\top \theta-\phi'(s,a)^\top \theta'\right]  \;\Big\vert \; \phi,\phi'\in\Phi_{h+1}\; \max(\|\theta\|_\infty,\|\theta'\|_\infty)\leq 1\right\},\\
    \Fcal_h^{(2)} &= \left\{ f(s): \max_a \left( \frac{r_{h+1}(s,a) + \min\left\{ w^\top \phi(s,a)  , 2\right\}}{2H+1} + {w'}^{\top} \phi(s,a)     \right) \Big \vert  \phi,\phi'\in \Phi_{h+1};  \;\|w\|_\infty\leq c,\|w'\|_\infty\leq 1 \right\}\nonumber
\end{align*}
We cover $\Fcal^{(1)}_h$ and $\Fcal^{(2)}_h$ separately. For $\Fcal^{(1)}_h$, let $\Theta$ be an $\ell_\infty$-cover of the set $\{\theta\in\RR^d:\|\theta\|_\infty\leq 1\}$ at scale $\gamma$. Then, we know $|\Theta|\leq \left(\frac{2}{\gamma}\right)^d$. Define the $\gamma$-covering set of $\Fcal^{(1)}_h$ as
\begin{equation*}
    \tilde\Fcal^{(1)}_h = \left\{ s\mapsto \left(\EE_{a\sim U(\Acal)}\left[\phi(s,a)^\top \theta-\phi'(s,a)^\top \theta'\right] \right) \mid \phi,\phi'\in\Phi_{h+1}, \theta,\theta'\in\Theta\right\}
\end{equation*}
Then, we have that for any $f\in\Fcal^{(1)}_h$, there exists a $\tilde f\in \tilde \Fcal^{(1)}_h$, s.t. $\|f-\tilde f\|_\infty\leq \gamma$, where we use the fact that $\phi(s,a)$ are one-hot vectors, and we have $|\tilde \Fcal^{(1)}_h|\leq |\Phi_{h+1}|^2\cdot \left(\frac{2}{\gamma}\right)^{2d}$.

For $\Fcal^{(2)}_h$, similarly let $\Wcal$ be an $\ell_\infty$-cover of the set $\{w\in\RR^d:\|w\|_\infty\leq c\}$ at scale $\gamma$. Let $\Wcal'$ be an $\ell_\infty$-cover of the set $\{w'\in\RR^d:\|w'\|_\infty\leq 1\}$ at scale $\gamma$. Then, we know that $|\Wcal'|\leq \left(\frac{2}{\gamma}\right)^{d}$ and $|\Wcal|\leq \left(\frac{2c}{\gamma}\right)^{d},c=\alpha_N/\sqrt{\lambda_N}$.
Define the $\gamma$-covering set of $\Fcal^{(2)}_h$ as
\begin{equation*}
    \tilde\Fcal^{(2)}_h = \left\{s\mapsto \max_a \left( \frac{r_{h+1}(s,a) + \min\left\{ w^\top \phi(s,a)  , 2\right\}}{2H+1} + {w'}^{\top} \phi(s,a)     \right) \;\Big\vert\; \phi\in\Phi_{h+1}, w\in\Wcal, w'\in\Wcal'\right\}
\end{equation*}
Then, we have that for any $g\in\Gcal'$, there exists a $\tilde g'\in \tilde \Gcal'$, s.t. 
$\|g-\tilde g\|_\infty\leq \gamma$, 
and $|\tilde \Gcal'|\leq |\Phi_{h+1}|\cdot \left(\frac{2}{\gamma}\right)^{d}\left(\frac{2\alpha_N}{\sqrt{\lambda_N}\gamma}\right)^{d}$. 
So the $\gamma$-covering number of $\Fcal_h$ is
\begin{align*}
    |\tilde\Gcal|+|\tilde\Gcal'| =& |\Phi_{h+1}|^2\cdot \left(\frac{2}{\gamma}\right)^{2d}+|\Phi_{h+1}|\cdot \left(\frac{2}{\gamma}\right)^{d}\left(\frac{2\alpha_N}{\sqrt{\lambda_N}\gamma}\right)^{d}\\
    \leq & 2|\Phi_{h+1}|^2\cdot \left(\frac{2\alpha_N}{\sqrt{\lambda_N}\gamma}\right)^{2d}
\end{align*}
where the last step is due to $\alpha_N/\sqrt{\lambda_N}\geq 1$.
\end{proof}

\begin{lemma}[Uniform Convergence for Square Loss]\label{lem:fastrate_sqloss}
Let there be a dataset $\Dcal\coloneqq \{(s_i,a_i,s'_i)\}_{i=1}^k$ collected in $k$ episodes. Denote that the data generating distribution in iteration $i$ by $d_i$, and $\rho = \frac{1}{k}\sum_{i=0}^k d_i$. Note that $d_i$ can depend on the randomness in episodes $0,...,i-1$. For a finite feature class $\Phi$ and a discriminator class $\Fcal:\Scal\rightarrow [0,L]$ with $\gamma$-covering number $\|\Fcal\|_\gamma$, we will show that, with probability at least $1-\delta$:
\begin{align*}
    &\left|\left[\Lcal_\rho(\phi,w,f)-\Lcal_\rho(\phi^\star,\theta^\star_f,f)\right]-\left[\Lcal_\Dcal(\phi,w,f)-\Lcal_\Dcal(\phi^\star,\theta^\star_f,f)\right]\right|\\
    \leq& \frac{1}{2}\left[\Lcal_\rho(\phi,w,f)-\Lcal_\rho(\phi^\star,\theta^\star_f,f)\right]+\frac{28L^2\log(\frac{2(4k)^d\cdot|\Phi|\cdot\|\Fcal\|_{L/2k}}{\delta})}{k}
\end{align*}
for all $\phi\in\Phi$, $\|w\|_\infty\leq L$ and $f\in\Fcal$, where recall that $\phi^\star$ is the true feature and $\theta^\star_f$ is defined as $\EE_{s'\sim P^\star(\cdot|s,a)}[f(s')|s,a]=\langle \phi^{\star}(s,a),\theta_f^{\star} \rangle$. 
\end{lemma}
\begin{proof}
Note that in \textsc{RepLearn}, everything is happening at a fixed time step and we drop the time step indexing for brevity.
To start, we focus on a given $f\in\Fcal$.
We first give a high probability bound on the following deviation term:
\begin{align*}
    \abr{ \Lcal_{\rho}( \phi,w,f) - \Lcal_{\rho}(  \phi^\star,\theta^\star_f,f) - \rbr{ \Lcal_{\Dcal}( \phi,w,f) - \Lcal_{\Dcal}( \phi^\star,\theta^\star_f,f)}}.
\end{align*}
Denote $g(s_i,a_i) = \inner{\phi(s_i,a_i)}{w}$ and $g^\star(s_i,a_i) = \inner{\phi^\star(s_i,a_i)}{\theta^\star_f}$. 

At episode $i$, let $\Fcal_{i-1}$ be the $\sigma$-field generated by all the random variables over the first
$k-1$ episodes, for the random variable $Y_i \coloneqq \rbr{g(s_i,a_i) - f(s'_i)}^2 - \rbr{g^\star(s_i,a_i) - f(s'_i)}^2$, we have:
\begin{align*}
    \EE[Y_i|\Fcal_{i-1}] = {} & \EE \sbr{\rbr{g(s_i,a_i) - f(s'_i)}^2 - \rbr{g^\star(s_i,a_i) - f(s'_i)}^2} \\
    = {} & \EE \sbr {\rbr{g(s_i,a_i) + g^\star(s_i,a_i) - 2f(s'_i)} \rbr{g(s_i,a_i) - g^\star(s_i,a_i)} } \\
    = {} & \EE \sbr {\rbr{g(s_i,a_i) - g^\star(s_i,a_i)}^2 }.
\end{align*}
Here the conditional expectation is taken according to the distribution $[d_i|\Fcal_{i-1}]$. The last equality is due to the fact that
\begin{align*}
    \EE \sbr {\rbr{g^\star(s_i,a_i) - f(s'_i)} \rbr{g(s_i,a_i) - g^\star(s_i,a_i)}} = \EE_{s_i,a_i}\EE_{s'_i|s_i,a_i} \sbr {\rbr{g^\star(s_i,a_i) - f(s'_i)} \rbr{g(s_i,a_i) - g^\star(s_i,a_i)}}=0
\end{align*}

Next, for the conditional variance of the random variable, we have:
\begin{align*}
    \VV[Y_i|\Fcal_{i-1}] \le {} & \EE\sbr{Y_i^2|\Fcal_{i-1}} = {} \EE \sbr {\rbr{g(s_i,a_i) + g^\star(s_i,a_i) - 2f(s'_i)}^2 \rbr{g(s_i,a_i) - g^\star(s_i,a_i)}^2 } \\
    \le {} & 16L^2 \EE\sbr{\rbr{g(s_i,a_i) - g^\star(s_i,a_i)}^2} \\
    \le {} & 16L^2 \EE[Y_i|\Fcal_{i-1}].
\end{align*}
Noticing $Y\in[-4L^2, 4L^2]$. 

From here on, we use $\EE[Y]$ to denote the conditional expectation and $\VV[Y]$ to denote the conditional variance for all $Y_i$, since they are all the same.
Now, applying Azuma-Bernstein's inequality on $Y_1+...+Y_k$ with respect to filtration $\{\Fcal_k\}_{k\geq 0}$, with probability at least $1-\delta'$, we can bound the deviation term above as:
\begin{align*}
    & \abr{\Lcal_{\rho}( \phi,w,f) - \Lcal_{\rho}(  \phi^\star,\theta^\star_f,f) - \rbr{ \Lcal_{\Dcal}( \phi,w,f) - \Lcal_{\Dcal}( \phi^\star,\theta^\star_f,f)}} \\  
    \le{}&  \sqrt{\frac{2\VV[Y]\log \frac{2}{\delta'}}{k}} + \frac{16L^2 \log \frac{2}{\delta'}}{3k}  \\
    \le{}&  \sqrt{\frac{32L^2\EE[Y] \log \frac{2}{\delta'}}{k}} + \frac{16L^2 \log \frac{2}{\delta'}}{3k}
\end{align*}
where in the last inequality is obtained by choosing $\gamma = \frac{L}{2k}$.

Further, consider a finite point-wise cover of the function class $\Gcal \coloneqq \{g(s,a) = \inner{\phi(s,a)}{w} : \phi \in \Phi, \|w\|_\infty \le L\}$. Note that, with a $\ell_\infty$-cover $\overline{\Wcal}$ of $\Wcal = \{\|w\|_\infty \le L\}$ at scale $\gamma$, we have for all $(s,a)$ and $\phi \in \Phi$, there exists $\bar{w} \in \overline{\Wcal}$, $|\langle \phi(s,a), w - \bar{w}\rangle| \le \gamma$, and we have $|\Wcal|=\left(\frac{2L}{\gamma}\right)^d$.
Let $\tilde\Fcal$ be a $\gamma$-covering set of $\Fcal$.

Then, applying a union bound over elements in $\Phi\times\overline{\Wcal}\times\tilde \Fcal$, with probability $1-|\Phi|\cdot|\overline{\Wcal}|\cdot|\tilde \Fcal|\delta'$, for all $w \in \Wcal$, $f \in \Fcal$, we have:
\begin{align*}
    & \abr{\Lcal_{\rho}( \phi,w,f) - \Lcal_{\rho}(  \phi^\star,\theta^*_f,f) - \rbr{ \Lcal_{\Dcal}( \phi,w,f) - \Lcal_{\Dcal}( \phi^\star,\theta^*_f,f)}} \\
    \le{}&  \abr{\Lcal_{\rho}( \phi,\bar{w},f) - \Lcal_{\rho}(  \phi^\star,\theta^\star_f,f) - \rbr{ \Lcal_{\Dcal}( \phi,\bar{w},f) - \Lcal_{\Dcal}( \phi^\star,\theta^\star_f,f)}} + 4L\gamma \\
    \le{}&  \sqrt{\frac{32L^2\EE[Y] \log \frac{2}{\delta'}}{k}}  + \frac{16L^2 \log \frac{2}{\delta'}}{3k} + 4L\gamma \\
    \le{}&  \frac{1}{2}\EE[Y_{\bar w}] + \frac{16L^2 \log \frac{2}{\delta'}}{k}+ \frac{16L^2 \log \frac{2}{\delta'}}{3k} + 4L\gamma\\
    \le{}&  \frac{1}{2}\EE[Y_{w}] + 2L\gamma+\frac{16L^2 \log \frac{2}{\delta'}}{k}+ \frac{16L^2 \log \frac{2}{\delta'}}{3k} + 4L\gamma\\
    \le{}&  \frac{1}{2}\rbr{\Lcal_{\rho}( \phi,w,f) - \Lcal_{\rho}(  \phi^\star,\theta^\star_f,f)} + \frac{22L^2 \log \frac{2}{\delta'}}{k} + 6L\gamma\\
    \le{}&  \frac{1}{2}\rbr{\Lcal_{\rho}( \phi,w,f) - \Lcal_{\rho}(  \phi^\star,\theta^\star_f,f)} + \frac{28L^2 \log \frac{2}{\delta'}}{k} \tag{setting $\gamma = L/k$}
\end{align*}
where we add subscript to $Y$ to distinguish $Y_{\bar w} \coloneqq \rbr{\inner{\phi(s,a)}{\bar w} - f(s')}^2 - \rbr{g^*(s,a) - f(s')}^2$ from $Y_{ w} \coloneqq \rbr{\inner{\phi(s,a)}{w} - f(s'_i)}^2 - \rbr{g^*(s_i,a_i) - f(s'_i)}^2$.

Finally, setting $\delta = \delta'/\rbr{|\Phi||\overline{\Wcal}||\tilde\Fcal|}$, we get $\log \frac{2}{\delta'}\le \log \frac{2(4k)^{d}|\Phi||\tilde\Fcal|}{\delta}$. This completes the proof.
\end{proof} 

We will see now how the above lemma can be adapted to the regularized objective.

Below, we use $\hat w^t_{f}$ to denote $ \argmin_{w}\Lcal_{\lambda,\Dcal}(\phi^t,w,f)$, $\hat w^t_{i}$ a shorthand for $\hat w^t_{f^i}$, and $\theta^{\star}_i = \argmin_{\theta}\Lcal_{\rho}(\phi^\star,\theta,f^i)$.

\begin{lemma}[Deviation Bounds for \pref{alg:rep_learn}]\label{lem:dev_bound} Let $\tilde \epsilon = \frac{56L^2\log(\frac{2(4k)^d\cdot|\Phi|\cdot\|\Fcal\|_{L/2k}}{\delta})}{k}$. If \pref{alg:rep_learn} is called with a dataset $\Dcal$ of size $k$ and terminal loss cutoff $\ell = \frac{3}{2}\epsilon_1+\tilde \epsilon+ \frac{2\lambda L^2d}{k}$, then with probability at least $1-\delta$, for any $f\in\Fcal\subset (S\to [0,L])$ and $t\leq T$, we have
\begin{align*}
    \sum_{i\leq t}\EE_{\rho}\left[\left(\phi^t(s,a)^\top \hat w^t_i-\phi^{\star}(s,a)^\top \theta^{\star}_i\right)^2\right]&\leq t\left(\tilde{\epsilon} + \frac{2\lambda L^2d}{k}\right)\\
    \EE_{\rho}\left[\left(\phi^t(s,a)^\top w-\phi^{\star}(s,a)^\top \theta^{\star}_{t+1}\right)^2\right]&\geq \epsilon_1 \text{, for all }w\in\RR^d.
\end{align*}
Furthermore, at termination, the learned feature $\phi^t$ satisfies:
\begin{equation*}
    \max_{f\in\Fcal} \EE_{\rho} \left[(\phi^{t\top} \hat w^t_f-\phi^{{\star}\top}\theta^{\star}_f)^2\right]\leq 3\epsilon_1 + 3\tilde{\epsilon} + \frac{4\lambda L^2d}{k}.
\end{equation*}
\end{lemma}

\begin{proof}
We begin by using the result in \pref{lem:fastrate_sqloss} such that, with probability at least $1-\delta$, for all $\|w\|_\infty \le L$, $\phi \in \Phi$ and $f \in \Fcal$, we have
\begin{align*}
\left|\left[\Lcal_\rho(\phi,w,f)-\Lcal_\rho(\phi^{\star},\theta^{\star}_f,f)\right]-\left[\Lcal_\Dcal(\phi,w,f)-\Lcal_\Dcal(\phi^{\star},\theta^{\star}_f,f)\right]\right|\leq \frac{1}{2}\left[\Lcal_\rho(\phi,w,f)-\Lcal_\rho(\phi^{\star},\theta^{\star}_f,f)\right]+\tilde\epsilon/2.
\end{align*}

Thus, for the feature selection step in iteration $t$, with probability at least $1-\delta$ we have: 
\begin{align*}
    &\sum_{i\leq t} \EE_\rho \sbr{\rbr{ \phi^{t\top} \hat w^t_i - \phi^{{\star}\top} \theta^{\star}_i}^2}\\
    = {} & \sum_{i\leq t}\rbr{ \Lcal_\rho(\phi^t, \hat w^t_i, f^i) - \Lcal_\rho(\phi^{\star}, \theta^{\star}_i, f^i)} \tag{since $\EE_{s'\sim P^\star(s,a)} f^i  = (\theta_i^\star)^\top \phi^*(s,a)$}\\
    \le {} & \sum_{i\leq t} 2  \rbr{ \Lcal_{\Dcal}( \phi^t,\hat w^t_i,f^i) - \Lcal_{\Dcal}( \phi^{\star},\theta^{\star}_i,f^i)} + t\tilde{\epsilon} \tag{\pref{lem:fastrate_sqloss}, and $\|\hat w^t_i\|_\infty\leq L$ by \pref{lem:bound_w}}\\
    \leq {} & \sum_{i\leq t} 2  \rbr{ \Lcal_{\lambda,\Dcal}( \phi^t,\hat w^t_i,f^i) - \Lcal_{\lambda,\Dcal}( \phi^{\star},\theta^{\star}_i,f^i)+\frac{\lambda}{k}\|\theta_i^\star\|^2_2} + t\tilde{\epsilon} \\
    \le {} & t\left(\tilde{\epsilon} + \frac{2\lambda L^2d}{k}\right)\tag{by the optimality of $\phi^t, \hat w^t_i$ under $\Lcal_{\lambda,\Dcal}(\cdot,\cdot,f^i)$, see \pref{alg:rep_learn}~\featurestep},
\end{align*} 
which means the first inequality in the lemma statement holds. Here, we use $\|\theta^*_i\|^2_2\leq L^2d$, which is easily derived using the block MDP assumption. 

For the discriminator selected at iteration $t$, let $\bar{w} \coloneqq \argmin_{w} \Lcal_{\lambda,\Dcal} (\phi^t, w, f^{t+1})$. Using the same sample size for the adversarial test function at each non-terminal iteration with loss cutoff $\ell$, for any vector $w \in \RR^d$, $\|w\|_\infty\leq L$ we get:
\begin{align*}
    &\EE_\rho \sbr{\rbr{ \phi^{t\top} w - \phi^{{\star}\top} \theta^{\star}_{t+1}}^2} \\
    = {} & \Lcal_\rho (\phi^t, w, f^{t+1}) - \Lcal_\rho (\phi^{\star}, \theta^{\star}_{t+1}, f^{t+1}) \\
    \ge {} & \frac{2}{3} \rbr{\Lcal_{\Dcal} (\phi^t, w, f^{t+1}) - \Lcal_{\Dcal} (\phi^{\star}, \theta^{\star}_{t+1}, f^{t+1})} - \frac{\tilde{\epsilon}}{3}\tag{\pref{lem:fastrate_sqloss}}\\
    \ge {} & \frac{2}{3} \rbr{\Lcal_{\lambda,\Dcal} (\phi^t, w, f^{t+1}) - \Lcal_{\lambda,\Dcal} (\phi^{\star}, \theta^{\star}_{t+1}, f^{t+1})-\frac{\lambda L^2d}{k}} - \frac{\tilde{\epsilon}}{3}\\
    \ge {} & \frac{2}{3} \rbr{\Lcal_{\lambda,\Dcal} (\phi^t, \bar w, f^{t+1}) - \Lcal_{\lambda,\Dcal} (\phi^{\star}, \theta^{\star}_{t+1}, f^{t+1})-\frac{\lambda L^2d}{k}} - \frac{\tilde{\epsilon}}{3}\tag{$\bar{w} \coloneqq \argmin_{w} \Lcal_{\lambda,\Dcal} (\phi^t, w, f^{t+1})$}\\
    \ge {} & \frac{2\ell}{3} + \frac{2}{3} \rbr{\min_{\tilde{\phi} \in \Phi_h, \tilde w} \Lcal_{\lambda,\Dcal} ( \tilde \phi, \tilde w, f^{t+1}) - \Lcal_{\lambda,\Dcal} (\phi^{\star}, \theta^{\star}_{t+1}, f^{t+1})} -\frac{2\lambda L^2d}{3k} - \frac{\tilde{\epsilon}}{3}\\
    \ge {} & \frac{2\ell}{3} + \frac{2}{3} \rbr{\min_{\tilde{\phi} \in \Phi_h, \tilde w} \Lcal_{\Dcal} ( \tilde \phi, \tilde w, f^{t+1}) - \Lcal_{\Dcal} (\phi^{\star}, \theta^{\star}_{t+1}, f^{t+1})-\frac{\lambda L^2d}{k}} -\frac{2\lambda L^2d}{3k} - \frac{\tilde{\epsilon}}{3}\\
    \ge {} & \frac{2\ell}{3} + \frac{1}{3} \rbr{\Lcal (\tilde \phi_{t+1}, \tilde w_{t+1}, f^{t+1}) - \Lcal (\phi^{\star}, \theta^{\star}_{t+1}, f^{t+1})} -\frac{4\lambda L^2d}{3k}- \frac{2\tilde{\epsilon}}{3} \tag{ $\|\tilde w_{t+1}\|_\infty\leq L$,  $\|\theta^*_{t+1}\|_\infty\leq L$, \pref{lem:fastrate_sqloss}}\\ 
    \ge {} & \frac{2\ell}{3} -\frac{4\lambda L^2d}{3k} - \frac{\tilde{2\epsilon}}{3}. 
\end{align*}
where $(\tilde \phi_{t+1}, \tilde w_{t+1})$ denote $\argmin_{\tilde{\phi} \in \Phi_h, \tilde w} \Lcal_{\lambda,\Dcal} ( \tilde \phi, \tilde w, f^{t+1})$.
In the first inequality, we invoke \pref{lem:fastrate_sqloss} to move to empirical losses. In the fourth inequality, we add and subtract the bias correction term along with the fact that the termination condition is not satisfied for $f^{t+1}$. In the next step, we again use \pref{lem:fastrate_sqloss} for the bias correction term for $f^{t+1}$.

Thus, if we set the cutoff $\ell$ for test loss to $3\epsilon_1/2 + \tilde{\epsilon} + \frac{2\lambda L^2d}{k}$, for a non-terminal iteration $t$, for any $w \in \RR^d$ with $\|w\|_\infty \le L$, we have:
\begin{align}\label{eq:bounded_gap}
    \EE_\rho \sbr{\rbr{ \phi^{t\top} w - \phi^{{\star}\top} \theta^{\star}_{t+1}}^2} \ge \epsilon_1.
\end{align}
Now, since we know $\phi^{{\star}\top} \theta^{\star}_{t+1}=\EE_{s'\sim P^\star(s,a)} f^{t+1}\leq L$, and we know the Bayes optimal solution
\begin{equation*}
    w_{\phi^t} = \argmin_{w\in\RR^d}\EE_\rho \sbr{\rbr{ \phi^{t\top} w - \phi^{{\star}\top} \theta^{\star}_{t+1}}^2}
\end{equation*}
satisfies $\|w_{\phi^t}\|_\infty\leq L$ by \pref{lem:bound_w}. Therefore, Eq.~\eqref{eq:bounded_gap} also applies to $w_{\phi^t}$, and since $w_{\phi^t}$ is the minimizer, we have in fact for all $w\in\RR^d$, 
\begin{equation*}
    \EE_\rho \sbr{\rbr{ \phi^{t\top} w - \phi^{{\star}\top} \theta^{\star}_{t+1}}^2} \ge \epsilon_1.
\end{equation*}
which implies the second inequality in the lemma statement holds.

At the same time, for the last iteration $t$, for all $f \in \Fcal$, define $\hat w_f = \argmin_{w} \Lcal_{\lambda,\Dcal}(\phi^t,w,f)$, then the feature $\phi^t$ satisfies:
\begin{align*}
    & \Lcal_\rho(\phi^t,\hat w_f,f)-\Lcal_\rho(\phi^{\star},\theta^{\star}_f,f)\\ 
    \le {} & 2\rbr{\Lcal_{\lambda,\Dcal}(\phi^t, \hat w_f, f) - \Lcal_{\lambda,\Dcal}(\phi^{\star}, \theta^{\star}_f, f)+\frac{\lambda L^2d}{k}} + \tilde{\epsilon}\tag{\pref{lem:fastrate_sqloss}, and $\|\theta^*_f\|_\infty\leq L$ by \pref{lem:bound_w}}\\
    \le {} & 2 \rbr{\Lcal_{\lambda,\Dcal}(\phi^t, \hat w_f, f) - \Lcal_{\lambda,\Dcal}(\tilde{\phi}_f, \tilde{w}_f, f) + \Lcal_{\lambda,\Dcal}(\tilde{\phi}_f, \tilde{w}_f, f) - \Lcal_{\lambda,\Dcal}(\phi^{\star}, \theta^{\star}_f, f)} +\frac{2\lambda L^2d}{k} + \tilde{\epsilon}\\
    \le {} & 2 \rbr{\Lcal_{\lambda,\Dcal}(\phi^t, \hat w_f, f) - \Lcal_{\lambda,\Dcal}(\tilde{\phi}_f, \tilde{w}_f, f)} +\frac{2\lambda L^2d}{k} + \tilde{\epsilon} \tag{optimality of $\tilde{\phi}_f, \tilde{w}_f$ on $\Lcal_{\lambda,\Dcal}(\cdot,\cdot,f)$.}\\
    \le {} & 2\ell + \frac{2\lambda L^2d}{k} + \tilde{\epsilon}\\
    = {} & 3\epsilon_1 + 3\tilde{\epsilon} + \frac{4\lambda L^2d}{k}.
\end{align*}
This gives us the third inequality in the lemma, thus completes the proof.
\end{proof}

\begin{theorem}[Sample and Iteration Complexity of \pref{alg:rep_learn}]\label{thm:rep_learn} Let $\epsilon_1$ be set to $ 16\sqrt{2}Ld^{3/2}\tilde\epsilon^{1/2}$ and the termination threshold be set to $\ell = \frac{3}{2}\epsilon_1+\tilde \epsilon+ \frac{2\lambda L^2d}{k}$ as in \pref{lem:dev_bound}, then \pref{alg:rep_learn} terminates in at most $T = \sqrt{\frac{L^2d}{2\tilde\epsilon}}$ iterations, and returns a $\phi^t$ such that
\begin{equation*}
    \max_{f\in\Fcal} \EE_{\rho} \left[(\phi^{t\top} \hat w_f-\phi^{{\star}\top}\theta^{\star}_f)^2\right]\leq 75Ld^{3/2}\tilde\epsilon^{1/2}.
\end{equation*}
For $\Fcal_h$ defined in \pref{lem:cov_num}, we have
\begin{align*}
    \max_{f\in\Fcal_h} \EE_{\rho} \left[(\phi^{t\top} \hat w_f-\phi^{{\star}\top}\theta^{\star}_f)^2\right]\leq \zeta_n\coloneqq O\left(d^{2}\sqrt{\frac{\log\left(dk|\Phi|/\delta\right)}{k}}\right)
\end{align*}
\end{theorem}
\begin{proof}
We first incur \pref{lem:dev_bound} and get that when $\ell = \frac{3}{2}\epsilon_1+\tilde \epsilon+ \frac{2\lambda L^2d}{k}$,
\begin{align}
    \sum_{i\leq t}\EE_{\rho}\left[\left(\phi^t(s,a)^\top \hat w^t_i-\phi^{\star}(s,a)^\top \theta^{\star}_i\right)^2\right]&\leq t\left(\tilde{\epsilon} + \frac{2\lambda L^2d}{k}\right)\leq 2t\tilde\epsilon\label{eq:1}\\
    \EE_{\rho}\left[\left(\phi^t(s,a)^\top w-\phi^{\star}(s,a)^\top \theta^{\star}_{t+1}\right)^2\right]&\geq \epsilon_1\label{eq:2}
\end{align}
for all $w$ and $t\leq T$.

At round $t$, for functions $f^1,\ldots,f^t \in \Fcal$ in \pref{alg:rep_learn}, let $\theta^{\star}_i = \theta^{\star}_{f^i}$ as before and further let $\Lambda_t = \sum_{i=1}^t \theta^\star_i\theta^{\star\top}_i + \lambda' I_{d \times d}$.

Further, let $W_t = [w_{t,1} \mid w_{t,2} \mid \ldots \mid w_{t,t}] \in \RR^{d \times t}$ be the matrix with columns $W_t^i$ as the linear parameter $\hat w^t_i = \argmin_{w} \Lcal_{\lambda, \Dcal}(\phi^t, w, f^i)$. Similarly, let $A_t = [\theta^{\star}_1 \mid \theta^{\star}_2 \mid \ldots \mid \theta^{\star}_t]$.

Using the linear parameter $\theta^{\star}_{t+1}$ of the adversarial test function $f^{t+1}$, define $\hat{w}_t = W_t A_t^\top \Lambda_t^{-1} \theta^{\star}_{t+1}$. For this $\hat w_t$, we can bound its norm as: 
\begin{align*}
    \|W_t A_t^\top \Lambda_t^{-1} \theta^{\star}_{t+1}\|_2 \le \|W_t\|_2 \|A_t^\top \Lambda_t^{-1}\|_2 \|\theta^{\star}_{t+1}\|_2 \le L^2d\sqrt{\frac{t}{4\lambda'}}.
\end{align*}
Here, $\|W_t\|_2 \le L\sqrt{dt}$, $\|\theta^{\star}_{t+1}\|_2 \le L\sqrt{d}$, and $\|A_t^\top \Lambda_t^{-1}\|_{\text{op}}$ can be shown to be less than $\sqrt{1/4\lambda'}$
\footnote{Applying SVD decomposition and the property of matrix norm, $\|A_t^\top \Lambda_t^{-1}\|_{\text{op}}$ can be upper bounded by $\max_{i \le d} \frac{\sqrt{\lambda_i}}{\lambda_i + \lambda'} \le \frac{1}{\sqrt {4\lambda'}}$ where $\lambda_i$ are the eigenvalues of $A_t A_t^\top$ and the final inequality holds by AM-GM.}. From Eq.~\pref{eq:2}, we have
\begin{align*}
    \epsilon_1 \leq \EE \sbr{\rbr{ \phi^{t\top} \hat{w}_t - \phi^{{\star}\top} \theta^{\star}_{t+1}}^2} = {} & \EE \sbr{\rbr{\hat{\phi}_t^\top W_t A_t^\top \Lambda_t^{-1} \theta_{t+1}^{\star} - \phi^{\star}{}^\top \Lambda_t \Lambda_t^{-1}\theta_{t+1}^{\star}}^2}\\
    \leq {} & \| \Lambda_t^{-1} \theta_{t+1}^{\star}\|_2^2 \cdot \EE \sbr{\| \phi^{t\top} W_t A_t^\top  - \phi^{\star}{}^\top \Lambda_t\|_2^2}\\
    \leq {} & 2\| \Lambda_t^{-1} \theta_{t+1}^{\star}\|_2^2 \cdot  \EE \sbr{\| \phi^{t\top} W_t A_t^\top  - \phi^{t\top} A_t A_t^\top \|_2^2 + \lambda^{\prime 2} \|\phi^{\star}{}^\top\|_2^2} \\
    \leq {} & 2 \| \Lambda_t^{-1} \theta^{\star}_{t+1}\|_2^2 \cdot \left(\sigma_1^2(A_t) \EE\sbr{\|\phi^{t\top} W_t - \phi^{{\star}}{}^{\top} A_t\|_2^2} + \lambda^{\prime 2}\right)\\
    \leq {} & 2 \|\Lambda_t^{-1} \theta^{\star}_{t+1}\|_2^2 \cdot\left( 2L^2dt^2\tilde\epsilon + \lambda^{\prime 2}\right).
\end{align*}
The second inequality uses $(a+b)^2 \leq 2a^2 + 2b^2$. The last inequality applies the upper bound $\sigma_1(A_t) \le L\sqrt{dt}$ and the guarantee from Eq.~\pref{eq:1}. Using the fact that $t \leq T$, this implies that
\begin{align*}
    \|\Lambda_t^{-1}\theta^{\star}_{t+1}\|_2 \geq \sqrt{ \frac{\epsilon_1}{2(2L^2dT^2\tilde\epsilon + \lambda^{\prime 2})}}.
\end{align*}

We now use the generalized elliptic potential lemma to upper bound the total value of $\|\Lambda_t^{-1}\theta^{\star}_{t+1}\|_2$. From \pref{lem:gen_elliptic_pot}, if we set $\lambda' = L^2d$ and we do not terminate in $T$ rounds, then
\begin{align*}
    T\sqrt{ \frac{\epsilon_1}{2(2L^2dT^2\tilde\epsilon + \lambda^{\prime 2})}} \leq \sum_{t=1}^T \|\Lambda_t^{-1}\theta^{\star}_{t+1}\|_2 \leq 2 \sqrt{\frac{Td}{\lambda'}}.
\end{align*}
From this chain of inequalities, we can deduce
\begin{align*}
    T\epsilon_1 \leq 8 (d/\lambda')  \left( 2L^2dT^2\tilde\epsilon + \lambda^{\prime 2}\right),
\end{align*}
therefore
\begin{align*}
T \leq \frac{8d\lambda'}{\epsilon_1 - 16 L^2d^2T\tilde\epsilon/\lambda'}. \label{eq:T_upper_bound}
\end{align*}
Now, if we set $\epsilon_1 = 32L^2d^2T\tilde\epsilon/\lambda'$ in the above inequality, we can deduce that
\begin{align*}
    T \leq \frac{\lambda^{\prime 2}}{2L^2dT\tilde\epsilon}\implies T\leq\sqrt{\frac{L^2d}{2\tilde\epsilon}}.
\end{align*}
and the proper value for $\epsilon_1$ is $ 16\sqrt{2}Ld^{3/2}\tilde\epsilon^{1/2}$. The sample complexity then readily follows from \pref{lem:dev_bound}.

Combining with \pref{lem:cov_num}, and noting that $\|f\|_\infty\leq 2$ for all $f\in\Fcal_h$, we have that
\begin{align*}
    \max_{f\in\Fcal_h} \EE_{\rho} \left[(\phi^{t\top} \hat w_f-\phi^{{\star}\top}\theta^{\star}_f)^2\right]\leq \zeta_n\coloneqq O\left(d^{2}\sqrt{\frac{\log\left(dk|\Phi|/\delta\right)}{k}}\right)
\end{align*}
which gives us \pref{lem:rep_learn_informal}.
\end{proof}

\section{Reward-free Exploration}
\begin{algorithm}[t!] 
\caption{Reward-free \algname} \label{alg:reward_free}
\begin{flushleft}
  \textbf{Exploration Phase:}
\end{flushleft}
\begin{algorithmic}[1]
  \STATE {\bf  Input:} Representation classes $\{\Phi_h\}_{h=0}^{H-1}$, discriminator classes $\{\Fcal_h\}_{h=0}^{H-1}$, parameters $N, T_n,\alpha_n,\lambda_n$
  \STATE Initialize policy $\hat\pi^0 = \{\pi_0,\dots, \pi_{H-1}\}$ arbitrarily and replay buffers $\Dcal_h = \emptyset, {\Dcal}'_h = \emptyset$ for all $h$
  \FOR{$n = 1 \to N$}
    \STATE Data collection from $\hat \pi^{n-1}$: $\forall h \in [H]$,\\
    \quad $ s\sim d^{\hat \pi^{n-1}}_h, a\sim U(\Acal), s' \sim P^\star_h(s,a);$ \\
    \quad $\tilde{s} \sim d^{\hat \pi^{n-1}}_{h-1}, \tilde a\sim U(\Acal), \tilde s' \sim P^\star_{h-1}(\tilde s,\tilde a)$, $\tilde a'\sim U(\Acal), \tilde s''\sim P^\star_h(\tilde s',\tilde a')$\\
    \quad $\Dcal_h = \Dcal_h \cup \{s,a,s'\}$ and $\Dcal_h' = \Dcal_h' \cup \{\tilde s',\tilde a',\tilde s''\}$.
    \label{line:data}
    \STATE Learn representations for all $h\in[H]$:\\
        $\hat\phi^n_h = \textsc{RepLearn}\left(\Dcal_h\cup\Dcal'_h, \Phi_h, \Fcal_h, \lambda_n, T_n, \ell_n\right)$
    \vspace{0.1cm}
    \STATE Define exploration bonus for all $h\in[H]$:\\
        $\hat b^n_h(s,a) := \min\Big\{ \alpha_n \sqrt{ {\hat \phi_h^n}(s,a)^\top \Sigma^{-1}_{h} {\hat \phi_h^n}(s,a)  }, 2 \Big\}$,  
    with $\Sigma_h := \sum_{s,a,s'\sim \Dcal_h} \hat\phi^n_h(s,a) {\hat\phi^n_h}(s,a)^{\top} + \lambda_n I$.
    \label{line:bonus}
    \STATE $\left(\hat\pi^n,\hat V^n_0(s_0)\right)\leftarrow \textsc{LSVI}\big( \{\hat b^n_h\}_{h=0}^{H-1}, \{ \hat\phi^n_h \}_{h=0}^{H-1}, \{ \Dcal_h \cup \Dcal'_h \}_{h=0}^{H-1}, \lambda_n \big)$.
  \ENDFOR 
  \STATE $\tilde n = \argmin_n \hat V^n_0(s_0)$.
\end{algorithmic}
 \begin{flushleft}
  \textbf{Planning Phase:}
\end{flushleft}
\begin{algorithmic}[1]
  \STATE {\bf  Input:} Reward function $r_h(s,a)$.
  \STATE $\hat\pi\leftarrow \textsc{LSVI}\big( \{r_h+\hat b^{\tilde n}_h\}_{h=0}^{H-1}, \{ \hat\phi^{\tilde n}_h \}_{h=0}^{H-1}, \{ \Dcal^{\tilde n}_h \cup \Dcal^{\prime\tilde n}_h \}_{h=0}^{H-1}, \lambda_{\tilde n} \big)$.
\end{algorithmic}
\end{algorithm}
\paragraph{Exploration Stage}
Run \algname~with reward function set to zero and discriminator class set to
\begin{align*} 
    &\Big\{f(s){=}  \EE_{a\sim U(\Acal)}\big[\phi(s,a)^\top \theta-\phi'(s,a)^\top \theta'\big] \;|\;\mbox{$\phi,\phi'\in\Phi_{h+1}$}\; \max(\|\theta\|_\infty,\|\theta'\|_\infty){\leq} 1\Big\}\bigcup\\
    &\Big\{ f(s){=} \tfrac{\phi'(s,\pi'(s))^\top \theta + \min\{ w^\top \phi(s,\pi(s))  , 2\}}{2H+1} + {w'}^{\top} \phi(s,\pi(s))     \;|\;  \phi,\phi'\in \Phi_{h+1};  \|w\|_\infty,\|w'\|_\infty,\|\theta\|_\infty\leq c;\pi,\pi'\in\tilde\Pi \Big\}
\end{align*}
where
\begin{align*}  
    \tilde\Pi = \left\{f(\psi(s)) \;|\; f:\Zcal\rightarrow \Acal, \psi\in\Psi\right\}
\end{align*}
Then main difference to the discriminator class in the reward-driven setting is that we replace the known reward function $r_{h+1}(s,a)$ in $\Fcal^{(2)}$ with a linear term $\phi'(s,a)^\top \theta$ which covers all possible reward functions that depends on the latent states and actions, i.e., $r_{h+1}(z,a)$. We also replace $\max_a$ by a pre-defined policy class that covers all possible policies $\hat\pi$ that can be returned by \pref{alg:reward_free}. $\tilde\Pi$ has cardinality $|\Phi|\cdot |\Acal|^{|\Zcal|}$. So the log-covering number of $\Fcal_h$ does not change more than constant factors and \pref{lem:rep_learn_informal} remains to hold.

We collect all historical version of the datasets as $\{\Dcal_h^n\}$ and $\{\Dcal_h^{n\prime}\}$, bonuses as $\{\hat b_h^n\}$ and learned features as $\{\phi_h^n\}$, for all $h\in[H]$ and $n\in[N]$.

\paragraph{Planning Stage}
The key observation here is that $V^{\pi^n}_{\hat P^n,\hat b^n}$ can actually observed as the returned value $\hat V^n_0(s_0)$ of LSVI for each iteration $n$ (assuming a unique starting state). Therefore, we can choose $\tilde n = \argmin_n V^{\pi^n}_{\hat P^n,\hat b^n}$ and let the exploration stage output $\hat P^{\tilde n},\hat b^{\tilde n}$.

Then, during planning stage, simply return $\hat\pi=\argmax_{\pi}V^{\pi}_{\hat P^{\tilde n},r+\hat b^{\tilde n}}$, which is done by LSVI.
\begin{lemma}\label{lem:rf1}
For any $n\in[N]$,
\begin{align*}
    NV^{\hat\pi^{\tilde n}}_{\hat P^n,\hat b^{\tilde n}}\leq\sum_{n=0}^{N-1} V^{\hat\pi^n}_{\hat P^n,\hat b^n}\leq O\left(H^{5/2}|\Acal|^{1/2}d\log(|\Phi|/\delta)^{1/4}N^{3/4}\right).
\end{align*}
\end{lemma}
\begin{proof}
This follows the exact same proof as \pref{thm:pseudo_regret} (starting from the 3rd equation), plus noting that during exploration, since the reward is $0$, $V^{\pi^n}_{P^\star,r}=0$.
\end{proof}

\begin{lemma}\label{lem:rf2}
We have
\begin{align*}
    V^{\hat\pi}_{\hat P^n,r+\hat b^n}-V^{\hat\pi}_{P^\star,r}\leq C\cdot V^{\hat\pi}_{\hat P^n,\hat b^n} + \sqrt{|\Acal|\zeta_n}
\end{align*}
for an absolute constant $C$.
\end{lemma}
\begin{proof}
Below we drop the superscript $n$.
First we apply simulation lemma (\pref{lem:pdl})
\begin{align*}
      & V^{\hat\pi}_{\hat P,r+\hat b}- V^{\hat\pi}_{P^{{\star}},r}\\
    = & \sum_{h=0}^{H-1}\EE_{(s_h,a_h)\sim d^{\hat\pi}_{\hat P,h}}\left[\hat b_h(s_h,a_h)+ \EE_{\hat P_h(s'_{h}\mid s_h,a_h)}[V^{\hat\pi}_{P^{\star},r,h+1}(s_h')]-\EE_{P^{\star}_h(s'_{h}\mid s_h,a_h)}[V^{\hat\pi}_{P^{\star},r,h+1}(s_h')] \right]\\
    = & V^{\hat\pi}_{\hat P,\hat b} +\sum_{h=0}^{H-1}\EE_{(s_h,a_h)\sim d^{\hat\pi}_{\hat P,h}}\left[\underbrace{\EE_{\hat P_h(s'_{h}\mid s_h,a_h)}[V^{\hat\pi}_{P^{\star},r,h+1}(s_h')]-\EE_{P^{\star}_h(s'_{h}\mid s_h,a_h)}[V^{\hat\pi}_{P^{\star},r,h+1}(s_h')]}_{g_h(s,a)} \right]
\end{align*}
We know $g_h\in \Fcal_h$, $\|g_h\|_\infty\leq 2$.
Therefore, by the \textsc{RepLearn} guarantee, we have
\begin{equation*}
    \EE_{(s,a)\sim \rho_h}\bracks{g_h^2(s,a)}\leq \zeta_n, \EE_{(s,a)\sim \beta_{h}}\bracks{g_h^2(s,a)}\leq \zeta_n
\end{equation*}

Then, we apply one-step-back trick to $\sum_{h=0}^{H-1}\EE_{(s_h,a_h)\sim d^{\hat\pi}_{\hat P,h}}g_h(s,a)$,
\begin{align*}
&\sum_{h=0}^{H-1}\EE_{(s,a)\sim d^{\hat\pi}_{\hat P,h }}[g_h(s,a)]\\  
\leq& 
\sum_{h=0}^{H-2}\EE_{(\tilde s,\tilde a)\sim d^{\hat\pi}_{\hat P ,h}} \|\phi_h(\tilde s,\tilde a)\|_{\Sigma_{\rho_h,\phi_h}^{-1}} \cdot
\sqrt{n|\Acal|^2\EE_{(s,a)\sim \beta_{h}}\bracks{g_{h+1}^2(s,a)}+ 4\lambda_n d+ n \zeta_n} 
+\sqrt{|\Acal|\EE_{(s,a)\sim \rho_0}[g_0^2(s,a)]}\\
\leq& 
\sum_{h=0}^{H-2}\EE_{(\tilde s,\tilde a)\sim d^{\hat\pi}_{\hat P ,h}} \|\phi_h(\tilde s,\tilde a)\|_{\Sigma_{\rho_h,\phi_h}^{-1}} \cdot
\sqrt{n|\Acal|^2\zeta_n+ 4\lambda_n d+ n \zeta_n} 
+\sqrt{|\Acal|\zeta_n}\\
\leq& 
\sum_{h=0}^{H-2}\EE_{(\tilde s,\tilde a)\sim d^{\hat\pi}_{\hat P ,h}} \alpha_n\cdot\|\phi_h(\tilde s,\tilde a)\|_{\Sigma_{\rho_h,\hat \phi_h}^{-1}}
+\sqrt{|\Acal|\zeta_n}\\
\leq& 
\sum_{h=0}^{H-2}\EE_{(\tilde s,\tilde a)\sim d^{\hat\pi}_{\hat P ,h}} c\cdot\hat b(\tilde s, \tilde a)
+\sqrt{|\Acal|\zeta_n} \tag{by \pref{lem:con}}\\
\end{align*}
\end{proof}

\begin{theorem}[Reward-free PAC bound]
Let $\pi^\star$ be the optimal policy under $P^\star$ and $r$. Then,
\begin{align*}
    V^{\pi^\star}_{P^\star,r}-V^{\hat\pi}_{P^\star,r}\leq O\left(H^{5/2}|\Acal|^{1/2}d\log(|\Phi|/\delta)^{1/4}N^{-1/4}\right)
\end{align*}
\end{theorem}
\begin{proof}
We have
\begin{align*}
    V^{\pi^\star}_{P^\star,r}-V^{\hat\pi}_{P^\star,r}
    \leq & \sqrt{|\Acal|\zeta_n}+ V^{\pi^\star}_{\hat P^{\tilde n},r+\hat b^{\tilde n}}-V^{\hat\pi}_{P^\star,r}
    \tag{Apply \pref{lem:optimism} to $\pi^*$}\\
    \leq & \sqrt{|\Acal|\zeta_n}+ V^{\hat\pi}_{\hat P^{\tilde n},r+\hat b^{\tilde n}}-V^{\hat\pi}_{P^\star,r}
    \tag{$\pi^{\star}\in \tilde \Pi$ and $\hat\pi = \argmax_\pi V^\pi_{\hat P^{\tilde n},r+\hat b^{\tilde n}}$}\\
    \leq & 2\sqrt{|\Acal|\zeta_n}+ CV^{\hat\pi}_{\hat P^{\tilde n},\hat b^{\tilde n}}
    \tag{by \pref{lem:rf2}}\\
    \leq & 2\sqrt{|\Acal|\zeta_n}+ CV^{\hat\pi^n}_{\hat P^{\tilde n},\hat b^{\tilde n}}
    \tag{$\hat \pi^n = \argmax_\pi V^{\pi}_{\hat P^n,\hat b^n}$}\\
    \leq & O\left(H^{5/2}|\Acal|^{1/2}d\log(|\Phi|/\delta)^{1/4}N^{-1/4}\right)
    \tag{by \pref{lem:rf1}}\\
\end{align*}
\end{proof}

\begin{theorem}[Model Estimation Error]
For any policy $\pi$, we have
\begin{align*}
    \sum_{h=1}^{H}\EE_{d^\pi_{P^\star,h}}\left[\left\|P_h^\star(\cdot|s,a)-\hat P_h^{\tilde n}(\cdot|s,a)\right\|_{TV}\right]\leq O\left(H^{5/2}|\Acal|^{1/2}d\log(|\Phi|/\delta)^{1/4}N^{-1/4}\right)
\end{align*}
\end{theorem}
\begin{proof}
We have
\begin{align*}
    &\sum_{h=1}^{H}\EE_{d^\pi_{P^\star,h}}\left[\left\|P_h^\star(\cdot|s,a)-\hat P_h^{\tilde n}(\cdot|s,a)\right\|_{TV}\right]
    \\
    =& \underbrace{\left(\sum_{h=1}^{H}\EE_{d^\pi_{\hat P_h^{\tilde n},h}}\left[\left\|P_h^\star(\cdot|s,a)-\hat P_h^{\tilde n}(\cdot|s,a)\right\|_{TV}\right]-\sum_{h=1}^{H}\EE_{d^\pi_{P^\star,h}}\left[\left\|P_h^\star(\cdot|s,a)-\hat P_h^{\tilde n}(\cdot|s,a)\right\|_{TV}\right]\right)}_{\text{term (a)}}\\
    &\qquad\qquad\qquad\qquad\qquad\qquad\qquad\qquad\qquad\qquad\qquad\qquad+\underbrace{\sum_{h=1}^{H}\EE_{d^\pi_{\hat P_h^{\tilde n},h}}\left[\left\|P_h^\star(\cdot|s,a)-\hat P_h^{\tilde n}(\cdot|s,a)\right\|_{TV}\right]}_{\text{term (b)}}
\end{align*}
Since $f(s,a) = \left\|P_h^\star(\cdot|s,a)-\hat P_h^{\tilde n}(\cdot|s,a)\right\|_{TV}\in [0,2]$, treating it as a reward function and apply simulation lemma, we have
\begin{align*}
    \text{term (a)} = & \sum_{h=1}^{H}\EE_{d^\pi_{\hat P_h^{\tilde n},h}}\left[P_h^\star V^{\pi}_{P^\star,f,h+1}(s,a)-\hat P_h^{\tilde n}V^{\pi}_{P^\star,f,h+1}(s,a)\right]
    \\
    \leq & 2H\sum_{h=1}^{H}\EE_{d^\pi_{\hat P_h^{\tilde n},h}}\left[\left\|P_h^\star(\cdot|s,a)-\hat P_h^{\tilde n}(\cdot|s,a)\right\|_{TV}\right]
\end{align*}
Therefore, we have
\begin{align*}
    & \sum_{h=1}^{H}\EE_{d^\pi_{P^\star,h}}\left[\left\|P_h^\star(\cdot|s,a)-\hat P_h^{\tilde n}(\cdot|s,a)\right\|_{TV}\right]
    \\
    = & \text{term (a)}+\text{term (b)}
    \\
    = & (2H+1)\sum_{h=1}^{H}\EE_{d^\pi_{\hat P_h^{\tilde n},h}}\left[\left\|P_h^\star(\cdot|s,a)-\hat P_h^{\tilde n}(\cdot|s,a)\right\|_{TV}\right]
    \\
    \preceq & (2H+1)\sum_{h=0}^{H-1}\EE_{d^\pi_{\hat P_h^{\tilde n},h}}\left[\hat b_h^{\tilde n}(s,a)\right] \tag{one-step back}
    \\
    \leq & (2H+1)\sum_{h=0}^{H-1}\EE_{d^{\pi^{\tilde n}}_{\hat P_h^{\tilde n},h}}\left[\hat b_h^{\tilde n}(s,a)\right]\tag{$\pi^{\tilde n}=\argmax_{\pi}\EE_{d^\pi_{\hat P_h^{\tilde n},h}}\left[\hat b_h^{\tilde n}(s,a)\right]$}
    \\
    \leq & O\left(H^{3}|\Acal|^{1/2}d\log(|\Phi|/\delta)^{1/4}N^{-1/4}\right) \tag{\pref{lem:rf1}}
\end{align*}
\end{proof}

\section{Auxiliary Lemmas}

\begin{lemma}[bounded LSVI solution for Block MDP]\label{lem:bound_w}
For any $\phi\in\Phi$, $f:\Scal\rightarrow[0,L]$, dataset $\Dcal = \{(s_i,a_i,s_i')\}_{i=1}^n$, the ridge regression solution
\begin{equation}
    \hat w_f = \argmin_{w} \sum_{i=1}^n(\phi(s,a)^\top w - f(s'))^2+\lambda\|w\|_2^2
\end{equation}
satisfies $\|\hat w_f\|_\infty\leq L$.
\end{lemma}
\begin{proof}
In Block MDP, $\hat w_f$ takes the following closed-form
\begin{align*}
    \hat w_f(z,a) = \frac{1}{N_\Dcal(\phi(s,a)+\lambda}\sum_{i:\phi(s_i,a_i)=(z,a)}f(s'_i)\leq \frac{N_\Dcal(\phi(s,a))}{N_\Dcal(\phi(s,a)+\lambda}L\leq L
\end{align*}
as needed.
\end{proof}
The following is a standard inequality to prove regret bounds for linear models. Refer to \citet[Lemma G.2.]{agarwal2020pc}.
\begin{lemma}\label{lem:reduction}
Consider the following process. For $n=1,\cdots,N$, $M_n=M_{n-1}+G_n$ with $M_0=\lambda_0 I$ and $G_n$ being a positive semidefinite matrix with eigenvalues upper-bounded by $1$. We have that:
\begin{align*}
    2\ln \det (M_N)-2\ln \det(\lambda_0 I)\geq \sum_{n=1}^N\text{Tr}(G_n M^{-1}_{n-1}). 
\end{align*}
\end{lemma}

\begin{lemma}[Potential Function Lemma]\label{lem:potential}
Suppose $\text{Tr}(G_n)\leq B^2$. 
\begin{align*}
      2\ln \det (M_N)-2\ln \det(\lambda_0 I)\leq d\ln\prns{1+\frac{NB^2}{d\lambda_0}}. 
\end{align*}
\end{lemma}

\begin{lemma}[Generalized Elliptic Potential Lemma, Lemma 24 of \citep{modi2021model}]
\label{lem:gen_elliptic_pot}
For any sequence of vectors $\theta^{\star}_1, \theta^{\star}_2 \ldots, \theta^{\star}_T \in \RR^{d \times T}$ where $\|\theta^{\star}_i\| \le L\sqrt{d}$, for $\lambda \ge L^2d$, we have:
\begin{align*}
    \sum_{t=1}^T \|\Sigma_t^{-1}\theta^{\star}_{t+1}\|_2 \le 2\sqrt{\frac{dT}{\lambda}}.
\end{align*}
\end{lemma}

Next, we provide an important lemma to ensure the concentration of the bonus term. The version for fixed $\phi$ is proved in \citet[Lemma 39]{zanette2021cautiously}.  Here, we take a union bound over the whole feature $\phi\in \Phi$.

\begin{lemma}[Concentration of the Bonus] \label{lem:con}
Set $\lambda_n=\Theta(d\ln(n|\Phi|/\delta))$ for any $n$. Let $\Dcal=\{s_i,a_i\}_{i=0}^{n-1}$ be a stochastic sequence of data where $(s_i,a_i)\sim \rho_i$ where $\rho_i$ can depend on the history of time steps $1,...,i-1$. Let $\rho = \frac{1}{n}\sum_{i=0}^{n-1}\rho$ and
define
\begin{align*}
    \Sigma_{\rho,\phi}=n\EE_{\rho}[\phi(s,a)\phi^{\top}(s,a)]+\lambda_n I,\quad \hat \Sigma_{n,\phi}=\sum_{i=0}^{n-1} \phi(s_i,a_i)\phi^{\top}(s_i,a_i)+\lambda_n I. 
\end{align*} 
Then, with probability $1-\delta$, we have
\begin{align*}
  \forall n \in \NN^{+},\forall \phi\in \Phi, c_1 \|\phi(s,a)\|_{\Sigma^{-1}_{\rho,\phi}}\leq \|\phi(s,a)\|_{\hat \Sigma^{-1}_{n,\phi}} \leq  c_2 \| \phi(s,a)\|_{\Sigma^{-1}_{\rho,\phi}}. 
\end{align*}
\end{lemma}

\begin{lemma}[Simulation Lemma]\label{lem:pdl} Given two MDPs $(P', r+ b)$ and $(P, r)$, for any policy $\pi$, we have:
\begin{align*}
       V^{\pi}_{P',r+b}-V^{\pi}_{P,r}=\sum_{h=1}^H\EE_{(s_h,a_h)\sim d^\pi_{P',h}}\left[b_h(s_h,a_h)+ \EE_{P'_h(s'_{h}\mid s_h,a_h)}[V^{\pi}_{P,r,h+1}(s_h')]-\EE_{P_h(s'_{h}\mid s_h,a_h)}[V^{\pi}_{P,r,h+1}(s_h')] \right]
\end{align*}
and
\begin{align*}
       V^{\pi}_{P',r+b}-V^{\pi}_{P,r}=\sum_{h=1}^H\EE_{(s_h,a_h)\sim d^\pi_{P,h}}\left[b_h(s_h,a_h)+ \EE_{P'_h(s'_{h}\mid s_h,a_h)}[V^{\pi}_{P,r+b,h+1}(s_h')]-\EE_{P_h(s'_{h}\mid s_h,a_h)}[V^{\pi}_{P,r+b,h+1}(s_h')] \right]
\end{align*}
\end{lemma}
We note that since both the occupancy measure and Bellman updates under $\hat P$ are defined in the exact same way as if $\hat P$ is a proper probability matrix, the classic simulation lemma also applies to $\hat P$.

\section{Experiment Details}
\label{sec:app:exp}

\subsection{Setup Details and Hyperparameters}
\textbf{CombLock environment Details} For the design of the diabolical combination lock environment in our comparison with \homer, we follow the design as in \homer, but we present the environment hyperparameters in Table.~\ref{app:table:comblock} for completeness. We also provide a more detailed explanation of how we record result in Fig.~\ref{tbl:results}b: after each policy update, we perform 20 i.i.d. rollouts using the latest policy and record the mean returns in these 20 evaluation runs. If one algorithm can get a mean return of 1 for 5 consecutive policy updates, we count the algorithm solving the environment and record the number of episodes it uses.

\begin{table}[ht] 
\centering
\begin{tabular}{ccc} 
\toprule
                                                & Value                     \\ 
\hline
Horizon                                         & 6,12,25,50,100            \\
Switch probability                              & 0.5                       \\
Anti reward                                     & 0.1                       \\
Anti reward probability                         & 0.5                       \\
Final reward                                    & 1                         \\
Number of actions                               & 10                        \\
Observation noise std                           & 0.1                       \\
Random seeds                                    & 1,12,123,1234,12345       \\
\toprule
\end{tabular}
\caption{Hyperparameters for the rich observation comblock environment with sparse and anti-shaped reward.}
\label{app:table:comblock}
\end{table}

\paragraph{\algname Implementation for CombLock Environment} For the dense reward diabolical combination lock environment, we first use a random policy to collect $10000$ episodes of samples before our first iteration of feature learning. We perform this sample warm-up procedure for $H=50$ and $H=100$ experiments only. We maintain a separate buffer for each timestep $h$, and in practice we mix the samples in $\Dcal$ and $\Dcal'$ together. For each buffer, we limit the size of the buffer to 10000 and update the buffer with  first-in-first-out procedure. Between each update we rollout $50 \times H$ episodes to collect data. For the optimization method we use SGD with a momentum factor of $0.99$. Finally due to a different latent state distribution, we use softmax with temperature  0.1 for $\phi_0$ and a temperature of 1 for all the other timesteps. We provide the full list of hyperparameters in Table.~\ref{app:table:hyperparam:ours_comblock}.

\begin{table}[ht] 
\centering
\begin{tabular}{ccc} 
\toprule
                                                & Value Considered          & Final Value  \\ 
\hline
Decoder $\phi$ learning rate                    & \{1e-2, 5e-3, 1e-3\}      & 1e-2         \\
Discriminator $f$ learning rate                 & \{1e-2, 5e-3, 1e-3\}      & 1e-2         \\
Discriminator $f$ hidden layer size             & \{256,512\}               & 256          \\
RepLearn Iteration $T$                          & \{10,20,30\}              & 30           \\
Decoder $\phi$ number of gradient steps         & \{32,64,128,256\}         & 64           \\
Discriminator $f$ number of gradient steps      & \{128,256\}               & 128          \\
Decoder $\phi$ batch size                       & \{128,256,512\}           & 512          \\
Discriminator $f$ batch size                    & \{128,256,512\}           & 512          \\
RepLearn regularization coefficient $\lambda$   & \{1,0.1,0.01\}            & 0.01         \\
Decoder $\phi$ softmax temperature              & \{1,0.5,0.1\}             & 1            \\
Decoder $\phi_0$ softmax temperature            & \{1,0.1\}                 & 0.1          \\
LSVI bonus coefficient $\beta$                  & \{10, $\frac{H}{5}$\}     & $\frac{H}{5}$\\
LSVI regularization coefficient $\lambda$       & \{1\}                     & 1            \\
Buffer size                                     & \{1e5\}                   & 1e5          \\
Update frequency                                & \{50,100\}                & 50           \\
Warm up samples ($H=50,100$)                    & \{10000\}                 & 10000        \\
\toprule
\end{tabular}
\caption{Hyperparameters for \algname in sparse reward comblock experiment.}
\label{app:table:hyperparam:ours_comblock}
\end{table}

\paragraph{Baseline Implementation for CombLock Environment} In this section we provide the hyperparameteres we use for PPO-RND in Table.~\ref{app:table:comblock:ppo} and LSVI-UCB in Table.~\ref{app:table:hyperparam:comblock:rff}.\footnote{We use the public code for PPO-RND, which is available at \href{https://github.com/mbhenaff/PCPG}{here}.} For the RFF feature, we choose the bandwidth with median trick. 

\begin{table}[ht] 
\centering
\begin{tabular}{cc} 
\toprule
                                                & Value                     \\ 
\hline
Learning rate                                   & 1e-3                      \\
Hidden layer size                               & 64                        \\
$\tau_{\text{GAE}}$                             & 0.95                      \\
Gradient clipping                               & 5.0                       \\
Entropy bonus                                   & 0.01                      \\
Clip ratio                                      & 0.2                       \\
Minibatch size                                  & 160                       \\
Optimization epoch                              & 5                         \\
Intrinsic reward normalization                  & False                     \\
Intrinsic reward coefficient                    & 1e3                       \\
Extrinsic reward coefficient                    & 1.0                       \\
\toprule
\end{tabular}
\caption{Hyperparameters for PPO-RND in comblock sparse reward experiment.}
\label{app:table:comblock:ppo}
\end{table}

\begin{table}[ht] 
\centering
\begin{tabular}{ccc} 
\toprule
                                                & Value Considered          & Final Value  \\ 
\hline
LSVI bonus coefficient $\beta$                  & \{10, $\frac{H}{5}$\}     & $\frac{H}{5}$\\
LSVI regularization coefficient $\lambda$       & \{1\}                     & 1            \\
Buffer size                                     & \{1e5,5e5,1e6\}           & 1e6          \\
Update frequency                                & \{50,100,250\}            & 250          \\
Kernel bandwidth                                & Median Trick              & 5            \\
Feature dimension $H=6$                         & \{200\}                   & 200          \\
Feature dimension $H=12$                        & \{200,500\}               & 500          \\
Feature dimension $H=25$                        & \{500,800,1000\}          & 1000         \\
\toprule
\end{tabular}
\caption{Hyperparameters for for LSVI-UCB with RFF feature in comblock sparse reward experiment.}
\label{app:table:hyperparam:comblock:rff}
\end{table}

\paragraph{Simplex Feature Experiment Details} We first present detailed description of dynamics of the new MDP. Given a observation-action pair $s,a$, the simplex feature is given by $z(s) = \text{softmax}(I_3 R^{-1} s / \tau_{\text{env}})$, where $R^{-1} \in \mathbb{R}^{\text{dim}_s \times \text{dim}_s}$ is the inverse of the Hadamard matrix, and $I_3 \in \mathbb{R}^{3 \times \text{dim}_s}$, a matrix with a $3 \times 3$ identity matrix at its first three columns and zero everywhere else. Thus the ground truth feature is given by $\phi^{\ast} = z(s) \otimes a$. With a observation-action pair $s,a$, the environment first sample a latent state according to the probability simplex $z(s)$, and then transit according to the action $a$ and the transition rules of the original comblock environment. For our results in Fig.~\ref{exp:fig:simplex}, we use $\tau_{\text{env}} = 0.2$. The moving average is cross 50 evaluations (and 20 i.i.d. rollouts per evaluation). We use the same set of hyperparameters for \algname as in Table.~\ref{app:table:hyperparam:ours_comblock} and for LSVI as in Table.~\ref{app:table:hyperparam:comblock:rff}.

\paragraph{Dense Reward Environment Details} In the dense reward environment, we remove the anti-shaped reward and the agent receives rewards while staying in the good states. We keep the final reward if the correct actions are taken in the last layer. We provide the hyperparameters of the environment in Table.~\ref{app:table:comblock_dense}. We provide the hyperparameters of \algname for the dense reward environment in Table.\ref{app:table:hyperparam:ours_comblock_dense} (note only the LSVI bonus coefficients are different). We provide the hyperparameters of PPO for the dense reward environment in Table.\ref{app:table:comblock:ppo_dense}.

\begin{table}[ht] 
\centering
\begin{tabular}{ccc} 
\toprule
                                                & Value                     \\ 
\hline
Horizon                                         & 30                        \\
Switch probability                              & 0.5                       \\
Step reward                                     & 0.1                       \\
Final reward                                    & 1                         \\
Reward probability                              & 1                         \\
Number of actions                               & 10                        \\
Observation noise std                           & 0.1                       \\
Random seeds                                    & 1,12,123,1234,12345       \\
\toprule
\end{tabular}
\caption{Hyperparameters for the rich observation comblock environment with dense reward.}
\label{app:table:comblock_dense}
\end{table}

\begin{table}[ht] 
\centering
\begin{tabular}{ccc} 
\toprule
                                                & Value Considered          & Final Value  \\ 
\hline
Decoder $\phi$ learning rate                    & \{1e-2, 5e-3, 1e-3\}      & 1e-2         \\
Discriminator $f$ learning rate                 & \{1e-2, 5e-3, 1e-3\}      & 1e-2         \\
Discriminator $f$ hidden layer size             & \{256,512\}               & 256          \\
RepLearn Iteration $T$                          & \{10,20,30\}              & 30           \\
Decoder $\phi$ number of gradient steps         & \{32,64,128,256\}         & 64           \\
Discriminator $f$ number of gradient steps      & \{128,256\}               & 128          \\
Decoder $\phi$ batch size                       & \{128,256,512\}           & 512          \\
Discriminator $f$ batch size                    & \{128,256,512\}           & 512          \\
RepLearn regularization coefficient $\lambda$   & \{1,0.1,0.01\}            & 0.01         \\
Decoder $\phi$ softmax temperature              & \{1,0.5,0.1\}             & 1            \\
Decoder $\phi_0$ softmax temperature            & \{1,0.1\}                 & 0.1          \\
LSVI bonus coefficient $\beta$                  & \{$\frac{H}{5}$,$\frac{H}{50}$\}     & $\frac{H}{50}$\\
LSVI regularization coefficient $\lambda$       & \{1\}                     & 1            \\
Buffer size                                     & \{1e5\}                   & 1e5          \\
Update frequency                                & \{50,100\}                & 50           \\
Warm up samples                                 & \{0\}                     & 0            \\
\toprule
\end{tabular}
\caption{Hyperparameters for \algname in sparse reward comblock experiment.}
\label{app:table:hyperparam:ours_comblock_dense}
\end{table}

\begin{table}[ht] 
\centering
\begin{tabular}{ccc} 
\toprule
                                                & Value Considered  & Value                     \\ 
\hline
Learning rate                                   &\{1e-3,5e-4,1e-4\} & 1e-3                      \\
Hidden layer size                               &\{64\}             & 64                        \\
$\tau_{\text{GAE}}$                             &\{0.95\}           & 0.95                      \\
Gradient clipping                               &\{5.0\}            & 5.0                       \\
Entropy bonus                                   &\{0.01,0.001\}     & 0.01                      \\
Clip ratio                                      &\{0.2\}            & 0.2                       \\
Minibatch size                                  &\{160\}            & 160                       \\
Optimization epoch                              &\{5\}              & 5                         \\
\toprule
\end{tabular}
\caption{Hyperparameters for PPO in comblock sparse reward experiment.}
\label{app:table:comblock:ppo_dense}
\end{table}

\begin{figure}[ht]
    \centering
    \includegraphics[width=1\linewidth]{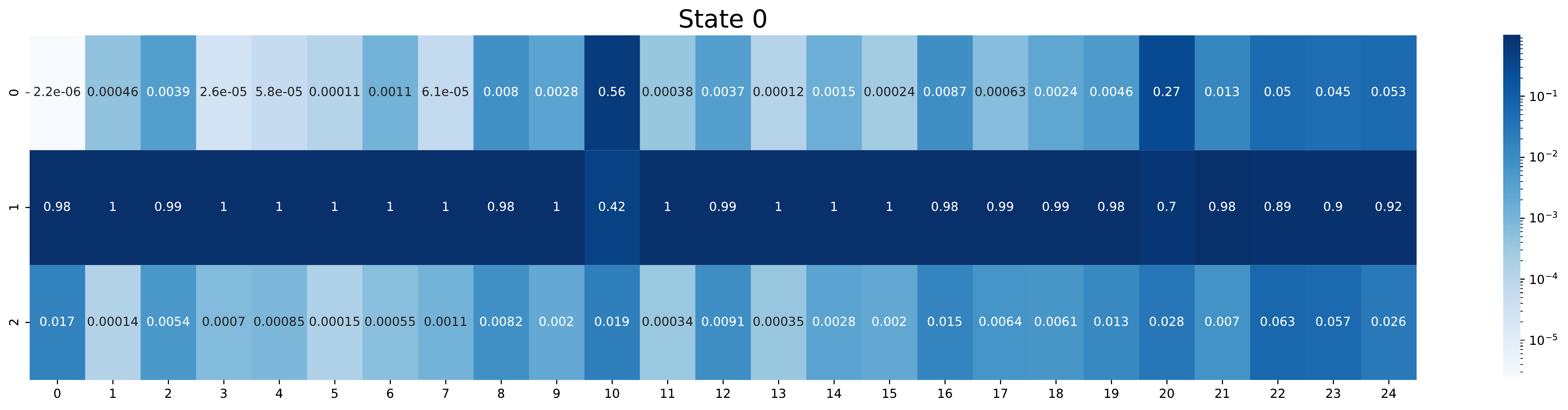}
    \centering
    \includegraphics[width=1\linewidth]{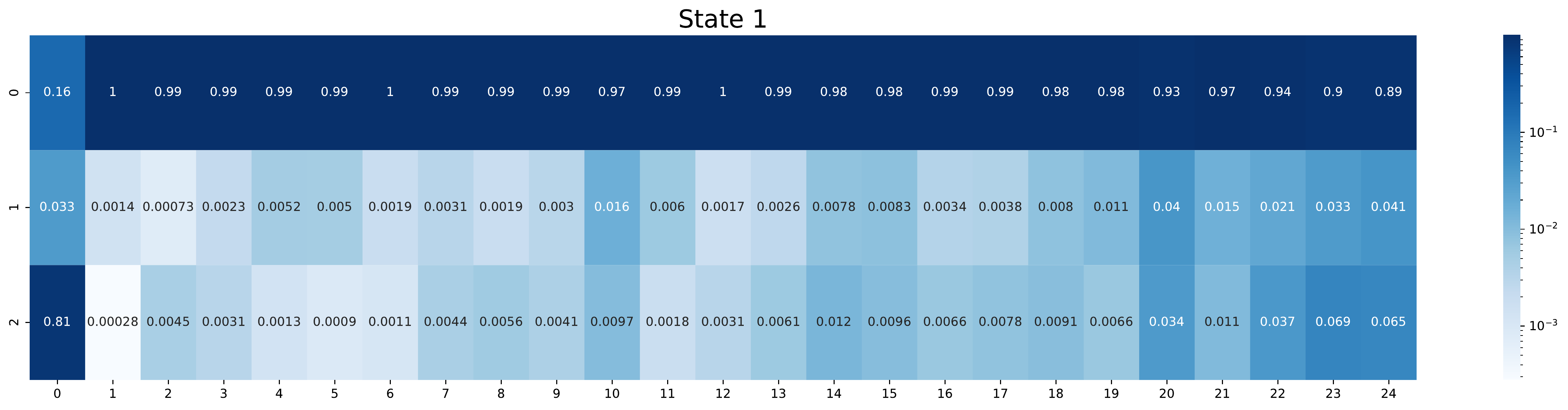}
    \centering
    \includegraphics[width=1\linewidth]{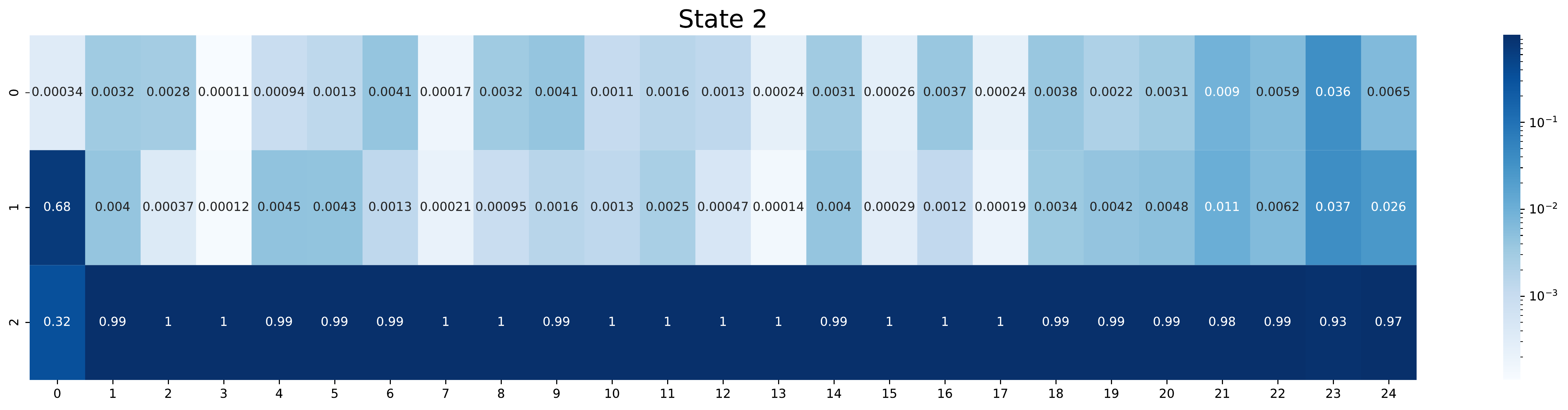}
    \caption{Visualization of the decoder. The $h$-th column in the $i$-th image denotes the averaged decoded states from the 50 observations generated by latent state $z_{i,h}$, for $i \in \{0,1,2\}$ and $h \in [25]$. The heatmap color transits in log scale. Note that we have near one-hot encoding in most of the levels, except $h=10$ and $h=20$. This is because we have the same optimal actions for the good states in $h=10$ and $h=20$ thus we still recover the ground-truth features (see text for more details). Also note that the optimal features are recovered up to permutation (e.g., in this case we decode state 0 as state 1 and state 1 as state 0, but this is still the optimal decoding).}
    \label{app:fig:decoder}
\end{figure}

\subsection{Visualization of the decoder}
\label{app:visualization_latent}
In Figure~\ref{app:fig:decoder}, we visualize the decoder on the combination lock example (with the Block MDP structure). We run \algname on the Block MDP combination lock until it solves the problem (i.e., achieve the optimal total reward). Denote the learned decoders as $\hat\psi_h$ for all $h \in [H]$. Note that $\hat\psi_h$ maps from state (i.e., observation) $s$ to a 3-dimensional vector in the simplex (since we use softmax, i.e., $\hat{\psi}(s) = \text{softmax}(A_h s / \tau)$, as defined in the implementation section). Ideally, we hope that $\hat\psi_h$ can output pretty deterministic distribution over 3 latent states (i.e., $\hat\psi_h(s)$ is close to a one-hot encoding vector), and can decode the latent state (up to permutation). 

We test the decoders as follows. For each state $z_{i;h}$ for $i \in \{0,1,2\}$ and $h \in [H]$, where in this section we use $H=25$, we sample 50 observations for each state (following the emission distribution described in the environment section), and we take the average of the 50 decoded states (decoded by $\hat \psi_h$). 

In Fig.~\ref{app:fig:decoder}, we demonstrate the decoded states. The $h$-th column in the $i$-th image denotes the average of the 50 decoded states from observations generated from $z_{i,h}$, for $i \in \{0,1,2\}$ and $h \in [25]$ (i.e., the image number denotes the ground truth state, the x-axis denotes the timestep, and the y-axis in each image denotes the averaged value of the decoded states on each dimension).

Interestingly, we notice that our decoder $\hat\psi_h$ fails to decode the two good latent states (i.e., state 0 and state 1) confidently at $h = 10$ and $h = 20$. However, this is not a failure case. The reason is that at $h = 10$ and $h = 20$, the two good latent states share the same optimal action (i.e., the action that transits  the agent from a good state to the next two good states). Namely, the two good states at $h = 10$ (and $h=20$) share the same transition.  Hence, there is no need for the decoder to distinguish these two states. Note that our decoders still successfully differentiate the bad state and the two good states at $h = 10$ and $20$. This phenomenon is also observed in \homer. Also note that for $h= 0$, the decoder is only required to distinguish state 0 from state 1, because the initial distribution is uniform over state 0 and state 1 only, and assigns 0 mass to state 2 (i.e., we never reach state 2 in $h=0$).

\end{document}